\documentclass{article}
\PassOptionsToPackage{round}{natbib}

\usepackage[preprint]{neurips_2026}

\usepackage[utf8]{inputenc} 
\usepackage[T1]{fontenc}    
\usepackage{hyperref}       
\usepackage{url}            
\usepackage{booktabs}       
\usepackage{amsfonts}       
\usepackage{nicefrac}       
\usepackage{microtype}      
\usepackage{xcolor}         

\usepackage{amsmath, amssymb, amsthm}
\allowdisplaybreaks
\usepackage{mathtools}
\mathtoolsset{showonlyrefs}
\usepackage[capitalise]{cleveref}
\usepackage[ruled, vlined]{algorithm2e}
\usepackage{subfiles}
\usepackage{enumitem}

\theoremstyle{definition}
\newtheorem{thm}{Theorem}
\newtheorem{asm}{Assumption}
\newtheorem{prop}{Proposition}
\newtheorem{dfn}{Definition}
\newtheorem{lem}{Lemma}
\newtheorem*{lem*}{Lemma}

\theoremstyle{remark}
\newtheorem{rem}{Remark}
\newtheorem*{rem*}{Remark}

\newcounter{num@prop@weak}
\newcounter{num@thm@strong}
\newcounter{num@prop@preservation}
\newcounter{num@thm@sft}

\crefname{thm}{Theorem}{Theorems}
\crefname{asm}{Assumption}{Assumptions}
\crefname{prop}{Proposition}{Propositions}
\crefname{dfn}{Definition}{Definitions}
\crefname{lem}{Lemma}{Lemmas}
\crefname{cor}{Corollary}{Corollaries}
\crefname{algocf}{Algorithm}{Algorithms}

\newcommand{\IE}{\mathop{\mathrm{IE}}}
\newcommand{\GE}{\mathop{\mathrm{GE}}}
\newcommand{\clip}{\mathop{\mathrm{clip}}}
\newcommand{\epsWeakLOne}{\varepsilon_w}
\newcommand{\epsWeakAlign}{\tilde{\varepsilon}_w}
\newcommand{\epsStrongAlign}{\tilde{\varepsilon}_s}
\newcommand{\epsStrongEff}{\mathcal{E}_s}
\newcommand{\TODO}[1]{}
\newcommand{\R}{\mathbb{R}}

\newcommand{\E}{\mathbb{E}}

\title{The Mechanism of Weak-to-Strong Generalization: Feature Elicitation from Latent Knowledge}

\author{%
  Ryoya Awano\textsuperscript{1\texttt{*}} \quad Taiji Suzuki\textsuperscript{1,2} \\[0.3em]
  \textsuperscript{1}University of Tokyo \quad
  \textsuperscript{2}Center for Advanced Intelligence Project, RIKEN \\[0.2em]
  \texttt{\textsuperscript{*}awano-ryoya284@g.ecc.u-tokyo.ac.jp}
}

\begin{document}

\maketitle

\begin{abstract}
  Weak-to-strong (W2S) generalization, in which a strong model is fine-tuned on outputs of a weaker, task-specialized model, has been proposed as an approach to aligning superhuman AI systems.
  Existing theoretical analyses either fix the student's representations or operate in restricted settings.
  Whether multi-step SGD can succeed in feature learning while preserving diverse pre-trained capabilities remains open.
  We study W2S in the setting of reward-model learning with two-layer neural networks.
  The strong model has pre-trained representations organized into low-dimensional subspaces $V_k$, and is fine-tuned under the supervision of a weak model specialized on task $\kappa$.
  We prove that the strong model efficiently learns task $\kappa$, eliciting its pre-trained knowledge while retaining general capabilities.
  This establishes W2S generalization in the feature-learning regime, in the sense that the strong model acquires the target feature direction through W2S training, rather than having it given a priori.
  Moreover, W2S preserves pre-trained off-target features, whereas standard supervised fine-tuning causes catastrophic forgetting when off-target feature directions are correlated with the target's.
  Numerical experiments on synthetic data confirm our theoretical results.
\end{abstract}

\section{Introduction}\label{sec:intro}

\citet{burns2023weaktostrong} showed that fine-tuning a strong language model on the outputs of a weaker, task-specialized model can produce a student that surpasses its teacher on that task (e.g., GPT-4~\citep{achiam2023gpt} fine-tuned on GPT-2~\citep{radford2019language} labels outperforms GPT-2 across NLP benchmarks).
This phenomenon, \emph{weak-to-strong} (W2S) generalization, has become a workhorse for studying how imperfect supervision by a less capable model can nonetheless elicit useful behavior from a more capable one. It underlies the superalignment program, in which humans act as the weak supervisors of models beyond their own capabilities.

Empirically, \citet{burns2023weaktostrong} demonstrate that W2S can successfully elicit capabilities even under linear probing, where the strong model's representations are fixed. 
This may partly explain why many early theoretical treatments of W2S have focused on the linear regime.
However, they also show that full fine-tuning—allowing for feature learning—unlocks better performance.
To explain this gain, they hypothesize that feature learning makes the concepts acquired during pre-training more "salient." 
W2S would thus elicit latent knowledge rather than instill entirely new capabilities.

The theoretical study of W2S has addressed either abstract frameworks that do not analyze gradient-based optimization \citep{lang2024theoretical, shin2025weaktostrong, charikar2024quantifying, xue2025representations} or linear and random feature models \citep{wu2025provable, ildiz2025highdimensional, dong2025discrepancies, medvedev2025weaktostrong}.
Neither line derives the difference between weak and strong model features from the optimization process itself.
Motivated by the empirical advantage of full fine-tuning, recent theoretical works \citep{moniri2025on,oh2025from} took the first steps toward W2S feature learning, but their theoretical settings leave significant gaps regarding architecture, pre-training priors, and task specialization.
\citet{oh2025from} analyze W2S by employing a linear CNN as the weak model and a two-layer ReLU CNN as the strong model, with convergence proved under full-batch gradient descent.
However, this CNN-based architecture differs from the LLMs used in \citet{burns2023weaktostrong}, assumes no pre-training prior on the strong model, and does not specialize the weak model to a particular task.
\citet{moniri2025on} use a two-component nonlinear model rather than a neural network (NN), assume the strong model already knows the exact direction of the off-target task at initialization, and their analysis tracks only a single gradient step.

A key open question is whether multi-step SGD can provably succeed at W2S feature learning while preserving diverse capabilities, in a more realistic setting.
Answering this question would shed light on the mechanism by which W2S makes pre-trained latent knowledge salient.
We study this question in the setting of reward-model learning, where, as in \citet{burns2023weaktostrong}, the strong model is trained on the weak model to estimate a reward directly.

Pre-trained models have been observed to localize individual tasks in distinct subsets of neurons \citep{dai2022knowledge, panigrahi2023taskspecific}.
We therefore adopt an additive reward $r^*(x) = \sum_{k=1}^K r^*_k(\theta_k^\top x)$, where each component $r^*_k$ corresponds to one task.
Furthermore, consistent with empirical findings that pre-trained models encode concepts in low-dimensional subspaces \citep{beaglehole2026toward}, we assume the strong model has pre-training knowledge of the task subspaces $V_k$.

We focus on the setting where the weak model is a small two-layer NN specialized to task $\kappa$, trained to approximate $r^*_\kappa$.
The strong model, a larger two-layer NN, queries the weak model based on its prior knowledge. 
It then trains by multi-step SGD on the weak model's outputs with nonlinear transformation.

\paragraph{Contributions}
\begin{itemize}[leftmargin=*]
    \item \textbf{Feature elicitation via SGD analysis of W2S.} 
    We study W2S in a setting closer to pre-trained LLMs, with an additive multi-feature reward and task subspaces as pre-training knowledge.
    To our knowledge, we give the first multi-step SGD convergence guarantee for W2S in the feature-learning regime, with sample complexity $\tilde{O}(s^{3/2})$ (\cref{thm:strong-informal}).
    This sample complexity depends only on the subspace dimension $s$, not the input dimension $d$, and gives the provable account of how W2S elicits pre-trained latent features.
    \item \textbf{Feature preservation under W2S.}
    The strong model's pre-trained input distribution, concentrated on $V_\kappa$, and the nonlinear teacher transformation together act as an implicit regularizer within the task subspaces.
    We prove that the feature directions of off-target tasks are preserved, provided those tasks are allocated more neurons than the target task (\cref{prop:preservation-informal}).
    This preservation, combined with feature elicitation on task $\kappa$, provides a theoretical account of W2S generalization.
    \item \textbf{Advantages over standard SFT.}
    While \cref{prop:preservation-informal} shows that W2S preserves off-target features, we also analyze the behavior of standard SFT on the same strong model, replacing the W2S algorithm with isotropic sampling and the direct teacher $r^*_\kappa$.
    \cref{thm:sft} shows that off-target tasks with feature directions similar to $\theta_\kappa$ suffer catastrophic forgetting.
    Together with \cref{thm:strong-informal}, this provides a theoretical account of the advantages of W2S over standard SFT in terms of sample complexity and feature preservation.
\end{itemize}

\paragraph{Related work}

\emph{Weak-to-strong generalization.}
Theoretical analyses of W2S fall into two groups. 
One line characterizes when W2S succeeds through abstract frameworks without analyzing gradient-based optimization \citep{lang2024theoretical, shin2025weaktostrong, charikar2024quantifying, xue2025representations}. 
Another restricts to linear or random feature models \citep{wu2025provable, ildiz2025highdimensional, dong2025discrepancies, medvedev2025weaktostrong}.
Moving beyond these settings, \citet{moniri2025on, oh2025from} initiate the non-linear feature-learning analysis of W2S.
\citet{moniri2025on} assume the strong model starts with pre-trained knowledge of the target direction and analyze one step of SGD, while \citet{oh2025from} use a CNN-based architecture without a pre-training prior.
We analyze multi-step SGD in a setting closer to pre-trained LLMs, with an additive multi-feature reward and a more natural pre-training prior.
Additional related work is discussed in \cref{app:related}.
We discuss the linear representation hypothesis motivating our subspace prior, and prior work on feature learning for single-index and additive models.

\section{Problem setting}\label{sec:problem}

We formalize the W2S generalization framework for reward learning, modeling each prompt–response pair $(u,v)$ as a vector $x \in \mathbb{R}^d$. 
The goal is to learn the true reward $r^*(x)$ from weak supervision.

\paragraph{Notations.}
We write $\|\cdot\|$ for the $\ell_2$ norm (vectors) and operator norm (matrices).
For $f\colon \R \to \R$, we denote $f^i$ for the $i$-th power and $f^{(i)}$ for the $i$-th derivative.
Let $\mathrm{He}_i$ be the $i$-th probabilist's Hermite polynomial.
We write $S^{d-1}$ for the unit sphere in $\R^d$, and $\tilde{\nabla}_w := (I - ww^\top)\nabla_w$ for the spherical gradient at $w \in S^{d-1}$.
We use $\tilde{O}(\cdot)$, $\tilde{\Omega}(\cdot)$, $\tilde{\Theta}(\cdot)$ to suppress $\mathrm{polylog}(d)$ factors.
An event holds \emph{with high probability} if $\mathbb{P}(\cdot) \geq 1 - d^{-C}$ for an arbitrarily large constant $C > 0$.
We write $\Delta^{K-1} = \{(\lambda_k)_{k=1}^K : \lambda_k \geq 0,\, \sum_k \lambda_k = 1\}$ for the probability simplex.

\subsection{Sample complexity and the information exponent}

We define the information \citep{dudeja2018learning, benarous2021online} and generative \citep{damian2024computationalstatistical} exponents, which govern the sample complexity of learning a single-index model.

\begin{dfn}[Information and generative exponents]\label{dfn:ie-ge}
  Let $f\colon \R \to \R$ be square-integrable under $\mathcal{N}(0,1)$ with Hermite expansion $f(z) = \sum_{j \geq 0} \frac{\alpha_j}{\sqrt{j!}}\,\mathrm{He}_j(z)$.
  The \emph{information exponent} of $f$ is $\IE(f) = \min\{j > 0 : \alpha_j \neq 0\}$.
  The \emph{generative exponent} of $f$ is $\GE(f) = \inf_{\mathcal{T} \in L^2}\, \IE(\mathcal{T} \circ f)$.
\end{dfn}

For a polynomial $f$ of degree $q$, we have $\GE(f) \leq \IE(f) \leq q$; moreover, $\GE(f) = 2$ if $f$ is even, and $\GE(f) = 1$ otherwise \citep{lee2024neural}.
Intuitively, a larger $p = \IE(f)$ means the signal about the direction $\theta$ appears only in higher-order Hermite components of $f$, making it harder to detect from samples $x \sim \mathcal{N}(0, I_d)$.
When $p > 2$, online SGD requires $\tilde{O}(d^{p-1})$ samples to find the hidden direction $\theta$ from observations $(x, f(\theta^\top x) + \zeta)$ of the single-index model, with $x \sim \mathcal{N}(0, I_d)$ \citep{benarous2021online}. 
However, by applying a nonlinear transformation, \citet{chen2020learning} achieve $\tilde{O}(d)$ sample complexity for polynomial link functions.

\subsection{Additive reward model}
As motivated in \cref{sec:intro}, we model the true reward as an additive combination of $K$ single-index components.
Following \citet{oko2024learningsumdiversefeatures}, we study the case where each link function is a polynomial.
Each feature direction $\theta_k$ is in a pre-trained subspace $V_k \subseteq \R^d$, motivated by empirical evidence that LLMs encode concepts in low-dimensional subspaces \citep{beaglehole2026toward}.
A similar assumption appears in \citet{oko2024pretrained} and \citet{nishikawa2025nonlinear}, where feature vectors are drawn from a low-dimensional subspace identified by pre-training.

\begin{asm}[Additive reward model with task subspaces]\label{asm:additive}
  The reward $r^*$ decomposes additively across $K$ tasks indexed by $k \in [K]$, with target task $\kappa$.
  \begin{enumerate}[leftmargin=*]
    \item[\textnormal{(i)}] \textbf{Additive reward.}
      The true reward is
      \begin{align}
        r^*(x) = \sum_{k=1}^K r^*_k(\theta_k^\top x), \qquad r^*_k = \pi_k \sigma^*_k,
      \end{align}
      observed as $y = r^*(x) + \zeta$ with $x \sim \mathcal{N}(0, I_d)$, $\zeta \sim \mathcal{N}(0,1)$, and $\|\theta_k\| = 1$.
      Each link function $\sigma^*_k \in L^2(\mathcal{N}(0,1))$ is normalized so that $\E_{t \sim \mathcal{N}(0,1)}[\sigma^*_k(t)] = 0$ and $\E_{t \sim \mathcal{N}(0,1)}[(\sigma^*_k(t))^2] = 1$.
      The target link function $\sigma^*_\kappa$ is a degree-$q$ polynomial with $\IE(\sigma^*_\kappa) = p > 2$, as in \citet{oko2024learningsumdiversefeatures, simsek2025learning, ren2025emergence}.
      The weights $\pi_1 \geq \cdots \geq \pi_K \geq 0$ satisfy $\sum_{k=1}^K \pi_k^2 = 1$.
    \item[\textnormal{(ii)}] \textbf{Task subspaces.}
      Each task $k$ is associated with a subspace $V_k \subseteq \R^d$ of dimension $s = d^\alpha$ for some $\alpha \in (0,1)$, containing its feature direction $\theta_k \in V_k$.
      The subspaces $V_k$ need not be mutually orthogonal; \cref{prop:preservation-informal} exploits nonzero inner products $\theta_k^\top \theta_\kappa$ across tasks.
      We write $\Sigma_k$ for the orthogonal projection onto $V_k$.
  \end{enumerate}
\end{asm}

\begin{rem*}
  The Gaussian assumption $x \sim \mathcal{N}(0, I_d)$ is standard in the theoretical study of feature learning \citep{benarous2021online, damian2024computationalstatistical}.
  Writing $\sigma^*_\kappa(t) = \sum_{i=p}^{q} \frac{\alpha_i}{\sqrt{i!}}\mathrm{He}_i(t)$, the normalization condition reads $\sum_{i=p}^{q} \alpha_i^2 = 1$.
  The subspaces $V_k$ encode the pre-training prior that the strong model has identified the subspace $V_k$ containing $\theta_k$, but not $\theta_k$ itself, extending the setting of \citet{moniri2025on} where off-target feature vector is known exactly.
\end{rem*}

\subsection{Weak model}\label{subsec:weak}
We consider that the weak model is small and specialized to task $\kappa$.
We model it as a two-layer network whose neurons are nearly aligned with $\theta_\kappa$ and whose output approximates $r^*_\kappa$ in $L^1$.

\begin{asm}[Weak model]\label{asm:weak}
  The weak model, fine-tuned on task $\kappa \in [K]$, is a two-layer NN
  \begin{align}
    r^w(x) = \frac{1}{N^w}\sum_{n=1}^{N^w} a^w_n \sigma^w_n(w_n^\top x + b^w_n),
  \end{align}
  where each activation $\sigma^w_n$ is a degree-$q$ polynomial with Hermite coefficients uniformly bounded in $n$, $\|w_n\| = 1$, $\|a^w\|_1 = \tilde{O}(\pi_\kappa N^w)$, and $|b^w_n| \leq C_b = \tilde{O}(1)$.
  The weak model has learned feature $\kappa$ in the following sense:
  \begin{enumerate}[leftmargin=*]
    \item[\textnormal{(i)}] \textbf{Feature alignment.}
      There exists $\epsWeakAlign > 0$ such that $\theta_\kappa^\top w_n \geq 1 - \epsWeakAlign$ for all $n \in [N^w]$.
    \item[\textnormal{(ii)}] \textbf{$L^1$ approximation.}
      There exists $\epsWeakLOne > 0$ such that
      $\E_{x \sim \mathcal{N}(0, I_d)}\bigl[\bigl|r^w(x) - r^*_\kappa(\theta_\kappa^\top x)\bigr|\bigr] \leq \pi_\kappa \epsWeakLOne$.
  \end{enumerate}
\end{asm}
The next proposition shows that a weak model with pre-training knowledge of $V_\kappa$ can satisfy \cref{asm:weak} via fine-tuning on task $\kappa$, with sample complexity improved by the subspace prior (for the formal statement, see \cref{prop:weak}).

\begin{prop}[Weak model construction; informal]\label{prop:weak-informal}
  Under \cref{asm:additive,asm:complexity-weak}, online SGD (\cref{alg:weak}) applied to the single-index model $r^*_\kappa = \pi_\kappa \sigma^*_\kappa$
  produces a weak model satisfying \cref{asm:weak} with
  $\epsWeakLOne = \tilde{\Theta}(\epsWeakAlign)$, using
  \begin{align}
    T = \tilde{O}\!\left(d^{p/2}\, s^{(p-2)/2} \vee d\,\epsWeakAlign^{-2} \vee \epsWeakAlign^{-3}\right)
  \end{align}
  samples. This improves over the baseline $\tilde{O}(d^{p-1})$ required without
  subspace knowledge.
\end{prop}
\setcounter{num@prop@weak}{\value{prop}}

\begin{rem*}
  The analysis builds on the literature \citep{benarous2021online, damian2024computationalstatistical, lee2024neural, oko2024learningsumdiversefeatures}.
  However, these results guarantee only (ii) in \cref{asm:weak}; satisfying (i) additionally requires PCA-like filtering of neurons without observing $\theta_\kappa$ (see \cref{weak-filtering}). 
  If any neuron violates condition~(i), the weak model's $L^1$ error under the subspace distribution is no longer controlled by its error under the $d$-dimensional isotropic distribution (see \cref{lem:strong-subspace-1}), and the bound in \cref{thm:strong-informal} fails.
  In addition, to incorporate the pre-trained knowledge of $V_\kappa$, we initialize $w_n$ uniformly on $S^{d-1} \cap V_\kappa$ rather than isotropically on $S^{d-1}$.
  This achieves complexity $\tilde{O}(d^{p/2} s^{(p-2)/2})$ rather than the $\tilde{O}(d^{p-1})$ of \citet{benarous2021online}.
  The number of samples required to drive the alignment $\theta_\kappa^\top w_n^t$ to $\Theta(1)$ scales as $(\theta_\kappa^\top w_n^0)^{-(p-2)}$, so the subspace prior reduces the factor $d^{(p-2)/2}$ to $s^{(p-2)/2}$.
  The remaining factor $d^{p/2}$, which arises from the input noise rather than initialization, remains unchanged.
\end{rem*}

\subsection{Strong model}

Task localization \citep{dai2022knowledge,panigrahi2023taskspecific} in pre-trained models motivates partitioning the strong model's neurons by task.
We model the strong model as a pre-trained two-layer network with $K$ groups of neurons, each group corresponding to one component $r^*_k$ of the additive reward:
\begin{align}\label{eq:strong-arch}
  r(x) = \sum_{k=1}^K r_k(x), \quad r_k(x) = \frac{1}{N_k}\sum_{n=1}^{N_k} a_{k,n}\,\sigma_{k,n}(w_{k,n}^\top x + b_{k,n}),
\end{align}
where $\|w_{k,n}\| = 1$. 
Unlike \citet{oko2024learningsumdiversefeatures}, where task localization emerges during training from isotropic initialization, we assume that the strong model is pre-trained and already localized into task-specific groups.
For task $\kappa$, the target neurons are initialized with $\Sigma_\kappa w^0_{\kappa,n}$ uniform on $S^{d-1} \cap V_\kappa$ and perpendicular component satisfying $\|\Sigma_\kappa^\perp w^0_{\kappa,n}\|^2 = o(s^{-1/2})$ (see \cref{asm:strong-init}).
For $k \neq \kappa$, the projection $\|\Sigma_\kappa w^0_{k,n}\|^2$ of a task-$k$ neuron onto $V_\kappa$ determines how much the task-$\kappa$ gradient update affects task-$k$ neurons.
Small values of this quantity preserve pre-trained features for $k \neq \kappa$, as formalized in \cref{prop:preservation-informal}.

The strong model generates responses from its pre-trained distribution, in line with recent W2S literature \citep{ji2024aligner, tao2025your} rather than the original setup of \citet{burns2023weaktostrong}.
We model the resulting embedding distribution as $\mu$ concentrating on $V_\kappa$, as an idealization of the pre-training prior that the strong model has internalized task-$\kappa$ knowledge.

\begin{asm}[Generative distribution]\label{asm:distribution}
  When the strong model is queried on task $\kappa$, we model its input distribution $\mu$ as the following mixture.
  \begin{align}
    \mu = \sum_{k=1}^K \lambda_k\, \mu_k, \quad (\lambda_k)_{k=1}^K \in \Delta^{K-1}, \quad 1 - \lambda_\kappa = o(s^{-1/2}),
  \end{align}
  where $\mu_\kappa = \mathcal{N}(0, \Sigma_\kappa)$ is the Gaussian supported on the target subspace $V_\kappa$, and for $k \neq \kappa$ each $\mu_k$ is a sub-Gaussian distribution supported on $V_k$ with identity covariance on $V_k$.
\end{asm}

With $1 - \lambda_\kappa = o(s^{-1/2})$, nearly all sampled embeddings lie in $V_\kappa$, reducing the effective dimension from $d$ to $s$ in the sample complexity of \cref{thm:strong-informal} and suppressing gradient updates to off-target neurons in \cref{prop:preservation-informal}.

\section{W2S learning}\label{sec:w2s}

\begin{algorithm}[t]
  \caption{W2S feature learning via online SGD}\label{alg:w2s}
  \DontPrintSemicolon
  \KwIn{Weak model $r^w$; initialized strong model $r_{(\Theta_k)_{k=1}^K}$ with $\Theta_k = (a_{k,n}, b_{k,n}, w^0_{k,n})_{n=1}^{N_k}$ (\cref{asm:strong-init}); learning rate schedule $(\eta^t)_{t \geq 0}$; iterations $T$.}
  \For{$t = 0, 1, \ldots, T - 1$}{
    Sample $x^t \sim \mu$ (\cref{asm:distribution}).\;
    $\bar{r}^w(x^t) \gets \clip\bigl(r^w(x^t),\; \pm 1/\log d\bigr)$.\;
    $y^t \gets \bar{r}^w(x^t)\,\exp\bigl(\bar{r}^w(x^t)\bigr)$.\label{line:transform}\tcp*{nonlinear transformation}
    \For{$k \in [K]$, $n \in [N_k]$}{
      $w_{k,n}^{t+1} \gets w_{k,n}^t + \eta^t\, y^t\, \tilde{\nabla}_{w_{k,n}} r_\Theta(x^t)$.\tcp*{spherical gradient}
      $w_{k,n}^{t+1} \gets w_{k,n}^{t+1} / \bigl\|w_{k,n}^{t+1}\bigr\|$.\;
      }
      }
      \KwOut{$\hat{w}_{k,n} \gets w_{k,n}^T$ for all $k \in [K]$, $n \in [N_k]$.}
    \end{algorithm}

  \cref{alg:w2s} trains the first-layer weights of the strong model via online SGD similar to \citet{benarous2021online,damian2024smoothing,lee2024neural,oko2024learningsumdiversefeatures}, using a nonlinear transformation of the weak model output as supervision and inputs drawn from the generative distribution $\mu$ (\cref{asm:distribution}).
  Each update takes a gradient step on the correlation loss $-y^t r_\Theta(x^t)$ with respect to $w_{k,n}$.
  The transformation on \cref{line:transform} clips the weak model output to $\pm 1/\log d$ and composes with $z \mapsto z\exp(z)$ analogous to \citet{nishikawa2025nonlinear}, placing the gradient update outside the correlational statistical query (CSQ) framework \citep{bshouty2002using} and exploiting $\GE(\sigma^*_\kappa)$.
  When $\sigma^*_\kappa$ is even, this reduces the effective exponent from $p$ to $2$.
  Together, this exponent reduction and the $s$-dimensional support of $\mu_\kappa$ (\cref{asm:distribution}) yield the $\tilde{O}(s^{3/2})$ complexity in \cref{thm:strong-informal}, which is lower than the $\tilde{O}(d^{p/2}\,s^{(p-2)/2})$ cost of the weak model (\cref{prop:weak-informal}).

\subsection{Main results}

\begin{thm}[W2S feature alignment; informal]\label{thm:strong-informal}
  Assume \cref{asm:additive,asm:weak,asm:distribution,asm:strong-init}.
  Let $\epsStrongEff = \epsWeakLOne \vee \epsWeakAlign \vee (1 - \lambda_\kappa) \lesssim s^{-1/2}$, where $\epsWeakLOne$ and $\epsWeakAlign$ are the $L^1$ approximation error and the alignment error of the weak model (\cref{asm:weak}).
  Fix a target accuracy $\tilde{\varepsilon} = \tilde{\Omega}(\epsStrongEff)$.
  Then \cref{alg:w2s} with
  \begin{align}\label{eq:strong-sample}
    T = \tilde{O}\!\left(s^{3/2} \vee s\,\tilde{\varepsilon}^{-1} \log{\tilde{\varepsilon}^{-1}} \vee \tilde{\varepsilon}^{-2} \log{\tilde{\varepsilon}^{-1}}\right)
  \end{align}
  iterations produces weights satisfying $\theta_\kappa^\top \hat{w}_{\kappa,n} \geq 1 - \tilde{\varepsilon}$ for $\Theta(N_\kappa)$ neurons, with high probability.
  Under the scaling $\tilde{\varepsilon} = \tilde{\Theta}(s^{-1/2})$, this reduces to $T = \tilde{O}(s^{3/2})$.
\end{thm}
\setcounter{num@thm@strong}{\value{thm}}
\begin{rem*}
The condition $\tilde{\varepsilon} = \tilde{\Omega}(\epsStrongEff)$ means the strong model cannot achieve alignment error below $\epsStrongEff$.
The three components of $\epsStrongEff$ reflect three sources of error: the $L^1$ approximation error $\epsWeakLOne$ and alignment error $\epsWeakAlign$ of the weak model (\cref{asm:weak}), and the mixture weight $1 - \lambda_\kappa$ on non-target components of $\mu$ (\cref{asm:distribution}).
\cref{asm:strong-init}\,(i) assumes $\sigma^*_\kappa$ is even, giving $\GE(\sigma^*_\kappa) = 2$; this is essential for feature preservation, as discussed in \cref{rem:ge-assumption}.
Under \cref{asm:distribution}, inputs concentrate on the $s$-dimensional subspace $V_\kappa$, so the sample complexity $\tilde{O}(s^{3/2})$ depends only on $s$, not on $d$, which is lower than the $\tilde{O}(d^{p/2}\,s^{(p-2)/2})$ cost of the weak model (\cref{prop:weak-informal}).
When $s < d^{2/3}$, this complexity $\tilde{O}(s^{3/2})$ improves on the information-theoretic lower bound $\tilde{\Omega}(d)$ for Gaussian single-index models \citep{dudeja2024statistical,damian2024computationalstatistical}.
The same holds unconditionally against the CSQ lower bound $\tilde{\Omega}(d^{p/2})$ \citep{damian2022neural}, since $s \leq d$ and $p > 2$.
The gap between $\tilde{O}(s^{3/2})$ and the information-theoretic limit $\tilde{\Omega}(s)$ in $\mathbb{R}^s$ reflects the cost of imperfect concentration of $\mu$ on $V_\kappa$ ($\lambda_\kappa < 1$); see \cref{rem:complexity}.

\citet{burns2023weaktostrong} find that early stopping is necessary to prevent overfitting to the weak teacher, motivating our analysis of feature preservation. 
\citet{medvedev2025weaktostrong} provide a theoretical account of this phenomenon in the kernel regime. 
The following proposition shows that \cref{alg:w2s} prevents forgetting of off-target features at termination, providing a theoretical justification for early stopping.
\end{rem*}

\begin{prop}[Preservation of pre-trained features; informal]\label{prop:preservation-informal}
  Under the conditions of \cref{thm:strong-informal}, let $\chi_k = N_\kappa \pi_k / (N_k \pi_\kappa)$ for $k \neq \kappa$.
  If $\chi_k \leq 1/\mathrm{polylog}(d)$, then at termination of \cref{alg:w2s}, with high probability, for every $k \neq \kappa$ and every $n \in [N_k]$,
  \begin{align}\label{eq:preservation}
    \theta_k^\top \hat{w}_{k,n} \geq \theta_k^\top w^0_{k,n}
    - \tilde{O}\!\Bigl(
      \underbrace{\|\Sigma_\kappa w^0_{k,n}\|^2}_{\text{init.\ overlap with } V_\kappa}
      + \tilde{\varepsilon}
      + \underbrace{\chi_k s^{-1/2}}_{\text{SGD noise}}
    \Bigr).
  \end{align}
  where $\tilde{\varepsilon} \gtrsim \epsStrongEff$ is as in \cref{thm:strong-informal}.
  The full bound is given in \cref{app:strong}.
\end{prop}
\setcounter{num@prop@preservation}{\value{prop}}

The factor $\chi_k$ measures the localization of task $k$ relative to task $\kappa$ and controls the magnitude of the gradient which is induced on off-target neurons. 
The condition $\chi_k \lesssim 1/\mathrm{polylog} (d)$ is satisfied when task $k$ occupies more neurons than task $\kappa$, carries smaller weight $\pi_k$ in the reward, or both.

\paragraph{W2S Generalization.}
\cref{thm:strong-informal,prop:preservation-informal} together establish that \cref{alg:w2s} achieves feature elicitation on task $\kappa$ and preservation of pre-trained features for $k \neq \kappa$ simultaneously.
The strong model surpasses the weak model in that it acquires $\theta_\kappa$ via W2S training while retaining its initial knowledge of off-target tasks $k \neq \kappa$, which the weak model, specialized to task $\kappa$ alone, does not possess.
For instance, task $\kappa$ may correspond to a specialized alignment criterion, while off-target tasks $k \neq \kappa$ represent broader capabilities related to $\kappa$ acquired during pre-training.
Our results are complementary to analyses that assume fixed strong-model representations \citep{charikar2024quantifying, xue2025representations, dong2025discrepancies, wu2025provable, medvedev2025weaktostrong}, in that we establish that pre-trained features are elicited through W2S training, whereas these analyses treat feature quality as given.
Unlike \citet{moniri2025on} and \citet{oh2025from}, our analysis derives how a subspace prior enables feature elicitation under weak supervision from gradient-based optimization, tracking multi-step dynamics in a multi-neuron network and yielding quantitative bounds on both the gain in target alignment and the change in off-target feature directions.
The bound on forgetting is enabled by the localization parameter $\chi_k$, which emerges from the multi-neuron architecture and has no counterpart in prior works.

\begin{rem*}
  Empirically, \citet{burns2023weaktostrong} find that generative finetuning, which trains the strong model on task-related text via a language modeling objective without human labels, prior to W2S training, improves W2S generalization.
  This can be interpreted as concentrating the strong model's generative distribution on task-$\kappa$ inputs, corresponding to the concentration $1 - \lambda_\kappa = o(s^{-1/2})$, which limits the achievable alignment accuracy in \cref{thm:strong-informal}.
  Generative finetuning may, however, pull off-target neurons toward $V_\kappa$, increasing $\|\Sigma_\kappa w^0_{k,n}\|^2$ and worsening the bound in \cref{prop:preservation-informal}.
\end{rem*}

\subsection{Proof sketch of Theorem~\ref{thm:strong-informal}}
Intuitively, the gradient dynamics keep $w^t_{\kappa,n}$ approximately within $V_\kappa$ throughout training, so $\theta_\kappa$ is learned without leakage outside the subspace.
Two quantities govern the training dynamics, the alignment $z^t := \theta_\kappa^\top w^t_{\kappa,n}$ and the complement-subspace deviation $D^t := \|\Sigma_\kappa^\perp w^t_{\kappa,n}\|^2$.
The following lemma gives a decomposition of the negative gradient $g^t_\kappa$ with respect to $w^t_{\kappa,n}$.

\begin{lem*}[Gradient decomposition; informal, \cref{lem:strong-onestep-1}]
  \label{lem:gradient-decomp-informal}
    $g^t_\kappa=y^t \tilde{\nabla}_{w^t_\kappa} r_\Theta(x^t)$ decomposes as
  \begin{align}
    g^t_\kappa = \bar{\alpha}_2\beta_{\kappa,n,2}\lambda_\kappa z^t\,\theta_\kappa + Z^t + R^t,
  \end{align}
  where $\bar{\alpha}_2, \beta_{\kappa,n,2}$ are the second Hermite coefficients of $\bar{r}^w\exp(\bar{r}^w)$ and $a_{\kappa,n}\sigma_\kappa(\cdot + b_{\kappa,n})$, $Z^t$ is mean-zero, and $\|R^t\| = \tilde{O}((\epsStrongEff \vee D^t)\,s^{1/2})$.
\end{lem*}

Since $\GE(\sigma^*_\kappa)=2$, \cref{lem:strong-transformation-2} gives $\bar{\alpha}_2 \neq 0$, so the alignment satisfies $z^{t+1} \approx (1 + \eta^t\bar{\alpha}_2\beta_{\kappa,n,2}\lambda_\kappa)\,z^t$.
By contrast, online SGD on $r^*_\kappa$ has drift $(z^t)^{p-1}$ \citep{benarous2021online}.
The W2S signal reduces this to a linear drift, which accounts for the improved sample complexity.
The following lemma bounds $D^t$ over each sub-interval of training.

\begin{lem*}[Subspace deviation bound; informal, \cref{lem:strong-deviation-1}]
  \label{lem:deviation-informal}
  For any step $t$ with $z^t \geq 0$, $D^{t+1} - D^t = \tilde{O}(\eta^t)$.
  Moreover, over any sub-interval $[\tau, \tau + \tau']$ with $z^t \geq 0$ for all $t \in [\tau, \tau + \tau']$, with probability $1 - \delta'$, $D^{\tau+t} \leq Q^t$ for all $t = 0, \dots, \tau'$, where
  \begin{align}
    Q^0 &= \underbrace{D^\tau}_{\text{initial}} 
     + \underbrace{\tilde{O}(\eta\,\epsStrongEff\,\tau')}_{\text{weak-model error}} 
     + \underbrace{\Theta((1-\lambda_\kappa)^{1/2}\delta'^{-1/2}\eta\tau'^{1/2} + (1-\lambda_\kappa)\delta'^{-1}\eta^2 s\tau')}_{\text{off-subspace perturbation}},\\
    Q^t &= Q^{t-1} + \tilde{O}(\eta(Q^{t-1})^2).
  \end{align}
  Here $\delta' = 1/\operatorname{polylog} d$.
\end{lem*}

Controlling this bound requires Doob's maximal inequality rather than standard concentration, which forces a smaller learning rate and drives the overall complexity to $\tilde{O}(s^{3/2})$; see \cref{rem:complexity} for details.
A large $D^t$ reduces the effective drift on $z^t$ to $\bar{\alpha}_2\beta_{\kappa,n,2}\lambda_\kappa z^t - O(D^t)$ (\cref{lem:strong-onestep-1}), so the condition $D^t = O(z^t)$ must be maintained throughout training.
Since the per-step increment of $D^t$ is $O(\eta (D^t)^2)$ (\cref{lem:strong-deviation-1}) while $z^t$ grows at rate $\Omega(\eta z^t)$ (\cref{lem:strong-weakrecovery-1,lem:strong-amplification-1}), the ratio $D^t/z^t$ is non-increasing during the initial phase where $z^t = o(1)$.
\Cref{lem:strong-weakrecovery-1,lem:strong-amplification-1} shows that the condition $D^t = o(z^t)$ is maintained throughout each sub-interval with probability $1 - o(1)$, using the bound on $D^t$ established in \cref{lem:strong-deviation-1}.

\paragraph{Alignment dynamics.} Three stages complete the argument (\cref{lem:strong-weakrecovery-1,lem:strong-amplification-1,lem:strong-strongrecovery-1,lem:strong-strongrecovery-2,lem:strong-strongrecovery-3}), following \citet{lee2024neural,oko2024learningsumdiversefeatures}.
In the first two stages the learning rate is $\eta_1 \lesssim s^{-3/2}$. In the third it is reduced to $\eta_2 \lesssim \tilde{\varepsilon}s^{-1} \wedge \tilde{\varepsilon}^2$.
In weak recovery (\cref{lem:strong-weakrecovery-1}), starting from $z^0 \geq s^{-1/2}$ (\cref{asm:strong-init}), each interval of $\tau_1 = O(s^{-1/2}\eta_1^{-1})$ steps multiplies $z^t$ by $1 + \tfrac{1}{2}\bar{\alpha}_2\beta_{\kappa,n,2}\eta_1\tau_1 = 1 + \Theta(1)$, so $O(s^{1/2}\log s)$ intervals suffice to reach $z^t = 1/\operatorname{polylog} d$, totaling $T_1 = \tilde{O}(\eta_1^{-1})$ steps.
In alignment amplification (\cref{lem:strong-amplification-1}), $z^t$ is driven from $1/\operatorname{polylog} d$ to $1 - o(1)$ in $T_2 = \tilde{O}(\eta_1^{-1})$ steps.
In strong recovery (\cref{lem:strong-strongrecovery-1,lem:strong-strongrecovery-4}), the Hermite structure condition in \cref{asm:strong-init} ensures that $g^t_\kappa$ has a component along $-\Sigma_\kappa^\perp w^t_{\kappa,n}$ of magnitude $\Omega(z^t)$, so the gradient update shrinks $\|\Sigma_\kappa^\perp w^{t+1}_{\kappa,n}\|$ by a factor $1 - \Omega(\eta_2)$ when $z^t \geq 1/2$. The following lemma quantifies this decay.
\begin{lem*}[Subspace deviation decay; informal, \cref{lem:strong-strongrecovery-1}]
  Suppose $\theta_\kappa^\top w^t_{\kappa,n} \geq 1/2$ and $\eta_2 \lesssim \tilde{\varepsilon}s^{-1} \wedge \tilde{\varepsilon}^2$.
  Then
  \[
    D^{t+1}
    \leq \left(1 - o(\eta_2)\right)D^t + o(\tilde{\varepsilon}).
  \]
\end{lem*}
Consequently, $D^t$ decays exponentially to $O(\tilde{\varepsilon})$ within $\tilde{O}(\eta_2^{-1}\log\tilde{\varepsilon}^{-1})$ steps (\cref{lem:strong-strongrecovery-2,lem:strong-strongrecovery-3}).
With $D^t = O(\tilde{\varepsilon})$ maintained, $z^t$ converges to $1 - O(\tilde{\varepsilon})$ within $T_3$ steps (\cref{lem:strong-strongrecovery-2,lem:strong-strongrecovery-3}).
The total cost $T_1 + T_2 + T_3 = \tilde{O}(s^{3/2} \vee s \tilde{\varepsilon} \log \tilde{\varepsilon}^{-1} \vee \tilde{\varepsilon}^{-2} \log \tilde{\varepsilon}^{-1})$ gives \cref{eq:strong-sample}.

\begin{rem}[Why $s^{3/2}$ rather than $s$]\label{rem:complexity}
  For online SGD on a single-index target with $\mathrm{IE}=2$, the standard complexity is $\tilde{O}(s)$ at learning rate $\eta\lesssim s^{-1}$ \citep{benarous2021online}.
  The difficulty stems from the heavy-tailed structure of $\Sigma_\kappa^\perp Z^t$. When $x^t \sim \mu_k$ for $k \neq \kappa$, the sample falls outside $V_\kappa$ and causes an $O(1)$ perturbation, even though this event has probability $1 - \lambda_\kappa = O(s^{-1/2})$.
  Concentration inequalities control sums of martingale differences only through their per-step maximum, so the rare $O(1)$ jumps dominate and the resulting bound is too loose.
  Doob's maximal inequality, applied to a supermartingale in \cref{lem:strong-deviation-1}, instead exploits the low probability of the large-increment event and yields a tight bound.
  Since Doob's inequality bounds the running maximum of each martingale by its cumulative variance, the bound becomes vacuous if the interval is too long. The horizon is therefore partitioned into $\tilde{O}(s^{1/2})$ sub-intervals and a union bound is taken over all of them.
  This forces $\eta\lesssim s^{-3/2}$ rather than $s^{-1}$, yielding $\tilde{O}(s^{3/2})$ steps.
\end{rem}

\section{SFT forgetting}\label{sec:sft}

While \cref{prop:preservation-informal} shows that W2S preserves off-target features, we show that SFT does not.
This is consistent with empirical observations of catastrophic forgetting in large language models during fine-tuning \citep{wang2024two, kotha2024understanding, luo2025empirical}.
In our setting, SFT on task $\kappa$ applies online SGD to the same strong model (\cref{eq:strong-arch}) with isotropic inputs $x^t \sim \mathcal{N}(0, I_d)$ and direct teacher labels $y^t = r^*_\kappa(\theta_\kappa^\top x^t) + \zeta^t$ as in \citet{benarous2021online}.
The first-layer update is
\begin{align}\label{eq:sft-update}
  w_{k,n}^{t+1} \gets w_{k,n}^t + \eta^t\, y^t\, \tilde{\nabla}_{w_{k,n}} r_\Theta(x^t), \quad w_{k,n}^{t+1} \gets w_{k,n}^{t+1}/\|w_{k,n}^{t+1}\|.
\end{align}
\cref{thm:sft} confirms that neurons initialized toward $\theta_k$ are driven to align with $\theta_\kappa$ whenever $\theta_k^\top\theta_\kappa$ is non-negligible, so the pre-trained feature $\theta_k$ is lost before task $\kappa$ is learned. The role of feature similarity in causing forgetting is also examined by \citet{hiratani2024disentangling} in the continual learning setting.

\begin{thm}[SFT forgetting; informal]\label{thm:sft}
  Assume \cref{asm:additive,asm:strong-init}.
  Let $t_{1,\kappa}$ be the first time at which $\theta_\kappa^\top w^t_{\kappa,n} \geq o(1)$ for some $n \in [N_\kappa]$.
  Under \cref{alg:sft-forgetting}, with high probability $\theta_\kappa^\top w^t_{k,n} \geq 1 - o(1)$ for every neuron $n \in [N_k]$ satisfying the sign condition, all $t \geq t_{1,\kappa}$, and every $k \neq \kappa$ satisfying \cref{asm:sft-init} (see \cref{app:sft} for full conditions).
\end{thm}
\setcounter{num@thm@sft}{\value{thm}}

The off-target neurons thus reach $\theta_\kappa$ no later than the task-$\kappa$ neurons, so the alignment $\theta_k^\top w^t_{k,n}$ with $\theta_k$ is lost once task $\kappa$ is learned.
Under the random initialization in \cref{strong-initialization}, $\Theta(N_k)$ neurons satisfy the sign condition, so a constant fraction of task-$k$ neurons is affected.
Since $r^w$ carries no signal for $k \neq \kappa$, this loss of alignment is irrecoverable, affecting the approximation quality of $r^*_k$.

\emph{Proof idea.}\; 
The base online SGD analysis follows the same framework as \cref{prop:weak-informal,thm:strong-informal} and prior work \citep{benarous2021online, damian2024smoothing, lee2024neural, oko2024learningsumdiversefeatures}.
Under isotropic sampling the gradient propagates to every neuron, and the Hermite expansion of the direct teacher has $\IE(\sigma^*_\kappa) = p$, so the drift on $\theta_\kappa^\top w^t_{k,n}$ scales as $\chi_k (\theta_\kappa^\top w^t_{k,n})^{p-1}$.
By \cref{asm:sft-init}, $\theta_\kappa^\top w^0_{k,n}$ exceeds the $\Theta(s^{-1/2})$ alignment of the target neurons $w^0_{\kappa,n} \in V_\kappa$.
A Bihari--LaSalle ODE comparison (formalized in \cref{app:sft}) shows that the larger initial condition forces $\theta_\kappa^\top w^t_{k,n}$ to reach $1 - o(1)$ before $\theta_\kappa^\top w^t_{\kappa,n}$ does, so the alignment $\theta_k^\top w^t_{k,n}$ with the pre-trained feature is lost.

\paragraph{W2S versus SFT.}
Our results establish that W2S improves over SFT in both sample complexity and feature preservation.
\cref{thm:strong-informal,prop:preservation-informal,thm:sft} share the same model architecture, so the difference between W2S and standard SFT is purely algorithmic.
Unlike SFT, which trains on a fixed labeled dataset, W2S treats the weak model as an oracle and queries it on inputs drawn from the strong model's generative distribution $\mu$.
W2S further applies a nonlinear transformation to the weak model output, achieving feature elicitation with $\tilde{O}(s^{3/2})$ samples, whereas SFT under isotropic sampling requires $\tilde{O}(d^{p-1})$ samples \citep{benarous2021online}.
Under the intended setting where task-$k$ neurons satisfy $\theta_k^\top w^0_{k,n} \approx 1$, \cref{asm:sft-init} reduces to $\theta_k^\top \theta_\kappa \gtrsim s^{-1/2}$, meaning forgetting occurs when tasks $k$ and $\kappa$ are more similar than two random directions in the $s$-dimensional subspace.
Even in this setting, \cref{prop:preservation-informal} guarantees preservation of $\theta_k$, since $\|\Sigma_\kappa w^0_{k,n}\|^2 \gtrsim (\theta_\kappa^\top \theta_k)^2 \gtrsim s^{-1}$, in contrast to SFT, under which the alignment $\theta_k^\top w^t_{k,n}$ with the pre-trained feature is lost.

\begin{rem*}
W2S succeeds in a setting where SFT cannot, even when SFT applies the same nonlinear teacher transformation as \cref{alg:w2s}.
Since SFT trains on a fixed labeled dataset without a weak model, it draws inputs isotropically from $\mathcal{N}(0, I_d)$; any such algorithm requires $\tilde{\Omega}(d)$ samples to learn a feature \citep{dudeja2024statistical,damian2024computationalstatistical}.
When $s < d^{2/3}$, the $\tilde{O}(s^{3/2})$ complexity of \cref{alg:w2s} is below $d$, so W2S learns the feature from fewer samples than SFT can, regardless of its teacher transformation.
\end{rem*}

\section{Numerical experiments}\label{sec:experiments}

We verify \cref{thm:strong-informal,prop:preservation-informal,thm:sft} numerically with input dimension $d = 1024$, subspace dimension $s = 128$, number of tasks $K = 2$ with target task $\kappa = 1$, neuron counts $N_1 = 64$ and $N_2 = 512$, link function $\sigma^*_k = \mathrm{He}_4$ ($\IE = 4$, $\GE = 2$), feature correlation $\theta_1^\top\theta_2 = 0.3$, and mixture weight $\lambda_\kappa = 0.9$, so that $1 - \lambda_\kappa = 0.1 \lesssim s^{-1/2}$ as required by \cref{asm:distribution}. 
The weak model is a two-layer network with $N^w = 3$ neurons whose coefficients $a_n$ and biases $b_n$ are set to the true values of $r^*_1$ plus independent $\mathcal{N}(0, 0.1^2)$ noise, and whose directions $w_n$ are similarly perturbed and then divided by $\sqrt{d}$; this construction is intended to satisfy \cref{asm:weak}.
The strong model is a two-layer network whose coefficients $a_{k,n}$ are initialized to $\pm 1$ and whose biases $b_{k,n}$ are initialized to $0$. 
Its directions $w_{k,n}$ are sampled uniformly on $V_k$ and then perturbed by a small additive term. 
For $k = 1$, the perturbation is drawn from $\mathcal{N}(0, 0.1^2 \Sigma_1^\perp)$ and normalized by $\sqrt{d-s}$ before being added. For $k = 2$, it is drawn from $\mathcal{N}(0, I_d)$ and normalized by $\sqrt{d}$. 
Neuron activations are drawn from $\{\pm\mathrm{He}_2 \pm \mathrm{He}_4\}$, chosen so that \cref{asm:strong-init} is satisfied. 
The learning rate is $\eta = 0.2 \sim s^{-3/2}$ for W2S and $\eta = 7.5 \times 10^{-5} \sim d^{-p/2}$ for SFT, matching the theoretical scalings. 
\cref{fig:main} plots the per-neuron alignment magnitude $|\theta_k^\top w^t_{k,n}|$ over training steps.

\cref{fig:main} (top) verifies \cref{thm:strong-informal,prop:preservation-informal} under W2S. 
The left panel shows that task-1 neurons ($\kappa = 1$, target task) satisfying the sign condition of \cref{prop:weak-informal} with initial alignment $|\theta_1^\top w^0_{1,n}| \geq s^{-1/2}$ converge to alignment $\approx 0.95$ with $\theta_1$, in agreement with \cref{thm:strong-informal}. 
The right panel shows that task-2 neurons (non-target) maintain alignment $\approx 1.0$ with $\theta_2$ throughout training without drifting toward $\theta_1$, in agreement with \cref{prop:preservation-informal}.
\cref{fig:main} (bottom) verifies \cref{thm:sft} under SFT, with both panels tracking alignment with $\theta_1$. 
The left panel shows that task-1 (target) neurons have not yet converged at $T = 2 \times 10^7$, consistent with the $\tilde{O}(d^2 s) \approx 10^8$ sample complexity predicted by \cref{prop:weak-informal}. 
The right panel shows that task-2 neurons reach alignment $\approx 0.8$ with $\theta_1$ before the target neurons have fully converged, so forgetting of the pre-trained feature $\theta_2$ precedes completion of the target task.

\begin{figure}[t]
  \centering
  \includegraphics[width=0.65\linewidth]{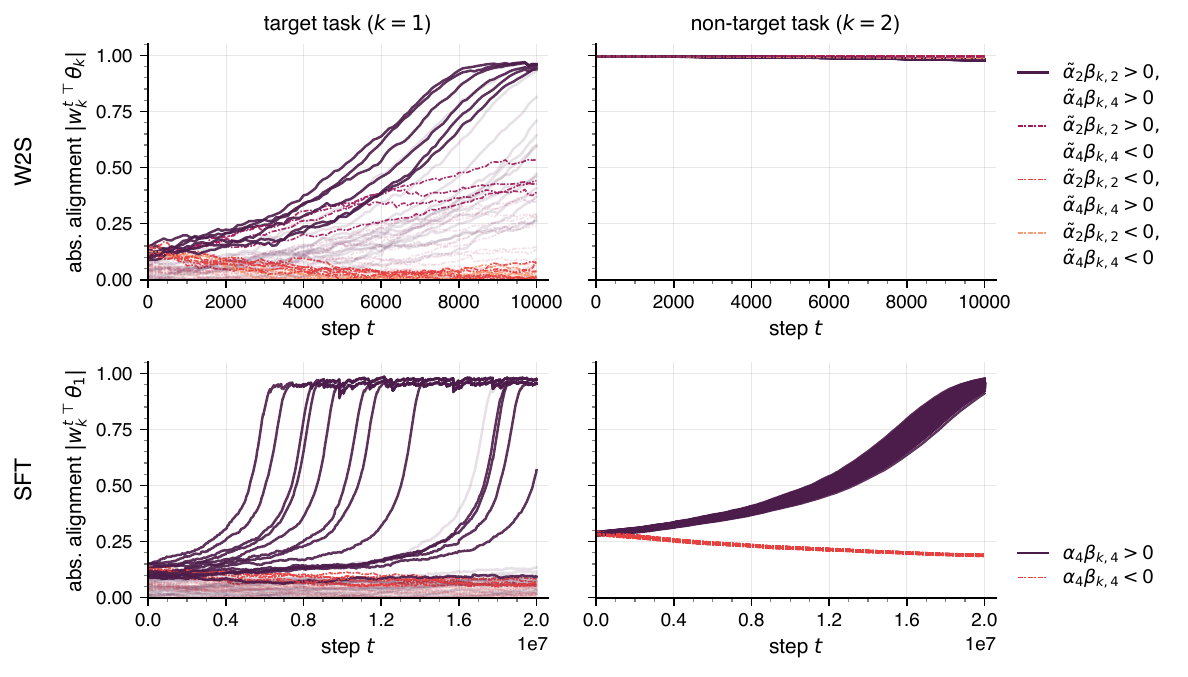}
  \caption{Per-neuron alignment magnitude $|\theta_k^\top w^t_{k,n}|$ during training ($d = 1024$, $s = 128$, $K = 2$, $\sigma^*_k = \mathrm{He}_4$, $\theta_1^\top\theta_2 = 0.3$). Line colors distinguish neuron types by the signs of $\tilde{\alpha}_2\beta_{k,2}$ and $\tilde{\alpha}_4\beta_{k,4}$ (W2S) or $\tilde{\alpha}_4\beta_{k,4}$ alone (SFT); neurons with initial alignment magnitude $|\theta_k^\top w^0_{k,n}| < s^{-1/2}$ are shown semi-transparent.
  \textbf{Top (W2S, $\eta = 0.2$, $T = 10000$):} Absolute alignment $|\theta_1^\top w^t_{1,n}|$ with the target feature (left) and $|\theta_2^\top w^t_{2,n}|$ with the off-target feature (right).
  \textbf{Bottom (SFT, $\eta = 7.5 \times 10^{-5}$, $T = 2 \times 10^7$):} Absolute alignment $|\theta_1^\top w^t_{k,n}|$ for both task-1 (left) and task-2 (right) neurons.}
  \label{fig:main}
\end{figure}

\section{Discussion}\label{sec:discussion}

\paragraph{Limitations.}
Second-layer learning for the strong model is not formally analyzed.
Our results characterize W2S generalization in terms of feature alignment rather than function approximation, and whether second-layer learning preserves reward quality for off-target tasks remains open. This limitation is shared with \citet{moniri2025on}, who similarly establishes preservation at the feature level.
\cref{thm:sft} is algorithm-specific rather than an information-theoretic lower bound, so forgetting can likely be avoided by other algorithms with explicit regularization.
The generative-exponent reduction requires $\sigma^*_\kappa$ to be even, which excludes odd link functions.

\paragraph{Future work.}
Extending second-layer learning to preserve off-target features remains open, with two obstacles.
Following \citet{oko2024learningsumdiversefeatures}, the second layer can be learned by ridge regression, but since $r^w \approx r^*_\kappa$, the supervision provides no signal for off-target neurons, driving $a_{k,n} \to 0$ for $k \neq \kappa$.
Moreover, standard practice resets the bias parameters $b_{k,n}$ before second-layer training, destroying $r_k$ preserved by first-layer W2S training.
A penalty on the deviation of second-layer parameters from their initialization may resolve both obstacles, analogous to the KL regularization used in LLM alignment.
The analysis of deeper architectures such as Transformers is also a natural future direction.

\begin{ack}
RA was partially supported by JSPS KAKENHI (25H01107).
TS was partially supported by JSPS KAKENHI (24K02905) and JST CREST (JPMJCR2015).
This research is supported by the National Research Foundation, Singapore and the Ministry of Digital Development and Information under the AI Visiting Professorship Programme (award number AIVP-2024-004). Any opinions, findings and conclusions or recommendations expressed in this material are those of the author(s) and do not reflect the views of National Research Foundation, Singapore and the Ministry of Digital Development and Information.
\end{ack}

\bibliographystyle{plainnat}
\bibliography{reference}


\newpage


\newpage
\appendix

\newpage
\tableofcontents
\newpage

\section{Notations}\label{app:notation}

\begin{center}
\begin{tabular}{@{}lp{0.72\linewidth}@{}}
\toprule
Symbol & Meaning \\
\midrule
\multicolumn{2}{@{}l}{\textit{Geometry and tasks}} \\
$d, K, \kappa$ & Ambient input dimension, number of tasks, and target task index. \\
$s = d^\alpha$ & Dimension of each task subspace $V_k$, with $\alpha \in (0,1)$. \\
$\theta_k$ & True hidden feature direction for task $k$; $\theta_k \in V_k \cap S^{d-1}$. \\
$\Sigma_k$ & Orthogonal projection onto $V_k$. \\
$\Sigma_k^\perp$ & Projection onto the orthogonal complement, $\Sigma_k^\perp = I - \Sigma_k$. \\[4pt]
\multicolumn{2}{@{}l}{\textit{Reward model}} \\
$\pi_k$ & Reward weight for task $k$. \\
$\sigma^*_k$ & True link function for task $k$; a degree-$q$ polynomial. \\
$p, q$ & Information exponent $p = \IE(\sigma^*_\kappa)$ and common degree of $\sigma^*_k$, $\sigma^w_n$, $\sigma_{k,n}$. \\
$r^*(x)$ & True total reward, $r^*(x) = \sum_{k=1}^K \pi_k \sigma^*_k(\theta_k^\top x)$. \\[4pt]
\multicolumn{2}{@{}l}{\textit{Weak and strong models}} \\
$r^w(x),\, r(x)$ & Outputs of the weak model and the strong model, respectively. \\
$N^w, N_k$ & Width of the weak model and of the $k$-th neuron group of the strong model. \\
$\sigma^w_n,\, \sigma_{k,n}$ & Activation of neuron $n$ in the weak model and in the $k$-th group of the strong model. \\
$w^w_n, a^w_n, b^w_n$ & First-layer weight, second-layer coefficient, and bias of neuron $n$ in the weak model. \\
$w^t_{k,n}, a_{k,n}, b_{k,n}$ & First-layer weight, second-layer coefficient, and bias of neuron $n$ in group $k$ of the strong model at step $t$. \\
$\bar{r}^*_\kappa,\, \rho$ & Clipped teacher signal $\bar{r}^*_\kappa = \mathop{\mathrm{clip}}(\rho^{-1} r^*_\kappa;\, \pm 1/\log d)$ and its temperature $\rho = \Theta(\log^{C_\rho} d)$. \\
$\alpha_i, \bar{\alpha}_i$ & $i$-th Hermite coefficients of $r^*_\kappa$ and $\bar{r}^*_\kappa \exp(\bar{r}^*_\kappa)$. \\
$\beta_{n,i}$ & $i$-th Hermite coefficient of $a_n\sigma_n(\cdot + b_n)$. \\
$\tilde{\beta}_{k,n,i}$ & $i$-th Hermite coefficient of $a_{k,n}\sigma_{k,n}(\cdot + b_{k,n})$ (before rescaling). \\
$\beta_{k,n,i}$ & Rescaled coefficient, $\beta_{k,n,i} = \pi_k^{-1}\tilde{\beta}_{k,n,i}$. \\[4pt]
\multicolumn{2}{@{}l}{\textit{Training dynamics}} \\
$\eta^t, T$ & Learning rate at step $t$ and total iteration count. \\
$\tilde{\nabla}_w$ & Riemannian gradient on $S^{d-1}$ at $w$, defined by $\tilde{\nabla}_w f = (I - ww^\top)\nabla_w f$. \\
$g^t_k$ & Rescaled negative gradient with respect to $w^t_{k}$ at step $t$. \\
$\chi_k$ & Cross-task gradient ratio $\chi_k = N_\kappa \pi_k / (N_k \pi_\kappa)$; governs how task-$k$ neurons respond to target-task supervision. \\[4pt]
\multicolumn{2}{@{}l}{\textit{Error and constants}} \\
$\varepsilon_w,\, \tilde{\varepsilon}_w$ & Weak-model $L^1$ error and alignment slack (\cref{asm:weak}). \\
$\mu, \mu_k, \lambda_k$ & Oracle input distribution $\mu = \sum_k \lambda_k \mu_k$, where $\mu_k$ is the task-$k$ probability measure on $V_k$ and $(\lambda_k)_{k=1}^K \in \Delta^{K-1}$. \\
$C_a, C_b$ & Universal bounds on neuron parameters, $O(1)$ and $\tilde{O}(1)$ respectively. \\
$c_\beta, C_\beta$ & Constants bounding the Hermite coefficients of neuron activations, $c_\beta^{-1}, C_\beta = \mathrm{polylog}(d)$; see \cref{asm:complexity-weak,asm:strong-init}. \\
\bottomrule
\end{tabular}
\end{center}

\newpage
\section{Additional related work}\label{app:related}

\emph{Linear representation hypothesis.}
The linear representation hypothesis posits that high-level concepts are encoded as linear subspaces in the representation space of neural networks \citep{mikolov2013linguistic, arora2016latentvariable, elhage2022toymodelssuperposition, park2024thelinear}.
This structure has been observed in Transformer LLMs across diverse domains, including relational knowledge \citep{merullo2024language, hernandez2024linearity} and spatial and temporal concepts \citep{gurnee2024language}.
A related body of work shows that these representations can be linearly probed and used to steer model behavior \citep{nanda2023emergent, turner2023steering}.
At scale, \citet{beaglehole2026toward} extract linear subspace representations for over 500 concepts in large-scale AI models, providing direct evidence for our assumption that the strong model's pre-training knowledge is organized into task subspaces $V_k$.

\emph{Feature learning for single-index and additive models.}
Online-SGD analyses of single-index models \citep{benarous2021online, ba2022high, mousavi2022neural, ba2023learning, mousavi2023gradient, moniri2024theory, mahankali2024beyond, berthier2024learning, damian2024smoothing, glasgow2025propagation} establish the information-exponent sample complexity $\tilde{O}(d^{p-1})$ for $p > 2$, with analogous guarantees in the misspecified setting \citep{oko2024learningsumdiversefeatures}.
A line of recent work \citep{chen2020learning, dandi2024the, lee2024neural, arnaboldi2024repetita, joshi2024on, damian2024computationalstatistical} shows that nonlinear transformations of the teacher signal replace the information exponent with the generative exponent, circumventing the $\Omega(d^{p/2})$ correlational statistical query (CSQ) \citep{bshouty2002using} lower bound of \citet{damian2022neural}; \cref{alg:w2s} exploits this mechanism.
In particular, \citet{chen2020learning} achieve $\tilde{O}(d)$ sample complexity for polynomial link functions, near the information-theoretic limit $\tilde{\Omega}(d)$ \citep{dudeja2024statistical, damian2024computationalstatistical}.
Additive-model extensions, which underlie our reward model, appear in \citet{oko2024learningsumdiversefeatures, simsek2025learning, ren2025emergence, benarous2025learning}.

\section{Experimental setup details}\label{app:exp-setup}

Both experiments were run on a shared laboratory cluster equipped with an Intel Xeon E5-2680 v4 processor (28 cores, 2.40\,GHz) and 503\,GB RAM, using CPU only.
The W2S experiment ($T = 10{,}000$ steps) completed in a few minutes and the SFT experiment ($T = 2 \times 10^7$ steps) required approximately 10 hours.
No preliminary or failed runs required substantially more compute than the reported results.

\section{Full proof of \texorpdfstring{\cref{prop:weak-informal}}{Proposition~\ref*{prop:weak-informal}}: weak model construction}\label{app:weak}

We prove \cref{prop:weak}.
The arguments in \cref{weak-initialization}, \cref{weak-decomposition}, \cref{weak-amplification}, \cref{weak-strongrecovery}, and \cref{weak-secondlayer} build on the discussion in \citet{oko2024learningsumdiversefeatures}.

\begin{asm}\label{asm:complexity-weak}
    When training the weak model, we impose the following conditions.
    \begin{enumerate}[leftmargin=*]
        \item [(i)] The weak model learns the single-index model $r^*_\kappa = \pi_\kappa \sigma^*_\kappa$ from observations
        \begin{align}
            x \sim \mathcal{N}(0, I_d), \quad y = r^*_\kappa(x) + \zeta_\kappa, \quad \zeta_\kappa \sim \mathcal{N}(0, \pi_\kappa^2),
        \end{align}
        where $r^*_\kappa$ is the reward for task $\kappa \in [K]$ defined in \cref{asm:additive}.
        \item [(ii)] The weak model $r^w_\Theta$ is initialized with $\tilde{N}^w \gtrsim N^w$ neurons:
            \begin{align}
                \frac{1}{\tilde{N}^w} \sum_{n=1}^{\tilde{N}^w} a^0_n \sigma^w_n(w^{0\top}_n x + b^0_n).
            \end{align}
            Each first-layer weight satisfies $\|w^0_n\| = 1$ and is drawn uniformly from the subspace $V_\kappa$, i.e., $w^0_n \sim \mathrm{Unif}(S^{d-1} \cap V_\kappa)$.
            Writing $r^*_\kappa(\cdot) = \sum_{i=p}^q \frac{\alpha_i}{\sqrt{i!}} \mathrm{He}_i(\cdot)$ and $a^0_n \sigma^w_n(\cdot + b_n) = \sum_{i=0}^{q} \frac{\beta_{n, i}}{\sqrt{i!}} \mathrm{He}_i(\cdot)$, the following hold for $\Theta(\tilde{N}^w)$ neurons:
            \begin{align}
                &\pi_\kappa c_\beta \leq |\beta_{n,p}| \leq \pi_\kappa C_\beta, \quad |\beta_{n,i}| \leq \pi_\kappa C_\beta \;\text{for}\; i \neq p, \\
                &\alpha_i \beta_{n, i} > 0 \;\text{for}\; i=p,\quad \alpha_i \beta_{n, i} \geq 0 \;\text{for}\; p < i \leq q. \label{eq:complexity-weak-1}
            \end{align}
    \end{enumerate}
\end{asm}

The most distinctive condition in \cref{asm:complexity-weak} is $w^0_n \sim \mathrm{Unif}(S^{d-1} \cap V_\kappa)$, which reflects the knowledge acquired through pre-training.
The remaining conditions are standard in the misspecified setting \citep{mousavi2023gradient, oko2024learningsumdiversefeatures, lee2024neural}.
Condition \eqref{eq:complexity-weak-1} is needed to ensure the expected correlation loss is monotone in the alignment $\theta^\top_\kappa w_n$ \citep[see][Appendix B.2]{oko2024learningsumdiversefeatures}, and is automatically satisfied in the well-specified case $\sigma^*_\kappa = \sigma^w_n$ \citep{benarous2021online}.
In \cref{weak-initialization}, following \citet{oko2024learningsumdiversefeatures, lee2024neural}, we show how to construct an initialization satisfying these conditions.

\paragraph{Training Algorithm}
\cref{alg:weak} describes the weak model training procedure.
It follows the layer-wise training paradigm commonly used in feature learning theory \citep{damian2022neural, ba2022high, bietti2022learningsingleindexmodelsshallow, abbe2023sgd, mousavi2023gradient}.

In Phase~I, the first-layer weights $w_n$ are trained by online SGD to minimize the correlation loss $\mathcal{L} = -y r^w_\Theta(x)$ \citep{oko2024learningsumdiversefeatures, lee2024neural}, which matches the behavior of squared loss when the learning rate is sufficiently small \citep{lee2024neural}.
The update uses the spherical gradient $\tilde{\nabla}_w r(x) = (I - ww^\top)\nabla_w r(x)$, where $\nabla_w$ denotes the Euclidean gradient.

After Phase~I, sufficiently many neurons satisfy $|\theta^\top_\kappa w^{T_1}_n| \geq 1 - \epsWeakAlign$, but some do not.
Phase~II filters these out by computing the empirical covariance matrix of $(\hat{w}_n)_n$, extracting its leading eigenvector $\hat{\theta}$, and retaining only neurons with $|\hat{\theta}^\top \hat{w}_n| \geq 1 - 2\epsWeakAlign$, which guarantees all remaining neurons satisfy \cref{asm:weak}~(i).
The biases $b_n$ are then re-randomized before Phase~III, in which the second-layer coefficients $a_n$ are learned by $\ell^2$-regularized least squares regression.

\begin{algorithm}[t]
    \caption{Weak model training by online SGD} \label{alg:weak}
    \DontPrintSemicolon

    \KwIn{Initialized weak model $r^w_\Theta$ with $\Theta = (a_n, b_n, w_n^0)_{n=1}^{\tilde{N}^w}$; learning rate schedule $\eta^t$; regularization parameter $\lambda$; sample sizes $T_1, T_2$; bias scale $C_b$; filtering threshold $\epsWeakAlign$.}

    \Indm \textbf{Phase I: first-layer training} \Indp

    \For{$t=0, 1, \dots, T_1-1$}{
        Draw $x^t \sim \mathcal{N}(0, I_d)$, $y^t = r^*_\kappa(\theta^\top_\kappa x^t) + \zeta$.\;
        $w_n^{t+1} \gets w_n^t + \eta^t y^t \tilde{\nabla}_w r^w_{(a_n, b_n, w_n^t)_{n=1}^{\tilde{N}^w}}(x^t)$ \;
        $w_n^{t+1} \gets w_n^{t+1} / \|w_n^{t+1}\|, \quad (n = 1, 2, \dots, \tilde{N}^w)$\;
    }
    $\hat{w}_n \gets w_n^{T_1}$ \;

    \Indm \textbf{Phase II: neuron filtering} \Indp

    \Indp $W_{\tilde{N}^w} \gets \frac{1}{\tilde{N}^w} \sum_{n=1}^{\tilde{N}^w} \hat{w}_n \hat{w}_n^\top$ \;
    $\hat{\theta} \gets$ leading unit eigenvector of $W_{\tilde{N}^w}$.\;
    $\mathcal{N}^w \gets \{n \in [\tilde{N}^w] \mid |\hat{\theta}^\top \hat{w}_n| \geq 1 - 2\epsWeakAlign\}$ \;

    \Indm \textbf{Re-initialize} Drop neurons not in $\mathcal{N}^w$; set $\Theta = (a_n, b_n, \hat{w}_n)_{n=1}^{N^w}$ with $N^w = |\mathcal{N}^w|$.
    Resample $b_n \sim \mathrm{Unif}([-C_b, C_b])$. \;

    \Indm \textbf{Phase III: second-layer training} \Indp

    \Indp Draw $x^t \sim \mathcal{N}(0, I_d)$, $y^t = r^*_\kappa(\theta^\top_\kappa x^t)$, for $t = T_1, \dots, T_1+T_2-1$.\;
    $\hat{a} \gets \mathop{\mathrm{argmin}}_{a \in \mathbb{R}^{N^w}} \frac{1}{T_2} \sum_{t=T_1}^{T_1+T_2-1} \bigl(r_{(a_n, b_n, \hat{w}_n)_{n=1}^{N^w}}(x^t) - y^t\bigr)^2 + \lambda \|a\|^2$ \; \Indm

    \KwOut{$r^w_{\hat{\Theta}}$ with $\hat{\Theta} = (\hat{a}_n, b_n, \hat{w}_n)_{n=1}^{N^w}$.}
\end{algorithm}

\setcounter{prop}{\numexpr\value{num@prop@weak}-1\relax}
\begin{prop}[Weak model construction; formal]\label{prop:weak}
  Under \cref{asm:additive,asm:complexity-weak}, fix $\epsWeakAlign \in (0, 3c_1)$.
  \cref{alg:weak} with the two-phase learning rate schedule
  \[
    \eta^t =
    \begin{cases}
      \eta_1 \leq c_\eta\, d^{-p/2} & (0 \leq t \leq T_{1} - 1), \\
      \eta_2 \leq c_\eta\, \epsWeakAlign\, d^{-1} \wedge c_\eta\, \epsWeakAlign^2 & (T_{1} \leq t \leq T_{1} + T_{2} - 1),
    \end{cases}
  \]
  produces a weak model satisfying
  \cref{asm:weak} with $\epsWeakLOne = \tilde{\Theta}(\epsWeakAlign)$,
  using a total of
  \begin{align}
    T = \tilde{O}\!\left(d^{p/2}\, s^{(p-2)/2} \vee d\,\epsWeakAlign^{-2} \vee \epsWeakAlign^{-3}\right)
  \end{align}
  samples.
\end{prop}

\begin{proof}
  Setting $\tilde{\varepsilon} = \epsWeakAlign/6$ in \cref{lem:weak-strongrecovery-2}, Phase~I produces neurons satisfying $\theta_\kappa^\top w^{T_1} > 1 - \epsWeakAlign/2$.
  Setting $\bar{\varepsilon} = \epsWeakAlign/2$ in \cref{lem:weak-filtering-2}, the neuron filtering step (\cref{weak-filtering}) retains neurons satisfying $\theta_\kappa^\top \hat{w}_n \geq 1 - \epsWeakAlign$, establishing \cref{asm:weak}(i).
  By \cref{lem:weak-second-layer-1,lem:weak-second-layer-2,lem:weak-second-layer-3}, second-layer ridge regression in Phase~III produces $\hat{a}$ satisfying $\E_x[|r_{\hat{a}}(x) - r^*_\kappa(\theta_\kappa^\top x)|] \leq \pi_\kappa \epsWeakLOne$, establishing \cref{asm:weak}(ii).
  Phase~I uses $T_1 = \tilde{O}(d^{p/2} s^{(p-2)/2} \vee d \epsStrongAlign^{-2} \vee \epsWeakAlign^{-3})$ samples (\cref{lem:weak-weakrecovery-2}, \cref{weak-amplification}, and \cref{lem:weak-strongrecovery-2}) and Phase~III uses $T_2 = \tilde{O}(\epsWeakAlign^{-2})$ samples (\cref{lem:weak-second-layer-3}).
  The total is $T = \tilde{O}(d^{p/2} s^{(p-2)/2} \vee d \epsStrongAlign^{-2} \vee \epsWeakAlign^{-3})$.
\end{proof}

\paragraph{Notations.}
    Throughout this section, the task $\kappa$ is fixed.
    Accordingly, we write $\theta = \theta_\kappa$ and $\alpha_i = \alpha_{\kappa, i}$ for all $i$.
    Moreover, during the training of the first layer, the gradient updating each $w^t_n$ does not depend on the other neurons.
    Hence, in what follows we focus on a single neuron $n$, suppress the index $n$, and write $\sigma = \sigma^w_n$, $a = a^w_n$, $b = b^w_n$, $w^t = w_n^t$, and $\beta_i = \beta_{n, i}$ for all $i$.

    The subsequent arguments are carried out after rescaling the learning rate in the stochastic gradient descent update rule appropriately.
    Specifically, by setting $\eta^t \gets N^w \pi_\kappa^{-1} \eta^t$, the gradient $\eta^t y^t \tilde{\nabla}_w r^w_\Theta (x^t)$ is expressed as $\pi^{-1}_\kappa \eta^t \sum_{n=1}^{N^w} (\sigma_\kappa^*(\theta^\top x^t) + \zeta)  a_n \sigma^w_n (w^{t\top}_n x^t + b_n)$ with $\zeta \sim \mathcal{N}(0, 1)$.

    We take constants of order $\mathop{\mathrm{polylog}} d$ satisfying the following ordering:
    \begin{align}
        C_1 \lesssim c_1^{-1} \lesssim C_2 \lesssim c_\eta^{-1} = \tilde{O}(1).
    \end{align}
    The precise ordering of these constants will be determined so as to keep the subsequent proofs consistent.
    For example, an estimate such as $C_1^2 \leq C_2$ is allowed.
    These constants are only valid within the present section, and denote constants different from those defined analogously in \cref{app:strong,app:sft}.
    Furthermore, we assume $d$ is sufficiently large, that $c_1, c_\eta$ are smaller than any of the finitely many $\Theta(1)$ constants appearing in the subsequent proofs, while $C_1$ and $C_2$ are larger than any such constant.

\subsection{Initialization} \label{weak-initialization}
We verify that random initialization satisfies \cref{asm:complexity-weak} with probability $\Omega(1)$.
Following \citet{oko2024learningsumdiversefeatures, lee2024neural}, we show that the initial alignment of $w^0_n$ is $s^{-1/2}$ with probability $\Theta(1)$, and that the sign condition \eqref{eq:complexity-weak-1} holds with probability $\Omega(1)$.

\begin{lem}\label{lem:weak-init}
    Let $w^0\sim \mathrm{Unif}(S^{s-1})$. Then,
    \begin{align}
        \mathbb{P}\left[e_1^\top w^0 \geq s^{-\frac{1}{2}}\right] = \mathbb{P}\left[e_1^\top w^0 \leq -s^{-\frac{1}{2}}\right] = \Theta(1).
    \end{align}
\end{lem}
\begin{lem}[\cite{chang2011chernoff}, Theorem 2]
    Let $\beta > 1$ and $a\in \mathbb{R}$ be arbitrary. Then,
    \begin{align}
    \frac{\sqrt{2e(\beta-1)}}{2\beta\sqrt{\pi}} e^{-\frac{\beta a^2}{2}} \leq \int_a^\infty \frac{1}{\sqrt{2\pi}} e^{-\frac{t^2}{2}}dt.
    \end{align}
\end{lem}
\begin{proof}[Proof of \cref{lem:weak-init}]
    Since $e_1^\top w^0 \overset{d}{=} e_1^\top\frac{g}{\|g\|}$ with $g\sim \mathcal N(0, I_s)$, we have
    \begin{align}
        \mathbb{P} \left[e_1^\top w^0 \geq s^{-\frac{1}{2}}\right] &\geq  \mathbb{P}_g\left[e_1^\top g \geq 1  \wedge \|g\|\leq s^\frac{1}{2}\right]\\
        &\geq \mathbb{P}_g\left[e_1^\top g \geq 1\right] - \mathbb{P}_g\left[\|g\| \geq s^\frac{1}{2}\right]\\
        &\gtrsim \frac{\sqrt{2e(\beta-1)}}{2\beta\sqrt{\pi}} e^{-\frac{\beta}{2}} - e^{-\Omega(s)} = \Theta(1).
    \end{align}
\end{proof}

\begin{lem}
    For $n = 1, 2, \dots N^w$, let $a_n\sim \mathrm{Unif}\{\pm 1\}$ and $\xi_{n, i} \sim \mathrm{Unif}\{\pm 1\}$, and let $1 \leq p\leq q$. Then, with probability $\Omega(1)$, $a_n \sigma_n = \sum_{i = 1}^q a_n \xi_{n, i}\mathrm{He}_i$ has its $i$-th Hermite coefficient ($p\leq i\leq q$) taking the desired sign for every $i$.
    Therefore, by taking these $a_n$, $\sigma_n$, and $b_n = 0$ as the initial values of $a^w_n$, $\sigma^w_n$, and $b^w_n$ respectively, condition \eqref{eq:complexity-weak-1} of \cref{asm:complexity-weak} is satisfied.
\end{lem}
\begin{proof}
    By construction, the probability that $a_n \xi_{n, i}$ has the desired sign is $2^{-q}$.
\end{proof}

\subsection{Gradient decomposition}\label{weak-decomposition}

\begin{lem}\label{lem:weak-onestep-1}
    The gradient term decomposes as follows:
    \begin{align}
        \nabla_w y^t a \sigma(w^{t\top} x^t + b) = \sum_{i=p}^q \left[ i \alpha_i \beta_i (\theta^\top w^t)^{i-1} \theta + \sqrt{(i+2)(i+1)}  \alpha_i \beta_{i+2} (\theta^\top w^t)^i w^t \right] + Z^t.
    \end{align}
    Here, the random variable $Z^t$ has mean $0$ and satisfies $\|Z^t\| = \tilde{O}(d^\frac{1}{2})$ with high probability, and for every $v \in \mathbb{R}^{d}$, $|v^\top Z^t| = \tilde{O}(1)$ with high probability.
    Furthermore, $\|\nabla_w y^t a \sigma(w^{t\top} x^t + b)\| = \tilde{O}(d^\frac{1}{2})$ with high probability, and for every $v \in \mathbb{R}^{d}$ with $\|v\| = O(1)$, $(\nabla_w y^t a \sigma(w^{t\top} x^t + b))^\top v = \tilde{O}(1)$ holds with high probability.
\end{lem}

\begin{proof}
    \begin{align}
        &\nabla_w\mathbb E_{x\sim \mathcal{N}(0, I_d)}\left[ y^t a \sigma(w^{t\top} x + b)\right] \\ &
        = \mathbb{E}_{x \sim \mathcal{N}(0, I_d)}\left[\sigma^*_\kappa (\theta^\top x) a \sigma'(w^{t\top} x + b) x \right] \\ &
        = \mathbb{E}_{x \sim \mathcal{N}(0, I_d)} \left[ \left( \sum_{i=p}^q \frac{\alpha_i}{\sqrt{i!}} \mathrm{He}_i (\theta^\top x) \right) \left( \sum_{i=1}^q i\frac{\beta_i}{\sqrt{i!}} \mathrm{He}_{i-1}(w^{t\top} x) \right) x \right] \\ &
        = \sum_{i=p}^q \sum_{j=1}^q \frac{j}{\sqrt{i! j!}} \alpha_i \beta_j \mathbb{E}_{x \sim \mathcal{N}(0, I_d)} \left[ i \theta \mathrm{He}_{i-1}(\theta^\top x) \mathrm{He}_{j-1} (w^{t\top} x) + (j-1) w^t \mathrm{He}_i( \theta^\top x) \mathrm{He}_{j-2}(w^{t\top} x) \right] \\ &
        = \sum_{i=p}^q \left[ i \alpha_i \beta_i (\theta^\top w^t)^{i-1} \theta + \sqrt{(i+2) (i+1)} \alpha_i \beta_{i+2}(\theta^\top w^t)^i w^t \right].
    \end{align}
    Define $Z^t = \nabla_w y^t a \sigma(w^{t\top} x^t + b) - \nabla_w \mathbb{E}[y^t a \sigma(w^{t\top}x + b)]$.
    Then $\mathbb{E}[Z^t] = 0$, and
    \begin{align}
        \|Z^t\|\leq |y^t||a||\sigma'(w^{t\top}x^t+b)|\|x^t\| + \mathbb{E}_{x\sim\mathcal{N}(0, I_d)}\left[|y(x)||a||\sigma'(w^{t\top}x+b)|\|x\|\right] = \tilde{O}(d^\frac{1}{2})
    \end{align}
    with high probability. Moreover, for every $v\in \mathbb{R}^d$, $|v^\top Z^t| = \tilde{O}(1)$ with high probability.
    Similarly, $\|\nabla_w y^t a \sigma(w^{t\top} x^t + b)\| = \tilde{O}(d^\frac{1}{2})$ with high probability, and for every $v \in \mathbb{R}^{d}$ with $\|v\| = O(1)$, $(\nabla_w y^t a \sigma(w^{t\top} x^t + b))^\top v = \tilde{O}(1)$ with high probability.
\end{proof}

\begin{lem}\label{lem:weak-onestep-2}
    Let $\eta = \eta^t \leq c_\eta d^{-1}$ and suppose $\theta^\top w^t \geq \frac{1}{2} s^{-\frac{1}{2}}$.
    Then,
    \begin{align}
        & \theta^\top w^t + \eta \sum_{i=p}^q \left[i \alpha_i \beta_i (\theta^\top w^t)^{i-1} (1 - (\theta^\top w^t)^2) \right] - \eta^2 C_1^2 (\theta^\top w^t) d + \eta \theta^\top P_{w^t}^{\perp} Z^t \\ &
        \leq \theta^\top w^{t+1} \\ &
        \leq \theta^\top w^t + \eta \sum_{i=p}^q \left[i \alpha_i \beta_i (\theta^\top w^t)^{i-1} (1 - (\theta^\top w^t)^2) \right] + \eta \theta^\top P_{w^t}^{\perp}Z^t.
    \end{align}
    Here, $Z^t$ is a mean-zero random variable satisfying $\|Z^t\| = \tilde{O}(d^\frac{1}{2})$ with high probability, and for every $v\in \mathbb{R}^d$ with $\|v\| = O(1)$, $|v^\top Z^t| = \tilde{O}(1)$ with high probability. Furthermore, $|\theta^\top w^t - \theta^\top w^{t+1}| = \tilde{O}(\eta)$ with high probability.
\end{lem}

\begin{proof}
    Let $g= \nabla_wy^ta^t\sigma(w^{t\top}x^t+b)$.
    Then,
    \begin{align}
        \|w^t -\eta P_{w^t}^{\perp}(-g)\|^{-1} &= (\|w^t\|^2+\eta^2\|P_{w^t}^{\perp}g\|^2)^{-\frac{1}{2}}\\
        &\geq 1-\frac{1}{2} \eta^2 \|P_{w^t}^{\perp}g\|^2 \geq 1 - \frac{1}{2} \eta^2 \|g\|^2.
    \end{align}
    Therefore, with high probability,
    \begin{align}
        \theta^\top w^{t+1} &= \theta^\top\frac{w^t-\eta P_{w^t}^{\perp}(-g)}{\|w^t-\eta P_{w^t}^{\perp}(-g)\|}\\
        &\geq (\theta^\top w^t + \eta\theta^\top P_{w^t}^{\perp}g)(1-\frac{1}{2}\eta^2\|g\|^2)\\
        &= \theta^\top w^t + \eta\theta^\top P_{w^t}^{\perp}g - \frac{1}{2}(\theta^\top w^t)\eta^2\|g\|^2 - \frac{1}{2}\eta^3|\theta P_{w^t}^{\perp}g|\|g\|^2\\
        &\geq \theta^\top w^t + \eta\theta^\top P_{w^t}^{\perp}\left(\sum_{i=p}^q\left[i \alpha_i \beta_i (\theta^\top w^t)^{i-1} \theta + \sqrt{(i+2)(i+1)} \alpha_i \beta_{i+2}(\theta^\top w^t)^i w^t \right] + Z^t \right) \\ &
        \quad - \frac{1}{2}(\theta^\top w^t)\eta^2\|g\|^2 - \frac{1}{2}\eta^3\|g\|^3\\
        & \geq \theta^\top w^t + \eta\sum_{i=p}^q\left[i \alpha_i \beta_i (\theta^\top w^t)^{i-1} \right]\theta^\top P_{w^t}^{\perp}\theta + \eta\theta^\top P_{w^t}^{\perp}Z^t - \frac{1}{2}(\theta^\top w^t)\eta^2C_1^2d - \frac{1}{2}\eta^3C_1^3d^\frac{3}{2}\\
        & \geq \theta^\top w^t + \eta\sum_{i=p}^q\left[i \alpha_i \beta_i (\theta^\top w^t)^{i-1} \right] (1-(\theta^\top w^t)^2)  - (\theta^\top w^t)\eta^2C_1^2d + \eta\theta^\top P_{w^t}^{\perp}Z^t
    \end{align}
    Here, the second and third inequalities use \cref{lem:weak-onestep-1}, and the fourth uses $\eta \leq c_\eta s^{-1}$ and $\theta^\top w^t \geq \frac{1}{2} s^{-\frac{1}{2}}$.
    On the other hand,
    \begin{align}
        \theta^\top w^{t+1} &\leq \theta^\top(w^t-\eta P_{w^t}^{\perp}(-g))\\
        &= \theta^\top w^t + \eta\theta^\top P_{w^t}^{\perp}g = \theta^\top w^t + \eta\sum_{i=p}^q\left[i \alpha_i\beta_i(\theta^\top w^t)^{i-1}\right](1-(\theta^\top w^t)^2) + \eta\theta^\top P_{w^t}^{\perp}Z^t.
    \end{align}
    Combining these gives the desired inequality.
    The properties of $Z^t$ follow from \cref{lem:weak-onestep-1}.
    We now show $|\theta^\top w^{t+1}-\theta^\top w^t| = \tilde{O}(\eta)$.
    From the discussion above,
    \begin{align}
        |\theta^\top w^{t+1}-\theta^\top w^t|\leq \eta|\theta^\top P_{w^t}^{\perp}g| + \frac{1}{2}|\theta^\top w^t|\eta^2\|g\|^2 + \frac{1}{2}\eta^3|\theta P_{w^t}^{\perp}g|\|g\|^2
    \end{align}
    and since $\eta \leq c_\eta d^{-1}$, the right-hand side is $\tilde{O}(\eta)$.
\end{proof}

\subsection{Weak alignment}\label{weak-weakrecovery}
\begin{lem}\label{lem:weak-weakrecovery-1}
    Suppose $\theta^\top w^0 \geq s^{-\frac{1}{2}}$ and $\theta^\top w^t \leq c_1$ for all $t\leq \tau$. Then, setting $\eta^t = \eta \leq c_\eta d^{-\frac{p}{2}}$, for every $t \leq \tau$ we have
    \begin{equation}\label{eq:weak-weakrecovery-1}
        \theta^\top w^{t+1} \geq (1 - c_1) (\theta^\top w^0) + \eta (1 - c_1) \sum_{t'=0}^t p \alpha_p \beta_p (\theta^\top w^{t'})^{p-1}
    \end{equation}
    with high probability. Furthermore, defining $(P^t)_{t=0}^{\tau+1}$ by
    \begin{align}
        P^0=(1-c_1)(\theta^\top w^0),
    \end{align}
    \begin{align}
        P^{t+1} = P^t + \eta(1-c_1) p \alpha_p \beta_p (P^t)^{p-1} \quad (0 < t \leq \tau),
    \end{align}
    we have $\theta^\top w^{t+1} \geq P^{t+1}$ with high probability for all $t \leq \tau$.
\end{lem}

\begin{proof}
    If $\theta^\top w^t \geq \frac{1}{2}s^{-\frac{1}{2}}$, then by \cref{lem:weak-onestep-2}
    \begin{align}
        \theta^\top w^{t+1} &\geq \theta^\top w^t + \eta \sum_{i=p}^q\left[i\alpha_i\beta_i(\theta^\top w^t)^{i-1}(1-(\theta^\top w^t)^2) \right] - (\theta^\top w^t)\eta^2C_1^2 d + \eta\theta^\top P_{w^t}^{\perp}Z^t\\
        &\geq \theta^\top w^t + \eta p\alpha_p\beta_p(\theta^\top w^t)^{p-1}(1-(\theta^\top w^t)^2)-(\theta^\top w^t)\eta^2C_1^2 d + \eta\theta^\top P_{w^t}^{\perp}Z^t\\
        &\geq \theta^\top w^t + \eta p\alpha_p\beta_p(\theta^\top w^t)^{p-1}\left(1-\frac{1}{3}c_1\right)-\eta(\theta^\top w^t)c_\eta d^{-\frac{p-2}{2}}C_1^2 + \eta\theta^\top P_{w^t}^{\perp}Z^t.
    \end{align}
    In the third inequality we used $(\theta^\top w^t)^2\leq c_1^2\leq \frac{1}{3}c_1$. Also, since $\theta^\top w^t \geq \frac{1}{2}s^{-\frac{1}{2}} \geq \frac{1}{2} d^{-\frac{1}{2}}$,
    \begin{align}
        \eta(\theta^\top w^t) c_\eta d^{-\frac{p-2}{2}}C_1^2 \leq \frac{1}{3} c_1\eta p\alpha_p\beta_p(\theta^\top w^t)^{p-1}
    \end{align}
    and hence
    \begin{equation}\label{eq:weak-weakrecovery-2}
        \theta^\top w^{t+1} \geq \theta^\top w^t + \eta p\alpha_p\beta_p(\theta^\top w^t)^{p-1}\left(1-\frac{2}{3}c_1\right)+ \eta\theta^\top P_{w^t}^{\perp}Z^t.
    \end{equation}
    We prove the claim by induction.
    Take $t'\leq\tau$.
    Assume that $\theta^\top w^t \geq (1-c_1)(\theta^\top w^0)$ and $\theta^\top w^t \geq\frac{1}{2}s^{-\frac{1}{2}}$ hold for $t=0, 1,\dots,t'$. Then, by \eqref{eq:weak-weakrecovery-2},
    \begin{align}
        \theta^\top w^{t'+1} &\geq \theta^\top w^{t'} + \left(1-\frac{2}{3}c_1\right)\eta p\alpha_p\beta_p(\theta^\top w^{t'})^{p-1} + \eta\theta^\top P_{w^{t'}}^{\perp}Z^{t'}\\ &
        \geq \theta^\top w^0 + \sum_{t=0}^{t'}\left(1-\frac{2}{3}c_1\right)\eta p\alpha_p\beta_p(\theta^\top w^t)^{p-1} + \sum_{t=0}^{t'}\eta\theta^\top P_{w^t}^{\perp}Z^t.
        \label{eq:weak-weakrecovery-3}
    \end{align}
    If $t'\leq C_2(\theta^\top w^0)^{2-2p}$, then
    \begin{align}
        -\sum_{t=0}^{t'} \eta\theta^\top P_{w^t}^\perp Z^t &\leq \eta C_1\sqrt{t'} \leq c_\eta s^{-\frac{p}{2}}C_1\sqrt{C_2}(\theta^\top w^0)^{1-p}\\ &
        \leq 2^pc_\eta C_1\sqrt{C_2}(\theta^\top w^0)\leq c_1(\theta^\top w^0)
    \end{align}
    with high probability. On the other hand, when $t'>C_2(\theta^\top w^0)^{2-2p}$,
    \begin{align}
        -\sum_{t=0}^{t'}\eta\theta^\top P_{w^t}^{\perp}Z^t &\leq \eta C_1\sqrt{t'} < \eta C_1 t'C_2^{-\frac{1}{2}}(\theta^\top w^0)^{p-1}\\ &
        \leq \frac{1}{3}c_1\eta t'p\alpha_p\beta_p((1-c_1)(\theta^\top w^0))^{p-1} \leq \frac{1}{3}\sum_{t=0}^{t'} c_1\eta p\alpha_p\beta_p(\theta^\top w^t)^{p-1}
    \end{align}
    with high probability. Substituting these into \eqref{eq:weak-weakrecovery-3},
    \begin{align}
        \theta^\top w^{t'+1} \geq (1-c_1)(\theta^\top w^0) + \sum_{t=0}^{t'}(1-c_1)\eta p\alpha_p\beta_p(\theta^\top w^t)^{p-1}
    \end{align}
    with high probability.
    Hence \eqref{eq:weak-weakrecovery-1} follows for $t=t'+1$, and $\theta^\top w^{t'+1} \geq (1-c_1)(\theta^\top w^0)$ and $\theta^\top w^{t'+1} \geq \frac{1}{2}s^{-\frac{1}{2}}$ are satisfied.
    Moreover, for $(P^{t'})_{t'=0}^{\tau}$ we have
    \begin{align}
        P^{t'+1} = P^0 + \sum_{t=0}^{t'}\eta(1-c_1)p\alpha_p\beta_p(P^t)^{p-1},
    \end{align}
    so comparing with \eqref{eq:weak-weakrecovery-1} yields $\theta^\top w^t \geq P^t$ for all $t\leq \tau +1$.
\end{proof}

\begin{lem}\label{lem:weak-weakrecovery-2}
    Let $\eta^t = \eta \leq c_\eta d^{-\frac{p}{2}}$.
    Then there exists $t_1 \leq T_{1, 1} = \Theta(\eta^{-1} s^\frac{p-2}{2})$ such that $\theta^\top w^{t_1} > c_1$ with high probability.
\end{lem}

\begin{proof}
    Suppose $\theta^\top w^t \leq c_1$ for all $t \leq T_{1, 1}$. Setting $c = \eta (1-c_1) p \alpha_p \beta_p$, define
    \begin{align}
        T_{1, 1} = \left\lfloor c^{-1}(1+c)^{p-1}(p-2)^{-1}(P^0)^{-(p-2)}\right\rfloor.
    \end{align}
    Then, by \cref{lem:weak-weakrecovery-1} and the Bihari--LaSalle inequality (\cref{lem:technical-bihari}),
    \begin{align}
         \theta^\top w^t \geq P^t & \geq \frac{P^0}{\left(1 - c (1+c)^{-(p-1)} (p-2)(P^0)^{p-2} t \right)^{\frac{1}{p-2}}} \label{eq:weak-weakrecovery-4}
    \end{align}
    and at $t = T_{1, 1}$, using the fact that $1 \leq a(\lfloor (a)^{-1} \rfloor + 1)$ for every $a > 0$,
    \begin{align}
        \theta^\top w^{T_{1,1}} \geq \frac{P^0}{\left(c (1+c)^{-(p-1)} (p-2) (P^0)^{p-2} \right)^\frac{1}{p-2}} = \frac{1}{\left(\eta (1-c_1) (1+c)^{-(p-1)} p (p-2) \alpha_p \beta_p \right)^\frac{1}{p-2}} > 1.
    \end{align}
    This contradicts $\theta^\top w^{T_{1, 1}} \leq 1$.
    Therefore, there exists $t_1 \leq T_{1, 1}$ such that $\theta^\top w^t \leq c_1$ for all $t \leq t_1 - 1$ and $\theta^\top w^{t_1} > c_1$.
\end{proof}

\subsection{Amplification of alignment}\label{weak-amplification}
For the remainder of this section, we shift the time index so that $t_1 = 0$, i.e., weak alignment is achieved at $t = 0$.
\begin{lem}\label{lem:weak-amplification-1}
    Suppose $\theta^\top w^0 \geq c_1$, $\eta = \eta_t \leq c_\eta d^{-\frac{p}{2}}$, and $\theta^\top w^t \leq 1-c_1$ for all $t\leq \tau$.
    Moreover, define $(P^t)_{t=0}^{\tau+1}$ by $P^0 = (1-c_1) (\theta^\top w^0)$ and $P^{t+1} = P^t + c_1 \eta p \alpha_p \beta_p (P^t)^{p-1}$.
    Then, for every $t \leq \tau+1$, $\theta^\top w^t \geq P^t$ with high probability.
\end{lem}

\begin{proof}
    By an argument analogous to the proof of \cref{lem:weak-weakrecovery-1}, if $\theta^\top w^t \geq \frac{1}{2}s^{-\frac{1}{2}}$ then
    \begin{align}
        \theta^\top w^{t+1} & \geq \theta^\top w^t + \eta p \alpha_p \beta_p (\theta^\top w^t)^{p-1} (1 - (\theta^\top w^t)^2) - \eta^2 C_1^2 (\theta^\top w^t) d + \eta \theta^\top P_{w^t}^{\perp} Z^t \\ &
        \geq \theta^\top w^t + \eta p \alpha_p \beta_p (\theta^\top w^t)^{p-1} (1 - (1-c_1)^2) - \eta c_\eta C_1^2 (\theta^\top w^t) d^{-\frac{p-2}{2}} + \eta \theta^\top P_{w^t}^{\perp}Z^t.
    \end{align}
    As in \cref{lem:weak-weakrecovery-1}, take $t'\leq \tau$ and assume that $\theta^\top w^t \geq (1 - c_1)(\theta^\top w^0)$ and $\theta^\top w^t \geq \frac{1}{2} s^{-\frac{1}{2}}$ for $t = 0,1,\dots,t'$.
    Since $\theta^\top w^t \geq \frac{1}{2}s^{-\frac{1}{2}} \geq d^{-\frac{1}{2}}$, we have $\eta c_\eta C_1^2 (\theta^\top w^t) d^{-\frac{p-2}{2}} \leq \frac{1}{3} c_1 \eta p \alpha_p \beta_p (\theta^\top w^t)^{p-1}$.
    Furthermore, since $c_1^2 \leq \frac{1}{3} c_1$, $(1 - (1-c_1)^2) = 2c_1 - c_1^2 \geq \frac{5}{3} c_1$.
    Thus,
    \begin{align}
        \theta^\top w^{t'+1} & \geq \theta^\top w^{t'} + \frac{4}{3} c_1 \eta p \alpha_p \beta_p (\theta^\top w^{t'})^{p-1} + \eta \theta^\top  P_{w^{t'}}^{\perp} Z^{t'}\\ &
        \geq \theta^\top w^0 + \frac{4}{3} \sum_{t=0}^{t'} c_1 \eta p \alpha_p \beta_p (\theta^\top w^t)^{p-1} + \sum_{t=0}^{t'} \eta \theta^\top  P_{w^t}^{\perp} Z^t.
    \end{align}
    Furthermore,
    \begin{align}
        -\sum_{t=0}^{t'} \eta \theta^\top P_{w^t}^{\perp} Z^t \leq c_1 (\theta^\top w^0) + \frac{1}{3} \sum_{t=0}^{t'} c_1 \eta p \alpha_p \beta_p (\theta^\top w^t)^{p-1}
    \end{align}
    with high probability. Therefore,
    \begin{align}\label{eq:weak-amplification-1}
        \theta^\top w^{t'+1} \geq (1 - c_1) (\theta^\top w^0) + \sum_{t=0}^{t'} c_1 \eta p \alpha_p \beta_p (\theta^\top w^t)^{p-1}
    \end{align}
    with high probability.
    Hence $\theta^\top w^{t'+1} \geq (1-c_1)(\theta^\top w^0)$ and $\theta^\top w^{t'+1} \geq \frac{1}{2} s^{-\frac{1}{2}}$.
    Therefore, for every $t'\leq \tau$, $\theta^\top w^{t'} \geq (1 - c_1)(\theta^\top w^0)$ and $\theta^\top w^{t'} \geq \frac{1}{2}s^{-\frac{1}{2}}$ hold, and \eqref{eq:weak-amplification-1} is satisfied by induction.
    Comparing with the update equation for $P^{t'}$, we obtain $\theta^\top w^{t'} \geq P^{t'}$ with high probability for every $t' \leq \tau+1$.
\end{proof}

\begin{lem}
    Let $\eta = \eta_t \leq c_\eta d^{-\frac{p}{2}}$.
    For neurons satisfying \cref{lem:weak-onestep-2}, there exists $t_2 \leq T_{1, 2} = \tilde{\Theta}(\eta^{-1})$ such that $\theta^\top w^{t_2} > 1-c_1$ with high probability.
\end{lem}

\begin{proof}
    Suppose $\theta^\top w^t \leq 1-c_1$ for all $t \leq T_{1, 2}$.
    Setting $c = \eta c_1 p \alpha_p \beta_p$, define
    \begin{align}
        T_{1, 2}= \lfloor c^{-1}(1+c)^{p-1}(p-2)^{-1}(P^0)^{-(p-2)}\rfloor.
    \end{align}
    Then, by \cref{lem:weak-amplification-1} and the Bihari--LaSalle inequality \cref{lem:technical-bihari}, at $t=T_{1, 2}$
    \begin{align}
        \theta^\top w^{T_{1,2}} \geq P^{T_{1,2}} & \geq \frac{P^0}{\left(1 - c (1+c)^{-(p-1)} (p-2) T_{1,2} \right)^{\frac{1}{p-2}}} \\ &
        \geq \frac{P^0}{\left(c (1+c)^{-(p-1)} (p-2) (P^0)^{p-2} \right)^\frac{1}{p-2}} = \frac{1}{\left(\eta c_1 (1+c)^{-(p-1)} p (p-2) \alpha_p \beta_p \right)^\frac{1}{p-2}} > 1.
    \end{align}
    This contradicts $\theta^\top w^{T_{1,2}} \leq 1$.
    Therefore, there exists $t_2 \leq T_{1,2}$ such that $\theta^\top w^t \leq 1 - c_1$ for all $t \leq t_2 - 1$ and $\theta^\top w^{t_2} > 1 - c_1$.
\end{proof}

\subsection{Strong alignment and localization}\label{weak-strongrecovery}

For the remainder of this section, we shift the time index so that $t_1 + t_2 = 0$, i.e., amplification of alignment is achieved at $t = 0$.
\begin{lem}\label{lem:weak-strongrecovery-1}
    Let $0 < \bar{\varepsilon} < c_1$. If $\eta^t = \eta \leq c_\eta \bar{\varepsilon} d^{-1} \wedge c_\eta \bar{\varepsilon}^2$, and if $\theta^\top w^0 \geq 1 - 2c_1$ and $\theta^\top w^t \leq 1 - \bar{\varepsilon}$ for all $t\leq \tau$, then for every $t \leq \tau+1$,
    \begin{align}
        \theta^\top w^t \geq \theta^\top w^0 - c_1 \bar{\varepsilon} + t \bar{\varepsilon} \eta p \alpha_p \beta_p
    \end{align}
    with high probability.
\end{lem}

\begin{proof}
    If $\theta^\top w^t \geq \frac{1}{2}s^{-\frac{1}{2}}$, then by \cref{lem:weak-onestep-2},
    \begin{align}
        \theta^\top w^{t+1} &\geq \theta^\top w^t + \eta\sum_{i=p}^q\left[i\alpha_i\beta_i(\theta^\top w^t)^{i-1}(1-(\theta^\top w^t)^2)\right] - (\theta^\top w^t)\eta^2 C_1^2 d + \eta\theta^\top P_{w^t}^{\perp} Z^t \\ &
        \geq \theta^\top w^t + \eta p \alpha_p \beta_p(\theta^\top w^t)^{p-1}(1-(\theta^\top w^t)^2) - (\theta^\top w^t)\eta^2 C_1^2 d + \eta\theta^\top P_{w^t}^{\perp} Z^t.
    \end{align}
    If $\theta^\top w^t \geq 1-3c_1$, then
    \begin{align}\label{eq:weak-strongrecovery-1}
        (\theta^\top w^t)^{p-1}(1-(\theta^\top w^t)^2) \geq (1-3c_1)^{p-1}(1+(1-3c_1))(1-\theta^\top w^t) \geq \frac{5}{3}(1-\theta^\top w^t).
    \end{align}
    Moreover, if $\theta^\top w^t \leq 1-\bar{\varepsilon}$ and $\eta \leq c_\eta \bar{\varepsilon} d^{-1}$, then
    \begin{align} \label{eq:weak-strongrecovery-2}
        (\theta^\top w^t)\eta^2 C_1^2 d\leq \eta c_\eta C_1^2 (1-\theta^\top w^t)\leq \frac{1}{3}\eta p \alpha_p \beta_p(1-\theta^\top w^t).
    \end{align}
    Take an arbitrary $t'\leq \tau$.
    If $1-3c_1\leq (\theta^\top w^t)\leq 1-\bar{\varepsilon}$ for all $t\leq t'$, then by \eqref{eq:weak-strongrecovery-1} and \eqref{eq:weak-strongrecovery-2},
    \begin{align}
        \theta^\top w^{t'+1} &\geq \theta^\top w^{t'} + \frac{4}{3}\eta p \alpha_p \beta_p(1-\theta^\top w^{t'}) + \eta\theta^\top P_{w^{t'}}^{\perp}Z^{t'} \\ &
        \geq \theta^\top w^0 + \frac{4}{3}\sum_{t=0}^{t'} \eta p \alpha_p \beta_p(1-\theta^\top w^t) + \sum_{t=0}^{t'} \eta\theta^\top P_{w^t}^{\perp}Z^t.
        \label{eq:weak-strongrecovery-3}
    \end{align}
    Furthermore, since $\eta \leq c_\eta\bar{\varepsilon}^2$,
    \begin{align}
        -\sum_{t=0}^{t'} \eta\theta^\top P_{w^t}^{\perp}Z^t &\leq \eta C_1\sqrt{t'}\\ &
        \leq
        \begin{cases}
            \eta C_1\sqrt{C_2}\bar{\varepsilon}^{-1} \leq c_1\bar{\varepsilon} & (t'\leq C_2\bar{\varepsilon}^{-2})\\
            t' \eta C_1{t'}^{-\frac{1}{2}} < t'\eta C_2^{-\frac{1} {2}}\bar{\varepsilon} \leq \frac{1}{3}t'\eta p \alpha_p \beta_p(1-\theta^\top w^t) & (t' > C_2\bar{\varepsilon}^{-2})
        \end{cases}
    \end{align}
    with high probability.
    Substituting the above into \eqref{eq:weak-strongrecovery-3},
    \begin{align}
        \theta^\top w^{t'+1} &\geq \theta^\top w^0 - c_1\bar{\varepsilon} + \sum_{t=0}^{t'} \eta p \alpha_p \beta_p(1-\theta^\top w^t)\\ &
        \geq \theta^\top w^0 -c_1\bar{\varepsilon} + t'\eta p \alpha_p \beta_p\bar{\varepsilon}
        \label{eq:weak-strongrecovery-4}
    \end{align}
    with high probability.
    Hence $1-3c_1\leq \theta^\top w^t$ follows at $t=t'+1$. By induction, $1-3c_1\leq \theta^\top w^t$ holds for every $t\leq \tau + 1$, and \eqref{eq:weak-strongrecovery-4} follows for every $t'\leq \tau$.
\end{proof}

\begin{lem}\label{lem:weak-strongrecovery-2}
    Let $0 < \tilde{\varepsilon} < c_1$.
    Let
    \begin{align}
        \eta^t =
        \begin{cases}
            \eta_1 \leq c_\eta d^{-\frac{p}{2}} & (0 \leq t\leq \Delta T_1 + \Delta T_2 - 1)\\
            \eta_2 \leq c_\eta \tilde{\varepsilon} d^{-1} \wedge c_\eta \tilde{\varepsilon}^2. & (\Delta T_1+\Delta T_2 \leq t\leq \Delta T_1 + \Delta T_2 + T_{1, 3} - 1)
        \end{cases}
    \end{align}
    Here, $\Delta T_1 = T_{1, 1}-t_1$, $\Delta T_2 = T_{1, 2}-t_2$, and $T_{1, 3} = \tilde{\Theta}(\tilde{\varepsilon}^{-1} \eta_2^{-1})$.
    Then, for every neuron satisfying $\theta^\top w^0 \geq 1 - c_1$, $\theta^\top w^{\Delta T_1 + \Delta T_2 + T_{1,3}} > 1 - 3\tilde{\varepsilon}$ with high probability.
\end{lem}
\begin{proof}
    First, we show that $\theta^\top w^{t} \geq 1-2c_1$ for every $0 \leq t \leq \Delta T_1+\Delta T_2$.
    Since $\eta_1 \leq c_\eta d^{-\frac{p}{2}}$, setting $\bar{\varepsilon} = d^{-\frac{1}{2}}$, the assumption $\eta \leq c_\eta \bar{\varepsilon} d^{-1} \wedge c_\eta \bar{\varepsilon}^2$ of \cref{lem:weak-strongrecovery-1} is satisfied.
    By \cref{lem:weak-strongrecovery-1}, until the first $\tau < \Delta T_1 + \Delta T_2$ at which $\theta^\top w^\tau > 1 - \bar{\varepsilon}$, we have $\theta^\top w^t \geq \theta^\top w^0 -\bar{\varepsilon} \geq 1-2c_1$ for every $t <\tau$.
    If there exists $t > \tau$ with $\theta^\top w^t < 1 - \bar{\varepsilon}$, let $\tau'$ be the smallest such $t$.
    By \cref{lem:weak-onestep-2}, $|\theta^\top w^{\tau'} - \theta^\top w^{\tau'-1}| \leq C_1 \eta_1 \leq \bar{\varepsilon}$, and $\theta^\top w^{\tau'} \geq \theta^\top w^{\tau'-1} - \bar{\varepsilon} \geq 1-2\bar{\varepsilon} \geq 1-c_1$.
    Then, by \cref{lem:weak-strongrecovery-1}, until $\theta^\top w^t$ exceeds $1-\bar{\varepsilon}$, we have $\theta^\top w^t \geq \theta^\top w^{\tau'}-c_1\geq 1-2c_1$.
    Applying this argument repeatedly yields the desired result. In particular, $\theta^\top w^{\Delta T_1+\Delta T_2} \geq 1-2c_1$.

    Next, let $\Delta T_1 + \Delta T_2 \leq t\leq \Delta T_1+\Delta T_2 + T_{1, 3}$.
    Setting $\bar{\varepsilon} = \tilde{\varepsilon}$, the assumption $\eta_2 \leq c_\eta \bar{\varepsilon} d^{-1}\wedge c_\eta \bar{\varepsilon}^2$ of \cref{lem:weak-strongrecovery-1} is satisfied.
    Define
    \begin{align}
        T_{1, 3} =\left \lceil 3c_1(\eta_2 \tilde{\varepsilon} p \alpha_p  \beta_p)^{-1} \right \rceil.
    \end{align}
    If $\theta^\top w^t \leq 1 - \tilde{\varepsilon}$ for every $t = \Delta T_1 + \Delta T_2, \Delta T_1 + \Delta T_2+1,\dots, \Delta T_1+\Delta T_2 + T_{1,3}$, then by \cref{lem:weak-strongrecovery-1}
    \begin{align}
        \theta^\top w^t & \geq 1 - 2c_1 - c_1 \tilde{\varepsilon} + (t - \Delta T_1-\Delta T_2)\tilde{\varepsilon} \eta p \alpha_p \beta_p \\ &
        \geq 1 - 3c_1 + (t - \Delta T_1 - \Delta T_2) \tilde{\varepsilon} \eta p \alpha_p \beta_p
    \end{align}
    with high probability, but at $t = \Delta T_1 + \Delta T_2 + T_{1,3}$ we have $(\mathrm{RHS}) \geq 1$, a contradiction.
    Therefore, with high probability, there exists $t_3 \leq \Delta T_1 + \Delta T_2+T_{1, 3}$ such that $\theta^\top w^{t_3} > 1-\tilde{\varepsilon}$.

    Finally, we show $\theta^\top w^{\Delta T_1 + \Delta T_2+T_{1,3}} \geq 1 - 3\tilde{\varepsilon}$.
    If there exists $t > t_3$ with $\theta^\top w^t < 1 - \tilde{\varepsilon}$, let $\tau$ be the smallest such $t$.
    By \cref{lem:weak-onestep-2}, $|\theta^\top w^\tau -\theta^\top w^{\tau-1}| \leq C_1 \eta_2$, and by the same argument as before, $\theta^\top w^\tau \geq 1 - 2\tilde{\varepsilon}$.
    Then, by \cref{lem:weak-strongrecovery-1}, until $\theta^\top w^t > 1-\tilde{\varepsilon}$, we have $\theta^\top w^t \leq 1 - 3\tilde{\varepsilon}$.
    Repeating this argument, we obtain $\theta^\top w^t \geq 1 - 3\tilde{\varepsilon}$ with high probability for every $t_3\leq t\leq \Delta T_1+\Delta T_2+T_{1, 3}$.
\end{proof}

\subsection{Neuron filtering}\label{weak-filtering}

\begin{lem}\label{lem:weak-filtering-1}
    Let $V \subset \mathbb{R}^d$ be an $s = d^\alpha$-dimensional subspace containing $\theta \in V$, and let $W = V^\perp$ be its orthogonal complement.
    Let $x$ be a random variable on $S^{d-1}$ whose probability density function is symmetric under any rotation within $V$ about the axis $\theta$, and invariant under any rotation within $W$.
    Let $\Sigma = \mathbb{E}[xx^\top]$, and let $\lambda_1 \geq \lambda_2 \geq \dots \geq \lambda_d$ be its eigenvalues.
    If, for $\bar{\varepsilon} = o(1)$, $p = \mathbb{P}[|x^\top \theta| \geq 1 - \bar{\varepsilon}] = \Theta(1)$, then
    \begin{align}
        \lambda_1 = \Theta(1),\quad \lambda_2 = \lambda_3 = \dots = \lambda_s = O(d^{-\alpha}), \quad \lambda_{s+1} = \lambda_{s+2} = \dots = \lambda_d = O(d^{-1}),
    \end{align}
    and $\sum_{i=1}^d \lambda_i = 1$.
\end{lem}

\begin{proof}
    W.l.o.g., let $\theta = e_1$ and $V = \mathop{\mathrm{span}}(e_1, e_2, \dots, e_s)$.
    By the rotational symmetry of $x$ about the axis $\theta$, for $i, j \leq s$ with $i\neq j$, $\Sigma_{ij} = 0$.
    By the symmetry of $x$ under any rotation within $W$, for $i > s$ and $j \neq i$, $\Sigma_{ij} = 0$.
    Therefore, there exist $\lambda_1', \lambda_2',\dots, \lambda_d' \geq 0$ such that $\Sigma = \mathop{\mathrm{diag}}(\lambda_1', \lambda_2', \dots, \lambda_d')$.
    Here, $\lambda_1' = \mathbb{E}[(x^\top \theta)^2] \geq p(1-\bar{\varepsilon})^2 = \Theta(1)$.
    Furthermore, since $\sum_{i=1}^d \lambda_i' = \mathop{\mathrm{Tr}}(\mathbb{E}[xx^\top]) = 1$, we have $\lambda_2' = \lambda_3' = \dots \lambda_s' = O(s^{-1})$ and $\lambda'_{s+1} = \lambda'_{s+2} = \dots =\lambda_d' = O(d^{-1})$.
    Therefore, $\lambda_1'$ is the largest eigenvalue, and $\lambda_1 = \lambda_1'$ follows.
\end{proof}

\begin{lem}\label{lem:weak-filtering-2}
    Suppose $\bar{\varepsilon} = o(1)$, and let $x_1, x_2, \dots, x_N$ be i.i.d. random variables on $S^{d-1}$ satisfying the symmetry assumption of \cref{lem:weak-filtering-1}.
    Define $X_N = \frac{1}{N} \sum_{n=1}^N x_n x_n^\top$ and let $\hat{\theta}$ be its leading unit eigenvector.
    If we take $N \gtrsim \bar{\varepsilon}^{-2} \log d$, then $|\hat{\theta}^\top \theta| \geq 1-\bar{\varepsilon}$ with high probability.
    Furthermore, if $y \in S^{d-1}$ satisfies $|\theta^\top y| \geq 1-\bar{\varepsilon}$, then $|\hat{\theta}^\top y| \geq 1-2\bar{\varepsilon}$ with high probability.
\end{lem}

\begin{proof}
    Setting $X = \mathbb{E}[x_1 x_1^\top]$, by \cref{lem:technical-matrixbernstein}, for $t \geq 0$,
    \begin{align}
        \mathbb{P}\left[ \|X_{N} - X \| \geq t\right] \leq 2d\exp\left( -\frac{t^2 N^2 / 2}{ N \|X\|  + 2tN/3} \right).
    \end{align}
    Here we used $\| \sum_{n=1}^{N} \mathbb{E} [ (x_n x_n^\top)^2 ] \| = \| \sum_{n=1}^{N} \mathbb{E} [x_n x_n^\top] \| = N \|X\|$.
    In particular, substituting $t = \sqrt{\frac{C \log d}{N}} = o(1)$ gives $\mathbb{P}[\|X_{N} - X\| \geq t] \lesssim \exp(-C \log d)$.
    That is, $\|X_{N}- X\| = O(\sqrt{\log d/N})$ with high probability.
    By \cref{lem:weak-filtering-1} and \cref{lem:technical-yu},
    \begin{align}
        \sin\Theta(\hat{\theta}, \theta) = O\left(\sqrt{\log d/N}\right).
    \end{align}
    Furthermore, $1 - |\hat{\theta}^\top \theta| = O(\sqrt{\log d/N})$.
    Therefore, by taking $N \gtrsim \bar{\varepsilon}^{-2} \log d$, $|\hat{\theta}^\top \theta| \geq 1- \bar{\varepsilon}$ with high probability.
    If $y \in S^{d-1}$ satisfies $|\theta^\top y| \geq 1-\bar{\varepsilon}$, then
    \begin{align}
        |\hat{\theta}^\top y| \geq |\hat{\theta}^\top \theta| |\theta^\top y| \geq 1 - 2\bar{\varepsilon}.
    \end{align}
\end{proof}

\subsection{Second-layer training}\label{weak-secondlayer}
Let $\sigma$ be a polynomial of degree $q$.
\begin{lem}[\cite{oko2024learningsumdiversefeatures}, Lemma 29]\label{lem:weak-second-layer-1}
    Let $b_j \sim \mathrm{Unif}([-C_b, C_b])$ with $C_b = \tilde{O}(1)$, and let $h(s)$ be a polynomial of degree $q$.
    Then, for $v \in \mathbb{S}^{d-1}$, there exist $a_1, a_2,\dots, a_N$ such that
    \begin{align}
        \sup_{t = T_1+1,\dots,T_1+T_2} \left| \frac{1}{N} \sum_{j=1}^N a_j \sigma(v^\top x^t + b_j) - h(v^\top x^t) \right| = \tilde{O}(N^{-1})
    \end{align}
    with high probability.
    Here, $\sum_{j=1}^N a_j^2 = \tilde{O}(N)$ and $\sum_{j=1}^N |a_j| = \tilde{O}(N)$.
\end{lem}

Define $r_a(x) = \frac{1}{N^w} \sum_{m=1}^{N^w} a_n \sigma(\hat{w}_n^\top x + b_n)$.

\begin{lem}\label{lem:weak-second-layer-2}
    Let $\tilde{N}^w \gtrsim N^w \log{d}$.
    Then, there exists $a^* = (a^*_n)_{n=1}^{N^w} \in \mathbb{R}^{N^w}$ such that
    \begin{align}
        \frac{1}{T_2} \sum_{t=T_1+1}^{T_1+T_2} \left( r_{a^*}(x^t) - \pi_\kappa \sigma^*_\kappa(\theta_\kappa^\top x^t) \right)^2 \leq C_1 \pi_\kappa^2 (|N^w|^{-2} + \bar{\varepsilon}^2).
    \end{align}
    Here, $\|a^*\|_2^2 = \tilde{O}(\pi_\kappa^2 ({\tilde{N}^w})^2 |N^w|^{-1})$ and $\|a^*\|_1 = \tilde{O}(\pi_\kappa \tilde{N}^w)$.
\end{lem}

\begin{proof}
    Obtained from \cite{oko2024learningsumdiversefeatures}, Lemma 31 by setting $M=1$ and multiplying the objective function by $\pi_\kappa$.
\end{proof}

Let $\hat{a}$ be the ridge-regularized empirical risk minimizer
\begin{align}
    \hat{a} = \mathop{\mathrm{argmin}}_{a \in \mathbb{R}^{N^w}} \frac{1}{T_2}\sum_{t=T_1+1}^{T_1+T_2} \left( y^t - \frac{1}{N^w} \sum_{n=1}^{N^w} a_n \sigma_n(\hat{w}_n^\top x^t + b_n) \right)^2 + \frac{\lambda}{2} \|a\|^2_2.
\end{align}
The regularization parameter $\lambda > 0$ can be chosen so that
\begin{align}
    \hat{\mathcal{L}}(\hat{a}) \leq \hat{\mathcal{L}}(a^*),\quad \|\hat{a}\|_2 \leq \|a^*\|_2
\end{align}
hold \citep{oko2024learningsumdiversefeatures}.

\begin{lem}\label{lem:weak-second-layer-3}
    Let $\tilde{N}^w = \Theta(N^w_{\min} \log{d})$ and let $\sigma$ be a polynomial of degree $q$.
    Then, there exists $\lambda > 0$ such that the ridge estimator $\hat{a}$ satisfies, with probability $1 - o_d(1)$,
    \begin{align}
        \mathbb{E}_x[|r_{\hat{a}}(x) - r^*(x)|] \lesssim \pi_\kappa (|N^w|^{-1} + \tilde{\varepsilon}) + \pi_\kappa \sqrt{\frac{\log{d}}{T_2}}.
    \end{align}
    In particular, setting $T_2 = \tilde{\Theta}(\varepsilon^{-2})$, $\tilde{\varepsilon} = \tilde{\Theta}(\varepsilon)$, and $N^w = \tilde{\Theta}(\varepsilon^{-1})$, we have $\mathbb{E}_x[|r_{\hat{a}}(x) - r^*(x)|] \lesssim \pi_\kappa \varepsilon$.
\end{lem}
\begin{proof}
    Obtained from \cite{oko2024learningsumdiversefeatures}, Lemma 14 by setting $M=1$ and multiplying the objective function by $\pi_\kappa$.
\end{proof}

\section{Full proof of \texorpdfstring{\cref{thm:strong-informal}}{Theorem~\ref*{thm:strong-informal}}: W2S feature learning}\label{app:strong}

\paragraph{Notations.}
    Throughout this section, we consider multiple tasks $k = 1, 2, \dots, K$.
    We write $\theta_k$ for the true feature vector corresponding to task $k$, and define $\Sigma_k^\perp = I - \Sigma_k$ for each $k$.
    Write the Hermite expansions of the transformed teacher signal and the activation functions of strong-model neurons as
    \[
        \bar{r}^*_\kappa \exp(\bar{r}^*_\kappa) = \sum_{i=1}^\infty \frac{\bar{\alpha}_i}{\sqrt{i!}} \mathrm{He}_i,
        \qquad
        a_{k,n} \sigma_{k,n}^s(\cdot + b_{k,n}) = \sum_{i=1}^q \frac{\tilde{\beta}_{k,n,i}}{\sqrt{i!}} \mathrm{He}_i.
    \]
    In the training of the strong model (\cref{alg:w2s}), the parameter updates of each neuron can be analyzed independently.
    When focusing on a specific neuron $n \in [N_k]$ belonging to task $k \in [K]$, we sometimes suppress indices and write
    \begin{align}
        w^t_k = w^t_{k,n}, \quad a_k = a_{k,n}, \quad b_k = b_{k,n}, \quad \tilde{\beta}_{k,i} = \tilde{\beta}_{k,n,i}.
    \end{align}

    The subsequent arguments are carried out after rescaling the learning rate in the SGD update rule.
    Specifically, we set $\eta \leftarrow N_\kappa \pi_\kappa^{-1} \eta$, scaling by the task-$\kappa$ signal, so that the weight update for task $k$ in \cref{alg:w2s} becomes
    \begin{align}
        w_k^{t+1} \gets w_k^t + \eta g_k^t, \qquad g_k^t = \frac{N_\kappa}{N_k} \pi_\kappa^{-1} \nabla_{w_k} \bar{y}^t a_k \sigma_k(w_k^{t\top}x + b_k).
    \end{align}
    Define the scaled expansion coefficients $\beta_{k,i} = \pi_k^{-1} \tilde{\beta}_{k,i}$.
    Then $c_\beta < \beta_{k,i} < C_\beta$ for all $k \in [K]$ and $i \leq q$.

    Set $c_w = 2\bar{\alpha}_2 \beta_{\kappa,2} + \sqrt{12}\bar{\alpha}_2 \beta_{\kappa,4}$.
    By \cref{asm:strong-init}~(ii), $2\bar{\alpha}_2 \tilde{\beta}_{\kappa,2} + \sqrt{12}\bar{\alpha}_2 \tilde{\beta}_{\kappa,4} > 0$, so $c_w > 0$.
    By \cref{lem:strong-transformation-2}, $|\bar{\alpha}_2| = \tilde{\Theta}(1)$, and hence $c_w = \tilde{\Theta}(1)$.

    We take constants of order $\mathrm{polylog}(d)$ satisfying the following ordering:
    \begin{align}
        &C_1 \lesssim c_1^{-1} \lesssim c_2^{-1} \lesssim C_2 \lesssim c_r^{-1} \lesssim c_3^{-1} \lesssim \begin{cases}
            \delta^{-1} \lesssim
            \begin{cases}
                c_\varepsilon^{-1} \\
                c_\eta^{-1}
            \end{cases} \\
            c_\chi^{-1}
        \end{cases}
        = \tilde{O}(1), \\
        &c_w^{-1} \lesssim C_1 \lesssim C_2 \lesssim C_3 \lesssim c_3^{-1}, \\
        &\tilde{\delta}^{-1} \lesssim \delta^{-1}.
    \end{align}
    The precise ordering of these constants will be determined so as to keep the subsequent proofs consistent.
    These constants are defined independently of those in \cref{app:weak,app:sft}.

The formal versions of \cref{thm:strong-informal,prop:preservation-informal} rely on the following assumption.

\begin{asm}[Strong-model initialization and Hermite conditions]\label{asm:strong-init}
    The following conditions hold for the strong model training in \cref{alg:w2s}.
    \begin{enumerate}[leftmargin=*]
        \item[\textnormal{(i)}] \textbf{Generative exponent.}
            The link function $\sigma^*_\kappa$ of the task-$\kappa$ reward is an even polynomial, so that $\GE(\sigma^*_\kappa) = 2$.
        \item[\textnormal{(ii)}] \textbf{Initialization and Hermite structure.}
            The strong model is initialized with $N^s := \sum_{k=1}^K N_k$ neurons and parameters $\Theta_k = (a^0_{k,n}, b^0_{k,n}, w^0_{k,n})_{n=1}^{N_k}$.
            For $k = \kappa$, $\|\Sigma_\kappa^\perp w^0_{\kappa,n}\|^2 \leq c_r s^{-1/2}$ and $\Sigma_\kappa w^0_{\kappa,n} / \|\Sigma_\kappa w^0_{\kappa,n}\| \sim \mathrm{Unif}(S^{d-1} \cap V_\kappa)$.
            Writing $\bar{r}^*_\kappa \exp(\bar{r}^*_\kappa) = \sum_{i \geq 0} \frac{\bar{\alpha}_i}{\sqrt{i!}}\mathrm{He}_i$ and $a^0_{k,n}\sigma_{k,n}(\cdot + b^0_{k,n}) = \sum_{i=0}^q \frac{\tilde{\beta}_{k,n,i}}{\sqrt{i!}}\mathrm{He}_i$, the constants $c_\beta, C_\beta$ satisfy $c_\beta \pi_k \leq |\tilde{\beta}_{k,n,i}| \leq C_\beta \pi_k$ for all $k, n, i \leq q$.
            Moreover, $\Theta(N_\kappa)$ of the task-$\kappa$ neurons satisfy
            \begin{align}
                \bar{\alpha}_i \tilde{\beta}_{\kappa,n,i} > 0 \;\;(i = p), \qquad \bar{\alpha}_i \tilde{\beta}_{\kappa,n,i} \geq 0 \;\;(p < i \leq q),
            \end{align}
            \begin{align}
                2\bar{\alpha}_2 \tilde{\beta}_{\kappa,n,2} + \sqrt{12}\,\bar{\alpha}_2 \tilde{\beta}_{\kappa,n,4} > 0, \qquad
                i\bar{\alpha}_i \tilde{\beta}_{\kappa,n,i} + \sqrt{(i+2)(i+1)}\,\bar{\alpha}_i \tilde{\beta}_{\kappa,n,i+2} \geq 0 \;\;(3 \leq i \leq q),
            \end{align}
            together with $1 - (2\bar{\alpha}_2\tilde{\beta}_{\kappa,n,2})^{-1}\sqrt{12}\,\bar{\alpha}_2\tilde{\beta}_{\kappa,n,4} = \tilde{\Omega}(1)$.
    \end{enumerate}
\end{asm}

A concrete initialization satisfying \cref{asm:strong-init}~(ii) is provided in \cref{strong-initialization} (\cref{lem:strong-random-init}).

\begin{rem}\label{rem:ge-assumption}
  The assumption $\GE(\sigma^*_\kappa) = 2$ is essential for feature preservation (\cref{prop:preservation}).
  If $\GE(\sigma^*_\kappa) = 1$, the transformed teacher signal retains a nonzero $i=1$ Hermite component $\bar{\alpha}_1$, contributing a constant gradient term $\bar{\alpha}_1 \beta_{k,1} \theta_\kappa$ to off-target neurons regardless of their alignment $\theta_\kappa^\top \tilde{w}^t_k$.
  Since this term does not decay with $\theta_\kappa^\top \tilde{w}^t_k$, it cannot be absorbed into the remainder $R^t$ in \cref{lem:strong-onestep-1}\,(ii), and off-target neurons accumulate alignment with $\theta_\kappa$ over iterations, causing forgetting of pre-trained features for $k \neq \kappa$.
\end{rem}

\setcounter{thm}{\numexpr\value{num@thm@strong}-1\relax}
\begin{thm}[W2S feature alignment; formal version of \cref{thm:strong-informal}]\label{thm:strong}
    Assume \cref{asm:additive,asm:weak,asm:distribution,asm:strong-init}.
    Let $\epsStrongEff := \epsWeakLOne \vee \epsWeakAlign \vee (1 - \lambda_\kappa)$, and suppose $\epsStrongEff \leq c_\varepsilon s^{-1/2}$.
    Fix $\tilde{\varepsilon}$ with $\tilde{\varepsilon} = \tilde{\Omega}(\epsStrongEff)$ and set
    \begin{align}
        T = \tilde{\Theta}\!\left(s^{3/2} \;\vee\; s\,\tilde{\varepsilon}^{-1} \log \tilde{\varepsilon}^{-1} \;\vee\; \tilde{\varepsilon}^{-2} \log \tilde{\varepsilon}^{-1} \right).
    \end{align}
    Then, for a suitable choice of the learning-rate schedule $(\eta^t)_{t \geq 0}$ and the temperature $\tilde{\rho}$, with probability $1 - 6\tilde{\delta}$, \cref{alg:w2s} outputs weights such that at least $\Theta(N_\kappa)$ of the task-$\kappa$ neurons satisfy
    \begin{align}
        \theta_\kappa^\top \hat{w}_{\kappa,n} \geq 1 - \tilde{\varepsilon}.
    \end{align}
    Under the scaling $\tilde{\varepsilon} = \tilde{\Theta}(s^{-1/2})$, this reduces to $T = \tilde{O}(s^{3/2})$.
\end{thm}

\begin{proof}
    Fix a neuron $n \in [N_\kappa]$.
    By \cref{asm:strong-init}~(ii) and \cref{lem:strong-init-alignment}, $\theta_\kappa^\top w^0_{\kappa,n} \geq s^{-1/2}$ with probability $\Theta(1)$, and $\|\Sigma_\kappa^\perp w^0_{\kappa,n}\|^2 \leq c_r s^{-1/2}$.
    Applying \cref{lem:strong-weakrecovery-2} with $\eta_1 \leq c_\eta s^{-3/2}$, there exists $t_1 \leq T_{1,1} = \tilde{\Theta}(\eta_1^{-1})$ such that $\theta_\kappa^\top w^{t_1}_{\kappa,n} > c_1$ and $\|\Sigma_\kappa^\perp w^t_{\kappa,n}\|^2 \leq 3c_r s^{-1/2}$ for all $t \leq t_1$, with probability $1 - \tilde{\delta} - \delta$.
    Since $3c_r s^{-1/2} \leq c_r c_1$, the conditions of \cref{lem:strong-amplification-2} hold at $t_1$, and there exists $t_2 \leq T_{1,2} = \tilde{\Theta}(\eta_1^{-1})$ such that $\theta_\kappa^\top w^{t_1+t_2}_{\kappa,n} > 1 - c_1$ and $\|\Sigma_\kappa^\perp w^{t_1+t_2}_{\kappa,n}\|^2 \leq 2c_r$, with probability $1 - \tilde{\delta}$.
    The conditions of \cref{lem:strong-strongrecovery-4} then hold at $t_1 + t_2$.
    Switching to $\eta_2 \leq c_\eta C_2^{-1}\tilde{\varepsilon} s^{-1} \wedge c_\eta C_2^{-2}\tilde{\varepsilon}^2$ and running for $T_{1,3} + T_{1,4} = \tilde{\Theta}(\eta_2^{-1}\log\tilde{\varepsilon}^{-1})$ additional steps gives $\theta_\kappa^\top \hat{w}_{\kappa,n} \geq 1 - \tilde{\varepsilon}$ with high probability.
    Let $N_\kappa' = \Theta(N_\kappa)$ denote the number of well-initialized neurons (those with $\theta_\kappa^\top w^0_{\kappa,n} \geq s^{-1/2}$), and let $F = \#\{n \in [N_\kappa'] : \theta_\kappa^\top \hat{w}_{\kappa,n} < 1 - \tilde{\varepsilon}\}$.
    A union bound and linearity of expectation give $\mathbb{E}[F] \leq 3N_\kappa'\tilde{\delta}$, and Markov's inequality gives $\mathbb{P}(F \geq N_\kappa'/2) \leq 6\tilde{\delta}$, so at least $N_\kappa'/2 = \Theta(N_\kappa)$ neurons satisfy the bound with probability $1 - 6\tilde{\delta}$.
    The total sample count is $T_{1,1} + T_{1,2} + T_{1,3} + T_{1,4} = \tilde{\Theta}\!\left(s^{3/2} \vee (s\,\tilde{\varepsilon}^{-1} \vee \tilde{\varepsilon}^{-2})\log\tilde{\varepsilon}^{-1}\right)$, which reduces to $\tilde{O}(s^{3/2})$ under $\tilde{\varepsilon} = \tilde{\Theta}(s^{-1/2})$.
\end{proof}

\setcounter{cor}{\numexpr\value{num@prop@preservation}-1\relax}
\begin{prop}[Preservation of pre-trained features; formal version of \cref{prop:preservation-informal}]\label{prop:preservation}
    Under the conditions of \cref{thm:strong}, let $\chi_k := N_\kappa \pi_k / (N_k \pi_\kappa)$ for $k \neq \kappa$, and assume $\chi_k \leq c_\chi \wedge c_2(\log\tilde{\varepsilon}^{-1})^{-1}$.
    Then, at the output of \cref{alg:w2s}, with high probability, for every $k \neq \kappa$ and every $n \in [N_k]$,
    \begin{align}
        \theta_k^\top \hat{w}_{k,n}
        \;\geq\; \theta_k^\top w^0_{k,n}
        \;-\; c_2 C_3 \bigl(\|\Sigma_\kappa w^0_{k,n}\|^2 + \tilde{\varepsilon}\bigr)
        \;-\; c_2 \chi_k s^{-1/2}.
    \end{align}
    In particular, the drop in alignment is $\tilde{O}(\|\Sigma_\kappa w^0_{k,n}\|^2 + \tilde{\varepsilon} + \chi_k s^{-1/2})$.
\end{prop}

\begin{proof}[Proof of \cref{prop:preservation}]
    Fix $k \neq \kappa$ and $n \in [N_k]$, and set $\bar{\varepsilon} = C_2^{-1}\tilde{\varepsilon}$.
    Since $C_2\epsStrongEff \leq \tilde{\varepsilon}$ by the hypothesis of \cref{lem:strong-strongrecovery-4}, we have $\epsStrongEff \leq \bar{\varepsilon}$.
    Since $\chi_k \leq c_\chi$, \cref{lem:strong-unforgetting-2}(ii) gives
    \begin{align}
        \theta_k^\top \hat{w}_{k,n}
        &\geq \theta_k^\top w^0_{k,n}
        - \chi_k(C_3 \vee 2C_2\log\bar{\varepsilon}^{-1})\bigl(\|\Sigma_\kappa w^0_{k,n}\|^2 + \epsStrongEff\bigr)
        - \chi_k c_3\bar{\varepsilon}\log\bar{\varepsilon}^{-1}
        - c_2(\chi_k s^{-3/4} + \chi_k^2 s^{-1/2}).
    \end{align}
    Using $\chi_k \leq c_2(\log\tilde{\varepsilon}^{-1})^{-1}$ and $\log\bar{\varepsilon}^{-1} = \log(C_2\tilde{\varepsilon}^{-1}) \leq 2\log\tilde{\varepsilon}^{-1}$ (for $\tilde{\varepsilon}$ small), the log factors cancel:
    $\chi_k(C_3 \vee 2C_2\log\bar{\varepsilon}^{-1}) \leq c_2 C_3$ and $\chi_k c_3\bar{\varepsilon}\log\bar{\varepsilon}^{-1} \leq c_2 c_3\bar{\varepsilon} \leq c_2 c_3\tilde{\varepsilon}$.
    Since $\epsStrongEff \leq \bar{\varepsilon} \leq \tilde{\varepsilon}$ and $s^{-3/4} \leq s^{-1/2}$ and $\chi_k^2 s^{-1/2} \leq \chi_k s^{-1/2}$, substituting yields the stated bound.
\end{proof}

The remainder of this section proves \cref{thm:strong,prop:preservation}.

\subsection{Nonlinear Transformation of the Teacher Signal}

We show that transforming the weak model's output reduces its information exponent.
Let $r^*$ be a polynomial of degree $q$ with $\mathbb{E}_{z \sim \mathcal{N}(0,1)}[r^*(z)] = 0$ and $\mathbb{E}_{z \sim \mathcal{N}(0,1)}[(r^*(z))^2] = 1$.
Let $\rho = \Theta(\log^{C_\rho} d)$ be a temperature parameter with $C_\rho \geq q/2 + 1$, and define $\bar{r}^*(z) = \mathop{\mathrm{clip}}(\rho^{-1} r^*;\, \pm 1/\log d)$.
By \cref{lem:strong-transformation-1}, $\bar{r}^* = \rho^{-1} r^*$ with high probability.

\begin{lem}[\cite{oko2024pretrained}, Corollary~17]\label{lem:strong-transformation-1}
    Let $\theta \in S^{d-1}$ and $x \sim \mathcal{N}(0, I_d)$. Then $|r^*(\theta^\top x)| \lesssim (\log d)^{q/2}$ with high probability.
\end{lem}

Write the Hermite expansion of $\bar{r}^* \exp(\bar{r}^*)$ as $\bar{r}^* \exp(\bar{r}^*) = \sum_{i=1}^\infty \frac{r_i}{i!} \mathrm{He}_i$.

\begin{lem}[Clip version of \cite{nishikawa2025nonlinear}, Lemma~9]\label{lem:strong-transformation-2}
    Let $p \geq \GE(r^*)$ and let $e_p = \min\{i \geq 1 \mid \IE((r^*)^i) \leq p\}$, where we assume $e_p < \infty$.
    Then $r_p = \Theta((\log d)^{-C_\rho e_p})$.

    On the other hand, when $\GE(r^*) > p \geq 1$, $r_p = O(d^{-C})$ for some sufficiently large constant $C$.
\end{lem}

\begin{proof}
    The proof follows \cite{nishikawa2025nonlinear}, Lemma~9.
    We first treat the case $p \geq \GE(r^*)$.
    By Taylor expansion,
    \begin{align}
        r_p &= \mathbb{E}_{z \sim \mathcal{N}(0,1)}\!\left[ \bar{r}^*(z) \exp(\bar{r}^*(z)) \mathrm{He}_p(z) \right]
        = \mathbb{E}\!\left[ \sum_{i=1}^\infty \frac{1}{(i-1)!} (\bar{r}^*(z))^i \mathrm{He}_p(z) \right] \\
        &= \frac{1}{(e_p-1)!} \mathbb{E}[(\bar{r}^*(z))^{e_p} \mathrm{He}_p(z)]
        + \sum_{i=e_p}^{C_\rho e_p - 1} \frac{1}{i!} \mathbb{E}[(\bar{r}^*(z))^{i+1} \mathrm{He}_p(z)]
        + \sum_{i \geq C_\rho e_p} \frac{1}{i!} \mathbb{E}[(\bar{r}^*(z))^{i+1} \mathrm{He}_p(z)] \\
        &=: I_1 + I_2 + I_3.
    \end{align}

    For $I_1$, we write
    \begin{align}
        &\mathbb{E}[(\bar{r}^*(z))^{e_p} \mathrm{He}_p(z)] \\
        &= \rho^{-e_p} \mathbb{E}[(r^*(z))^{e_p} \mathrm{He}_p(z)]
        - \rho^{-e_p} \mathbb{E}\!\left[ \left( (r^*(z))^{e_p} - \mathop{\mathrm{sign}}(r^*(z)) \left(\tfrac{\rho}{\log d}\right)^{e_p} \right) \mathrm{He}_p(z) \mathbf{1}(|r^*(z)| \geq \rho/\log d) \right].
    \end{align}
    By the Cauchy--Schwarz inequality,
    \begin{align}
        \mathbb{E}[|(r^*(z))^{e_p} \mathrm{He}_p(z) \mathbf{1}(|r^*(z)| \geq \rho/\log d)|]
        &\leq \mathbb{E}[|(r^*(z))^{e_p} \mathrm{He}_p(z)|^2]^{1/2} \mathbb{P}[|r^*(z)| \geq \rho/\log d]^{1/2}, \\
        \mathbb{E}\!\left[|\mathop{\mathrm{sign}}(r^*(z)) \left(\rho(\log d)^{-1}\right)^{e_p} \mathrm{He}_p(z) \mathbf{1}(|r^*(z)| \geq \rho/\log d)|\right]
        &\leq \left(\rho(\log d)^{-1}\right)^{e_p} \mathbb{E}[|\mathrm{He}_p(z)|^2]^{1/2} \mathbb{P}[|r^*(z)| \geq \rho/\log d]^{1/2}.
    \end{align}
    By \cref{lem:strong-transformation-1}, $\mathbb{P}[|r^*(z)| \geq \rho/\log d] = O(d^{-1})$, so $I_1 = \Theta((\log d)^{-C_\rho e_p})$.

    For $I_2$, similarly to $I_1$, the leading term of $\mathbb{E}[(\bar{r}^*(z))^{i+1} \mathrm{He}_p(z)]$ for $i = e_p, \dots, C_\rho e_p - 1$ is $\rho^{-(i+1)} \mathbb{E}[(r^*(z))^{i+1} \mathrm{He}_p(z)]$.
    By Cauchy--Schwarz inequality, $|\rho^{-(i+1)} \mathbb{E}[(r^*(z))^{i+1} \mathrm{He}_p(z)]| \lesssim \rho^{-(i+1)} \sqrt{\mathbb{E}[(r^*(z))^{2(i+1)}]}$, so
    \begin{align}
        |I_2| \lesssim \sum_{i=e_p}^{C_\rho e_p - 1} \rho^{-(i+1)} \sqrt{\mathbb{E}[(r^*(z))^{2(i+1)}]} \lesssim \rho^{-(e_p+1)} = O((\log d)^{-C_\rho(e_p+1)}).
    \end{align}

    For $I_3$, by the boundedness of $\bar{r}^*$ and Cauchy--Schwarz, for $i \geq C_\rho e_p$,
    \begin{align}
        \left| \frac{1}{i!} \mathbb{E}[(\bar{r}^*(z))^{i+1} \mathrm{He}_p(z)] \right|
        \leq \frac{1}{i!} \mathbb{E}[\bar{r}^*(z)^{2(i+1)}]^{1/2} \mathbb{E}[\mathrm{He}_p(z)^2]^{1/2}
        \lesssim \frac{1}{i!} (\log d)^{-(i+1)},
    \end{align}
    giving $|I_3| \lesssim \sum_{i \geq C_\rho e_p} \frac{1}{i!} (\log d)^{-(i+1)} \lesssim (\log d)^{-(C_\rho e_p + 1)}$.
    Combining, $I_1 = \Theta((\log d)^{-C_\rho e_p})$ dominates $I_2$ and $I_3$, giving $r_p = \Theta((\log d)^{-C_\rho e_p})$.

    We now treat the case $\mathrm{ge}(r^*) > p \geq 1$.
    Write
    \begin{align}
        r_p &= \mathbb{E}_{z \sim \mathcal{N}(0,1)}[\bar{r}^*(z) \exp(\bar{r}^*(z)) \mathrm{He}_p(z)] \\
        &= \mathbb{E}[\rho^{-1} r^*(z) \exp(\rho^{-1} r^*(z)) \mathrm{He}_p(z)] \\
        &\quad - \mathbb{E}\!\left[ \left( \rho^{-1} r^*(z) \exp(\rho^{-1} r^*(z)) - \mathop{\mathrm{sign}}(r^*(z))(\log d)^{-1} \exp((\log d)^{-1}) \right) \mathrm{He}_p(z) \mathbf{1}(|r^*(z)| \geq \rho/\log d) \right].
    \end{align}
    Since $p < \mathrm{ge}(r^*)$, $\mathbb{E}[\rho^{-1} r^*(z) \exp(\rho^{-1} r^*(z)) \mathrm{He}_p(z)] = 0$.
    By \cref{lem:strong-transformation-1},
    \begin{align}
        &\mathbb{E}\!\left[|\rho^{-1} r^*(z) \exp(\rho^{-1} r^*(z)) \mathrm{He}_p(z) \mathbf{1}(|r^*(z)| \geq \rho/\log d)|\right] \\
        &\leq \mathbb{E}[|\rho^{-1} r^*(z) \exp(\rho^{-1} r^*(z)) \mathrm{He}_p(z)|^2]^{1/2} \mathbb{P}[|r^*(z)| \geq \rho/\log d]^{1/2} = O(d^{-C}), \\
        &\mathbb{E}\!\left[|\mathop{\mathrm{sign}}(r^*(z))(\log d)^{-1} \exp((\log d)^{-1}) \mathrm{He}_p(z) \mathbf{1}(|r^*(z)| \geq \rho/\log d)|\right] \\
        &\leq (\log d)^{-1} e^{(\log d)^{-1}} \mathbb{E}[|\mathrm{He}_p(z)|^2]^{1/2} \mathbb{P}[|r^*(z)| \geq \rho/\log d]^{1/2} = O(d^{-C}).
    \end{align}
    Therefore $r_p = O(d^{-C})$.
\end{proof}

\subsection{Initialization}\label{strong-initialization}

\begin{lem}\label{lem:strong-init-alignment}
    Under \cref{asm:strong-init}~(ii), $\theta_\kappa^\top w^0_{\kappa,n} \geq s^{-1/2}$ with probability $\Theta(1)$.
\end{lem}
\begin{proof}
    Write $w^0 := w^0_{\kappa,n}$ and let $\tilde{w} := \Sigma_\kappa w^0 / \|\Sigma_\kappa w^0\|$.
    By assumption, $\tilde{w} \sim \mathrm{Unif}(S^{d-1} \cap V_\kappa)$, so identifying $V_\kappa \cong \mathbb{R}^s$ the argument of \cref{lem:weak-init} gives $\theta_\kappa^\top \tilde{w} \geq 2s^{-1/2}$ with probability $\Theta(1)$.
    Since $\|\Sigma_\kappa^\perp w^0\|^2 \leq c_r s^{-1/2}$ by assumption, we have $\|\Sigma_\kappa w^0\| \geq \sqrt{1 - c_r s^{-1/2}} \geq 1/2$.
    Therefore $\theta_\kappa^\top w^0 = \|\Sigma_\kappa w^0\| \cdot \theta_\kappa^\top \tilde{w} \geq s^{-1/2}$.
\end{proof}

The following lemma shows that the Hermite sign conditions in \cref{asm:strong-init}~(ii) are satisfied with positive probability under a natural random initialization.

\begin{lem}\label{lem:strong-random-init}
    Let $a_n \sim \mathrm{Unif}\{\pm\pi_\kappa\}$ and $\xi_{n,i} \sim \mathrm{Unif}\{\pm\delta_i\}$ for $0 \leq i \leq q$, where $(\delta_i)_{i=0}^q$ satisfies $\delta_0 = \delta_1 = \delta_2 = 1$ and $0 < \delta_{i+2} < \frac{i}{(i+1)(i+2)} \delta_i$ for every $i = 1, 2, \dots, q-2$.

    Then $a_n \sigma_n = \sum_{i=0}^q a_n \xi_{n,i} \mathrm{He}_i$ has its $i$-th Hermite coefficient ($1 \leq i \leq q$) taking the desired sign with probability $\Omega(1)$.
    Taking these $a_n$, $\sigma_n$, and $b_n = 0$ as initial values therefore gives a strong model initialization satisfying \cref{asm:strong-init}~(ii).
\end{lem}

\begin{proof}
    The coefficient $\beta_{\kappa,n,i} = \sqrt{i!}\, a_n \xi_{n,i}$ takes any given sign with probability $2^{-q}$.
    By the construction of $\xi_{n,i}$, we have $i|\beta_{\kappa,n,i}| > \sqrt{(i+2)(i+1)}|\beta_{\kappa,n,i+2}|$ for every neuron and every $i = 1, 2, \dots, q$.
    Combined with $\bar{\alpha}_2 = \tilde{\Theta}(1)$ from \cref{lem:strong-transformation-2}, the desired conclusion follows.
\end{proof}

\subsection{Error of the Teacher Signal in the Subspace}
\begin{lem}\label{lem:strong-subspace-1}
    Let $x \sim \mathcal{N}(0, I_d)$.
    Let $\sigma_n$ be a degree-$q$ polynomial with uniformly bounded coefficients over $n = 1, 2, \dots, N$, and let $f(x) = \frac{1}{N} \sum_{n=1}^N a_n \sigma_n(w_n^\top x + b_n)$.
    Suppose $\max_{n} |b_n| = \tilde{O}(1)$, and that there exists $\theta \in S^{d-1}$ such that $|w_n^\top \theta| \geq 1 - \tilde{\varepsilon}$ for all $n$.
    Let $V \subset \mathbb{R}^d$ be a subspace, and let $P_V$, $P_V^\perp$ denote the orthogonal projections onto $V$ and its complement respectively.
    Then the following hold for the $L^1$ norm of $f$ restricted to $V$.
    \begin{enumerate}[leftmargin=*]
        \item[(i)] If $\theta \in V$, then $\mathbb{E}_{x}[|f(P_V x)|] \leq \mathbb{E}_x[|f(x)|] + \tilde{O}(N^{-1}\|a\|_1 \tilde{\varepsilon})$.
        \item[(ii)] If $\theta \perp V$, then $\mathbb{E}[|f(P_V x)|] = \tilde{O}(N^{-1}\|a\|_1 + \tilde{\varepsilon}).$
    \end{enumerate}
\end{lem}

\begin{proof}
    Set $u_n = w_n - (w_n^\top \theta)\theta$, so that
    \begin{align}\label{eq:strong-subspace-1}
        \|u_n\| = \sqrt{\|w_n - (w_n^\top \theta)\theta \|^2} = \sqrt{1 - (w_n^\top \theta)^2} \leq \sqrt{2\tilde{\varepsilon}}.
    \end{align}
    For $x \sim \mathcal{N}(0, I_d)$, write $x_1 = P_V x$ and $x_2 = P_V^\perp x$; these are independent.

    \textit{Proof of (i).}
    Set $\delta_n = w_n^\top P_V^\perp x$ and expand $\sigma_n(w_n^\top x + b_n)$ in a Taylor series around $w_n^\top P_V x + b_n$:
    \begin{align}\label{eq:strong-subspace-2}
        f(x) &= \frac{1}{N} \sum_{n=1}^N a_n \sigma(w_n^\top P_V x + b_n + \delta_n) \\ &
        = \frac{1}{N} \sum_{n=1}^N a_n \sigma(w_n^\top P_V x + b_n) + \frac{1}{N} \sum_{n=1}^N \sum_{i=1}^q \frac{1}{i!} a_n (\delta_n)^i \sigma_n^{(i)}(w_n^\top P_V x + b_n) \\ &
        = f(P_V x) + \frac{1}{N} \sum_{n=1}^N \sum_{i=1}^q \frac{1}{i!} a_n (\delta_n)^i \sigma_n^{(i)}(w_n^\top P_V x + b_n)
    \end{align}
    Since $\mathbb{E}_{P_V^\perp x}[\delta_n] = 0$ and $\frac{(P_V^\perp u_n)^\top}{\|P_V^\perp u_n\|} x \sim \mathcal{N}(0,1)$, we have by \eqref{eq:strong-subspace-1}
    \begin{align}
        |\delta_n| = |w_n^\top P_V^\perp x | = |u_n^\top P_V^\perp x| \leq \|u_n\| \left|\frac{(P_V^\perp u_n)^\top}{\|P_V^\perp u_n\|} x\right| = \tilde{O}(\tilde{\varepsilon}^{1/2})
    \end{align}
    with high probability.
    Taking the expectation of \eqref{eq:strong-subspace-2} over $x_2 = P_V^\perp x$,
    \begin{align}
        |f(x_1)| &= \left| \mathbb{E}_{x_2} [f(x_1 + x_2)] - \frac{1}{N} \sum_{n=1}^N \sum_{i=1}^q \frac{1}{i!} a_n \mathbb{E}_{x_2} [(\delta_n)^i] \sigma_n^{(i)}(w_n^\top x_1 + b_n) \right| \\ &
        \leq \mathbb{E}_{x_2} [|f(x_1 + x_2)|] + \frac{1}{N} \sum_{n=1}^N \sum_{i=2}^q \frac{1}{i!} \left| a_n \mathbb{E}_{x_2} [(\delta_n)^i] \sigma_n^{(i)}(w_n^\top x_1 + b_n) \right|
    \end{align}
    Taking the expectation over $x_1 = P_V x$ as well, and using the high-probability bound $|\delta_n|^i = \tilde{O}(\tilde{\varepsilon}^{1/2})$,
    \begin{align}
        \mathbb{E}_{x_1, x_2}[|f(x_1)|] \leq \mathbb{E}_{x_1, x_2}[|f(x_1 + x_2)|] + \tilde{O}(N^{-1}\|a\|_1 \tilde{\varepsilon}).
    \end{align}

    \textit{Proof of (ii).}
    Since $\theta \perp V$, we have $w_n^\top P_V x + b_n = u_n^\top P_V x + b_n$, so
    \begin{align}
        f(P_V x) = \frac{1}{N} \sum_{n=1}^N \sum_{i=0}^q \frac{1}{i!} a_n (u_n^\top P_V x)^i \sigma_n^{(i)}(b_n).
    \end{align}
    By the same argument as in (i), $\mathbb{E}_x[u_n^\top P_V x] = 0$ and $\mathbb{E}_x[|u_n^\top P_V x|^i] = \tilde{O}(\tilde{\varepsilon}^{1/2})$ for $i \geq 2$, giving
    \begin{align}
        \mathbb{E}_x[|f(P_V x)|] \leq \frac{1}{N} \sum_{n=1}^N |a_n \sigma_n(b_n)| + \tilde{O}(\tilde{\varepsilon}) = \tilde{O}(N^{-1}\|a\|_1 + \tilde{\varepsilon}). 
    \end{align}
\end{proof}

\subsection{Gradient Decomposition}

The following holds for the effective gradient $g_k^t$.
Part (i) in \cref{lem:strong-onestep-1} separates the signal and noise terms for the target task $\kappa$, analogously to \cref{lem:weak-onestep-1}.
Part (ii) describes the effect of the teacher signal on neurons belonging to tasks $k \neq \kappa$; no clear signal term exists in this case.
The task weighting introduces a scaling factor $\chi_k = \frac{N_\kappa \pi_k}{N_k \pi_\kappa}$ for the gradient.

\begin{lem}[Formal version of \cref{lem:gradient-decomp-informal}]\label{lem:strong-onestep-1}
    The gradient decomposes as follows.
    \begin{enumerate}[leftmargin=*]
        \item[(i)] For $k = \kappa$: setting $\tilde{w}^t_\kappa = \Sigma_\kappa w^t_\kappa / \|\Sigma_\kappa w^t_\kappa\|$,
        \begin{align}
            g^t_\kappa = \sum_{i=1}^q \lambda_\kappa \left[ i\bar{\alpha}_i \beta_{\kappa,i} (\theta_\kappa^\top \tilde{w}_\kappa^t)^{i-1} \theta_\kappa + \sqrt{(i+2)(i+1)} \bar{\alpha}_i \beta_{\kappa,i+2} (\theta_\kappa^\top \tilde{w}_\kappa^t)^i \tilde{w}_\kappa^t \right] + Z^t + R^t.
        \end{align}
        Here $Z^t$ is a mean-zero random variable with $\|Z^t\| = \tilde{O}(s^{1/2})$ with high probability, and $|v^\top Z^t| = \tilde{O}(1)$ with high probability for every $v \in S^{d-1}$.
        With probability $\lambda_\kappa$, $\|\Sigma_\kappa^\perp Z^t\| = \tilde{O}((1-\lambda_\kappa)s^{1/2})$, and $\mathbb{E}[\|\Sigma_\kappa^\perp Z^t\|^2] \lesssim (1-\lambda_\kappa)s$.
        The remainder $R^t$ satisfies $\|R^t\| = \tilde{O}((\epsWeakLOne \vee \epsWeakAlign \vee (1-\lambda_\kappa) \vee \|\Sigma_\kappa^\perp w^t_\kappa\|^2) s^{1/2})$ with high probability, $\|\Sigma_\kappa^\perp R^t\| = \tilde{O}((1-\lambda_\kappa)s^{1/2})$ with high probability, and for every $v \in S^{d-1}$,
        \[
            |v^\top R^t| = \tilde{O}(\epsWeakLOne \vee \epsWeakAlign \vee (1-\lambda_\kappa) \vee \|\Sigma_\kappa v\| \|\Sigma_\kappa^\perp w^t_\kappa\|^2)
        \]
        with high probability.
        Furthermore, $\|g^t_\kappa\| = \tilde{O}(s^{1/2})$ with high probability, and $g^{t\top}_\kappa v = \tilde{O}(1)$ with high probability for every $v \in \mathbb{R}^d$ with $\|v\| = O(1)$.

        \item[(ii)] For $k \neq \kappa$:
        \begin{align}
            g^t_k = \chi_k (Z^t + R^t), \qquad \chi_k = \frac{N_\kappa \pi_k}{N_k \pi_\kappa}.
        \end{align}
        Here $Z^t$ is mean-zero with $\|Z^t\| = \tilde{O}(s^{1/2})$ with high probability\ and $|v^\top Z^t| = \tilde{O}(1)$ with high probability for every $v \in S^{d-1}$.
        With probability $\lambda_\kappa$, $\Sigma_\kappa^\perp Z^t = 0$, and $\mathbb{E}[\|\Sigma_\kappa^\perp Z^t\|^2] \lesssim (1-\lambda_\kappa)s$.
        The remainder satisfies $\|R^t\| = \tilde{O}((\epsWeakLOne \vee \epsWeakAlign \vee (1-\lambda_\kappa) \vee \|\Sigma_\kappa w^t_k\|) s^{1/2})$ with high probability, and for every $v \in S^{d-1}$,
        \[
            |v^\top R^t| = \tilde{O}(\epsWeakLOne \vee \epsWeakAlign \vee (1-\lambda_\kappa) \vee \|\Sigma_\kappa v\|\|\Sigma_\kappa w^t_k\|)
        \]
        with high probability.
        Furthermore, $\|g^t_k\| = \tilde{O}(\chi_k s^{1/2})$ with high probability, and $g^{t\top}_k v = \tilde{O}(\chi_k)$ with high probability for every $v \in \mathbb{R}^d$ with $\|v\| = O(1)$.
    \end{enumerate}
\end{lem}

\begin{proof}
    \textit{Proof of (i).}
    By \cref{lem:strong-transformation-2}, $|\bar{\alpha}_2| = \tilde{\Theta}(1)$, so by \cref{asm:strong-init}~(ii) the absolute values of the Hermite coefficients of $\pi_\kappa^{-1} a_\kappa \sigma_\kappa(\cdot + b_\kappa)$ satisfy $|\beta_{\kappa,i}| \leq C_\beta = \tilde{O}(1)$ for all $i \leq q$.
    Define
    \begin{align}
        R^t_1 &= \nabla_{w_\kappa} \mathbb{E}_{x \sim \mu}[\bar{y}^t a_\kappa \sigma_\kappa(w_\kappa^{t\top} x + b_\kappa)]
        - \nabla_{w_\kappa} \mathbb{E}_{x \sim \mu}[\bar{r}^*_\kappa(\theta_\kappa^\top x) \exp(\bar{r}^*_\kappa(\theta_\kappa^\top x)) a_\kappa \sigma_\kappa(w_\kappa^{t\top} x + b_\kappa)], \\
        R^t_2 &= \nabla_{w_\kappa} \sum_{k' \neq \kappa} \lambda_{k'} \mathbb{E}_{x \sim \mathcal{N}(0,\Sigma_{k'})}[\bar{r}^*_\kappa(\theta_\kappa^\top x) \exp(\bar{r}^*_\kappa(\theta_\kappa^\top x)) a_\kappa \sigma_\kappa(w^{t\top}_\kappa x + b_\kappa)], \\
        R_3^t &= \lambda_\kappa \mathbb{E}_{x \sim \mathcal{N}(0,\Sigma_\kappa)}\!\left[\bar{r}^*_\kappa(\theta_\kappa^\top x) \exp(\bar{r}^*_\kappa(\theta_\kappa^\top x)) a_\kappa \left(\sigma'_\kappa(w^{t\top}_\kappa x + b_\kappa) - \sigma'_\kappa(\tilde{w}^{t\top}_\kappa x + b_\kappa)\right) x\right], \\
        R^t &= \pi_\kappa^{-1}(R_1^t + R_2^t + R_3^t).
    \end{align}
    Then
    \begin{align}
        &\nabla_{w_\kappa} \mathbb{E}_{x \sim \mu}[\bar{y}^t a_\kappa \sigma_\kappa(w_\kappa^{t\top} x + b_\kappa)] \\
        &= \lambda_\kappa \mathbb{E}_{x \sim \mathcal{N}(0,\Sigma_\kappa)}\!\left[\bar{r}^*_\kappa(\theta_\kappa^\top x) \exp(\bar{r}^*_\kappa(\theta_\kappa^\top x)) a_\kappa \sigma'_\kappa(\tilde{w}^{t\top}_\kappa x + b_\kappa) x\right] + R_1^t + R_2^t + R_3^t.
    \end{align}

    For $R_1^t$, we have
    \begin{align}
        R^t_1 &= \nabla_{w_\kappa} \mathbb{E}_{x \sim \mu} \left[ \bar{y}^t a_\kappa \sigma_\kappa(w_\kappa^{t \top}x + b_\kappa) \right] - \nabla_{w_\kappa} \mathbb{E}_{x \sim \mu} \left[ \bar{r}^*_\kappa(\theta_\kappa^\top x) \exp(\bar{r}^*_\kappa(\theta_\kappa^\top x) a_\kappa \sigma_\kappa(w_\kappa^{t\top} x + b_\kappa) \right] \\ &
        = \mathbb{E}_{x \sim \mu} \left[ \left( \bar{r}^w(x) \exp(\bar{r}^w(x)) - \bar{r}^*_\kappa(\theta_\kappa^\top x) \exp(\bar{r}^*_\kappa(\theta_\kappa^\top x) \right) a_\kappa \sigma'_\kappa(w_\kappa^{t \top}x + b_\kappa) x \right] \\ &
        = \mathbb{E}_{x \sim \mu}\left[ \sum_{j=0}^\infty \frac{1}{j!} \left( \left( \bar{r}^w(x)^{j+1} \right) - \left(\bar{r}^*_\kappa(\theta_\kappa^\top x)\right)^{j+1} \right) a_\kappa \sigma'_\kappa(w_\kappa^{t \top}x + b_\kappa) x \right].
    \end{align}
    Thus,
    \begin{align}
        &\|R^t_1\| \\& 
        \leq  \sum_{j=0}^\infty \frac{1}{j!} \mathbb{E}_{x \sim \mu} \left[ \left| \left( \bar{r}^w(x) - \bar{r}^*_\kappa(\theta_\kappa^\top x) \right) \left[ (\bar{r}^w(x))^{j} + (\bar{r}^w(x))^{j-1} \bar{r}^*_\kappa (\theta_\kappa^\top x) + \cdots + (\bar{r}^w(x))^{j} \right] \right| \left\| a_\kappa \sigma'_\kappa( w^{t\top}_\kappa x + b_\kappa) x \right\| \right] \\ &
        \leq \left( \sum_{j=0}^\infty \frac{j+1}{j!} (\log d)^{-j} \right) \mathbb{E}_{x \sim \mu} \left[ \left| \left( \bar{r}^w(x) - \bar{r}^*_\kappa(\theta_\kappa^\top x) \right) \right| \left\| a_\kappa \sigma'_\kappa( w^{t\top}_\kappa x + b_\kappa) x \right\| \right]
    \end{align}
    Since $\pi_\kappa^{-1}\|a_\kappa \sigma'_\kappa(w^{t\top}_\kappa x + b_\kappa) x\| = \tilde{O}(s^{1/2})$ with high probability for $x \sim \mu$, and by \cref{lem:strong-subspace-1}, $\mathbb{E}_{x \sim \mu}[|\bar{r}^w(x) - \bar{r}^*_\kappa(\theta_\kappa^\top x)|] = \tilde{O}(\epsWeakLOne \vee \epsWeakAlign \vee (1-\lambda_\kappa))$, we obtain $\pi_\kappa^{-1}\|R^t_1\| = \tilde{O}((\epsWeakLOne \vee \epsWeakAlign \vee (1-\lambda_\kappa))s^{1/2})$.

    For $R_2^t$: since $\pi_\kappa^{-1}\|\bar{r}^*_\kappa(\theta_\kappa^\top x)\exp(\bar{r}^*_\kappa(\theta_\kappa^\top x)) a_\kappa \sigma'_\kappa(w^{t\top}_\kappa x + b_\kappa) x\| = \tilde{O}(s^{1/2})$ for $x \sim \mathcal{N}(0, \Sigma_{k'})$ with $k' \neq \kappa$, we get $\pi_\kappa^{-1}\|R^t_2\| = \tilde{O}((1-\lambda_\kappa)s^{1/2})$.

    For $R_3^t$: by Taylor expansion,
    \begin{align}
        |\sigma'_\kappa(w^{t\top}_\kappa x + b_\kappa) - \sigma'_\kappa(\tilde{w}^{t\top}_\kappa x + b_\kappa)|
        \leq \left|\sum_{i=1}^{q-1} \frac{1}{i!}((w^t_\kappa - \tilde{w}^t_\kappa)^\top x)^i \sigma_\kappa^{(i+1)}(\tilde{w}^{t\top}_\kappa x + b_\kappa)\right|.
    \end{align}
    For $x \sim \mathcal{N}(0, \Sigma_\kappa)$, $(w^t_\kappa - \tilde{w}^t_\kappa)^\top x = (\|\Sigma_\kappa w^t_\kappa\| - 1)\tilde{w}^{t\top}_\kappa x$, and since $1 - \|\Sigma_\kappa w^t_\kappa\| = \frac{\|\Sigma_\kappa^\perp w^t_\kappa\|^2}{1 + \|\Sigma_\kappa w^t_\kappa\|} \leq \|\Sigma_\kappa^\perp w^t_\kappa\|^2$ and $\tilde{w}^{t\top}_\kappa x \sim \mathcal{N}(0,1)$, we get $\pi_\kappa^{-1}\|R_3^t\| = \tilde{O}(\|\Sigma_\kappa^\perp w^t_\kappa\|^2 s^{1/2})$.

    Combining, $\|R^t\| = \tilde{O}((\epsWeakLOne \vee \epsWeakAlign \vee (1-\lambda_\kappa) \vee \|\Sigma_\kappa^\perp w^t_\kappa\|^2)s^{1/2})$.
    Since $\Sigma_\kappa^\perp R^t = \pi_\kappa^{-1}\Sigma_\kappa^\perp(R_1^t + R_2^t)$, we have $\|\Sigma_\kappa^\perp R^t\| = \tilde{O}((1-\lambda_\kappa)s^{1/2})$.
    For every $v \in S^{d-1}$, since $v^\top x$ is sub-Gaussian with parameter $\|\Sigma_k v\|$ for $x \sim \mu_k$, $|v^\top R^t| = \tilde{O}(\epsWeakLOne \vee \epsWeakAlign \vee (1-\lambda_\kappa) \vee \|\Sigma_\kappa v\| \|\Sigma^\perp_\kappa w^t_\kappa\|^2)$ with high probability.

    By the same computation as in \cref{lem:weak-onestep-1},
    \begin{align}
        &\lambda_\kappa \mathbb{E}_{x \sim \mathcal{N}(0,\Sigma_\kappa)}\!\left[\bar{r}^*_\kappa(\theta_\kappa^\top x)\exp(\bar{r}^*_\kappa(\theta_\kappa^\top x)) a_\kappa \sigma'_\kappa(\tilde{w}^{t\top}_\kappa x + b_\kappa) x\right] \\
        &= \sum_{i=1}^q \lambda_\kappa \left[i\bar{\alpha}_i \tilde{\beta}_{\kappa,i} (\theta_\kappa^\top \tilde{w}^t_\kappa)^{i-1}\theta_\kappa + \sqrt{(i+2)(i+1)}\bar{\alpha}_i \tilde{\beta}_{\kappa,i+2} (\theta_\kappa^\top \tilde{w}^t_\kappa)^i \tilde{w}^t_\kappa\right].
    \end{align}
    Setting $Z^t = \mathbb{E}_{x \sim \mu}[g^t_\kappa] - g^t_\kappa$, we have $\mathbb{E}[Z^t] = 0$, $\|Z^t\| = \tilde{O}(s^{1/2})$, and $|v^\top Z^t| = \tilde{O}(1)$ for every $v \in S^{d-1}$.
    By construction, with probability $\lambda_\kappa$, $\|\Sigma_\kappa^\perp Z^t\| = \tilde{O}((1-\lambda_\kappa)s^{1/2})$, giving $\mathbb{E}[\|\Sigma_\kappa^\perp Z^t\|^2] \lesssim (1-\lambda_\kappa)s$.

    \textit{Proof of (ii).}
    Define $\tilde{w}^t_k = \Sigma_k w^t_k / \|\Sigma_k w^t_k\|$. By an argument identical to (i), $|\beta_{k,i}| \leq C_\beta = \tilde{O}(1)$.
    Define $R_1^t, R_2^t, R_3^t, R_4^t$ as
    \begin{align}
        R^t_1 &= \nabla_{w_k} \mathbb{E}_{x \sim \mu}[\bar{y}^t a_k \sigma_k(w_k^{t\top} x + b_k)]
        - \nabla_{w_k} \mathbb{E}_{x \sim \mu}[\bar{r}^*_\kappa(\theta_\kappa^\top x)\exp(\bar{r}^*_\kappa(\theta_\kappa^\top x)) a_k \sigma_k(w_k^{t\top} x + b_k)], \\
        R^t_2 &= \nabla_{w_k} \sum_{k' \neq \kappa} \lambda_{k'} \mathbb{E}_{x \sim \mathcal{N}(0,\Sigma_{k'})}[\bar{r}^*_\kappa(\theta_\kappa^\top x)\exp(\bar{r}^*_\kappa(\theta_\kappa^\top x)) a_k \sigma_k(w^{t\top}_k x + b_k)], \\
        R_3^t &= \lambda_\kappa \mathbb{E}_{x \sim \mathcal{N}(0,\Sigma_\kappa)}\!\left[\bar{r}^*_\kappa(\theta_\kappa^\top x)\exp(\bar{r}^*_\kappa(\theta_\kappa^\top x)) a_k \left(\sigma'_k(w^{t\top}_k x + b_k) - \sigma'_k(\tilde{w}^{t\top}_k x + b_k)\right) x\right], \\
        R_4^t &= \lambda_\kappa \mathbb{E}_{x \sim \mathcal{N}(0,\Sigma_\kappa)}\!\left[\bar{r}^*_\kappa(\theta_\kappa^\top x)\exp(\bar{r}^*_\kappa(\theta_\kappa^\top x)) a_k \sigma'_k(\tilde{w}^{t\top}_k x + b_k) x\right] = \lambda_\kappa \bar{\alpha}_1 \sigma'_k(b_k)\theta_\kappa,
    \end{align}
    and set $R^t = \pi_k^{-1}(R_1^t + R_2^t + R_3^t + R_4^t)$.
    By the same bounds as in (i), $\pi_k^{-1}\|R_1^t\| = \tilde{O}((\epsWeakLOne \vee \epsWeakAlign \vee (1-\lambda_\kappa))s^{1/2})$ and $\pi_k^{-1}\|R_2^t\| = \tilde{O}((1-\lambda_\kappa)s^{1/2})$.
    For $R_3^t$, since $w^{t\top}_k x / \|\Sigma_\kappa w^t_k\| \sim \mathcal{N}(0,1)$ for $x \sim \mathcal{N}(0,\Sigma_\kappa)$,
    \begin{align}
        R^t_3 &= \lambda_\kappa \mathbb{E}_{x \sim \mathcal{N}(0, \Sigma_\kappa)} \left[\bar{r}^*_\kappa(\theta_\kappa^\top x) \exp(\bar{r}^*_\kappa(\theta_\kappa^\top x)) a_k  \left( \sum_{i=1}^{q-1} \frac{1}{i!} ((w^t_k - \tilde{w}^t_k)^\top x)^i \sigma^{(i+1)}_k (\tilde{w}^{t\top}_k x + b_k) \right) x \right] \\ &
        = \sum_{i=1}^{q-1} \frac{1}{i!} \lambda_\kappa a_k \sigma_k^{(i+1)}(b_k) \mathbb{E}_{x \sim \mathcal{N}(0, \Sigma_\kappa)} \left[ \bar{r}^*_\kappa(\theta_\kappa^\top x) \exp(\bar{r}^*_\kappa (\theta_\kappa^\top x)) (w^{t\top}_k x)^i x\right].
    \end{align}
    giving $\pi_k^{-1}\|R_3^t\| = \tilde{O}(\|\Sigma_\kappa w^t_k\|s^{1/2})$.
    By \cref{asm:strong-init} and \cref{lem:strong-transformation-2}, $\pi_k^{-1}\|R_4^t\| = \tilde{O}(d^{-C})$ for a sufficiently large constant $C$.
    Setting $Z^t = \chi_k^{-1}(\mathbb{E}_{x \sim \mu}[g^t_k] - g^t_k)$, one verifies that $R^t, Z^t, g^t_k$ satisfy the stated bounds by the same argument as in (i).
\end{proof}

For the rest of the analysis, define the effective error of the strong model training as
\begin{align}
    \epsStrongEff = \epsWeakLOne \vee \epsWeakAlign \vee (1 - \lambda_\kappa).
\end{align}

\begin{lem}\label{lem:strong-onestep-2}
    Let $\eta = \eta^t \leq c_\eta s^{-1}$ and $\theta_\kappa^\top w^t_\kappa \geq \frac{1}{2}s^{-1/2}$. Then
    \begin{align}
        &\theta_\kappa^\top w_\kappa^t + \eta\lambda_\kappa \theta_\kappa^\top P_{w_\kappa^t}^\perp(Z^t + R^t) - \eta^2 C_1^2 (\theta_\kappa^\top w_\kappa^t) s \\
        &\quad + \eta\lambda_\kappa \sum_{i=1}^q \left[ i\bar{\alpha}_i\beta_{\kappa,i} \|\Sigma_\kappa w_\kappa^t\|^{-(i-1)} (\theta_\kappa^\top w_\kappa^t)^{i-1}(1 - (\theta_\kappa^\top w_\kappa^t)^2) \right. \\
        &\qquad\qquad \left. + \sqrt{(i+2)(i+1)}\bar{\alpha}_i\beta_{\kappa,i+2} \|\Sigma_\kappa w_\kappa^t\|^{-(i+1)} \|\Sigma_\kappa^\perp w_\kappa^t\|^2 (\theta_\kappa^\top w_\kappa^t)^{i+1} \right] \\
        &\leq \theta_\kappa^\top w_\kappa^{t+1} \\
        &\leq \theta_\kappa^\top w_\kappa^t + \eta\lambda_\kappa \theta_\kappa^\top P_{w_\kappa^t}^\perp(Z^t + R^t) \\
        &\quad + \eta\lambda_\kappa \sum_{i=1}^q \left[ i\bar{\alpha}_i\beta_{\kappa,i} \|\Sigma_\kappa w_\kappa^t\|^{-(i-1)} (\theta_\kappa^\top w_\kappa^t)^{i-1}(1 - (\theta_\kappa^\top w_\kappa^t)^2) \right. \\
        &\qquad\qquad \left. + \sqrt{(i+2)(i+1)}\bar{\alpha}_i\beta_{\kappa,i+2} \|\Sigma_\kappa w_\kappa^t\|^{-(i+1)} \|\Sigma_\kappa^\perp w_\kappa^t\|^2 (\theta_\kappa^\top w_\kappa^t)^{i+1} \right],
    \end{align}
    where $Z^t$, $R^t$ satisfy the conditions in \cref{lem:strong-onestep-1}.
    Furthermore, $|\theta_\kappa^\top w^{t+1}_\kappa - \theta_\kappa^\top w^t_\kappa| = \tilde{O}(\eta)$.
\end{lem}

\begin{proof}
    From $\tilde{w}^t_\kappa = \Sigma_\kappa w^t_\kappa / \|\Sigma_\kappa w^t_\kappa\|$,
    \begin{align} \label{eq:strong-onestep-1}
        \theta_\kappa^\top \tilde{w}^t_\kappa = \frac{\theta_\kappa^\top w^t_\kappa}{\| \Sigma_\kappa w^t_\kappa \|}, \quad \theta_\kappa^\top P_{w^t_\kappa}^\perp \tilde{w}^t_\kappa = \theta_\kappa^\top \tilde{w}^t_\kappa - (\theta_\kappa^\top w^t_\kappa) (w^{t\top}_\kappa \tilde{w}^t_\kappa) = \left(\frac{1}{\| \Sigma_\kappa w^t_\kappa \|} - \| \Sigma_\kappa w^t_\kappa \|\right) \theta_\kappa^\top w^t_\kappa.
    \end{align}
    Setting $g = g^t_\kappa$ and arguing as in \cref{lem:weak-onestep-2},
    \begin{align}
        \theta_\kappa^\top w_\kappa^{t+1} & = \theta_\kappa^\top \frac{w_\kappa^t - \eta P_{w_\kappa^t}^\perp (-g)}{\|w_\kappa^t - \eta P_{w_\kappa^t}^\perp (-g)\|}  \\
        &\ge \theta_\kappa^\top w_\kappa^t + \eta \theta_\kappa^\top P_{w_\kappa^t}^\perp g - \frac{1}{2} (\theta_\kappa^\top w_\kappa^t) \eta^2 \|g\|^2 - \frac{1}{2} \eta^3 |\theta_\kappa ^\top  P_{w_\kappa^t}^\perp g | \|g\|^2 \\ &
        \ge \theta_\kappa^\top w_\kappa^t + \eta \lambda_\kappa \theta_\kappa^\top P_{w_\kappa^t}^\perp (Z^t + R^t) - \frac{1}{2} (\theta_\kappa^\top w_\kappa^t) \eta^2 C_1^2 s - \frac{1}{2} \eta^3 C_1^3 s^{\frac{3}{2}} \\ &
        \quad + \eta \lambda_\kappa \sum_{i=1}^q \left[ i \alpha_i \beta_{\kappa,i} (\theta_\kappa^\top \tilde{w}_\kappa^t)^{i-1} (\theta_\kappa^\top P_{w_\kappa^t}^\perp \theta_\kappa) + \sqrt{(i+2)(i+1)} \alpha_i \beta_{\kappa,i+2} (\theta_\kappa^\top \tilde{w}_\kappa^t)^i (\theta_\kappa^\top P_{w_\kappa^t}^\perp \tilde{w}_\kappa^t) \right] \\ &
        \geq \theta_\kappa^\top w_\kappa^t + \eta \lambda_\kappa \theta_\kappa^\top P_{w_\kappa^t}^\perp (Z^t + R^t) - \eta^2 C_1^2 (\theta_\kappa^\top w_\kappa^t) s \\ &
        \quad + \eta \lambda_\kappa \sum_{i=1}^q \left[ i \alpha_i \beta_{\kappa,i} \| \Sigma_\kappa w_\kappa^t \|^{-(i-1)} (\theta_\kappa^\top w_\kappa^t)^{i-1} (1 - (\theta_\kappa^\top w_\kappa^t)^2) \right.  \\
        &\qquad \left. + \sqrt{(i+2)(i+1)} \alpha_i \beta_{\kappa,i+2} \| \Sigma_\kappa w_\kappa^t \|^{-(i+1)} (1 - \| \Sigma_\kappa w_\kappa^t \|^2) (\theta_\kappa^\top w_\kappa^t)^{i+1} \right]
    \end{align}
    where the second inequality uses \cref{lem:strong-onestep-1}, and the third uses $\eta \leq c_\eta s^{-1}$, $\theta_\kappa^\top w^t_\kappa \geq \frac{1}{2}s^{-1/2}$, and \eqref{eq:strong-onestep-1}.
    The upper bound follows from
    \begin{align}
        \theta_\kappa^\top w_\kappa^{t+1} &\leq \theta_\kappa^\top ( w^t_\kappa - \eta P^\perp_{w^t_\kappa} (-g) ) \\ &
        = \theta_\kappa^\top w_\kappa^t + \eta \lambda_\kappa \theta_\kappa^\top P_{w_\kappa^t}^\perp (Z^t + R^t) \\ &
        \quad + \eta \lambda_\kappa \sum_{i=1}^q \left[ i \alpha_i \beta_{\kappa,i} \| \Sigma_\kappa w_\kappa^t \|^{-(i-1)} (\theta_\kappa^\top w_\kappa^t)^{i-1} (1 - (\theta_\kappa^\top w_\kappa^t)^2) \right.  \\
        &\qquad \left. + \sqrt{(i+2)(i+1)} \alpha_i \beta_{\kappa,i+2} \| \Sigma_\kappa w_\kappa^t \|^{-(i+1)} (1 - \| \Sigma_\kappa w_\kappa^t \|^2) (\theta_\kappa^\top w_\kappa^t)^{i+1} \right].
    \end{align}
    Finally, $|\theta_\kappa^\top w^{t+1}_\kappa - \theta_\kappa^\top w^t_\kappa| \leq \eta|\theta_\kappa^\top P^\perp_{w^t_\kappa} g| + \frac{1}{2}|\theta_\kappa^\top w^t_\kappa|\eta^2\|g\|^2 + \frac{1}{2}\eta^3|\theta_\kappa P^\perp_{w^t_\kappa} g|\|g\|^2 = \tilde{O}(\eta)$, since $\eta \leq c_\eta s^{-1}$.
\end{proof}

\subsection{Upper Bound on Complement-Subspace Deviation}

\begin{lem}[Formal version of \cref{lem:deviation-informal}]\label{lem:strong-deviation-1}
    Let $\eta = \eta^t \leq c_\eta s^{-1}$.
    For every $t = 0, 1, 2, \dots$, if $\theta_\kappa^\top w^t_\kappa \geq 0$ then $\|\Sigma_\kappa^\perp w_\kappa^{t+1}\|^2 - \|\Sigma_\kappa^\perp w_\kappa^t\|^2 \leq 7C_1\eta$.
    Moreover, let $\tau, \tau' > 0$ and suppose $\theta_\kappa^\top w^t_\kappa \geq 0$ for $t = \tau, \tau+1, \dots, \tau+\tau'-1$.
    Then there exists a constant $C > 0$ independent of $\tau, \tau'$ such that, with probability $1 - \delta'$, for every $t = 0, 1, \dots, \tau'$,
    \begin{align}
        \|\Sigma_\kappa^\perp w_\kappa^{\tau+t}\|^2
        &\leq \|\Sigma_\kappa^\perp w^\tau_\kappa\|^2
        + \sum_{t'=0}^{t-1}\!\left[2C_1\eta\|\Sigma_\kappa^\perp w^{\tau+t'}_\kappa\|^4
        + 3C_1^2\eta^2\|\Sigma_\kappa^\perp w^{\tau+t'}_\kappa\|^2
        + 8C_1\eta\,\epsStrongEff\right] \\
        &\quad + C\!\left[\left((1-\lambda_\kappa)^{\frac{1}{2}} + \|\Sigma_\kappa^\perp w^\tau_\kappa\|^2 + (C_1\eta\tau')^2\right){\delta'}^{-\frac{1}{2}}\eta{\tau'}^{\frac{1}{2}}
        + (1-\lambda_\kappa){\delta'}^{-1}\eta^2 s\tau'\right].
    \end{align}
    Furthermore, defining
    \begin{align}
        Q^0 &= \|\Sigma_\kappa^\perp w^\tau_\kappa\|^2 \vee C_1\eta + 8C_1\eta\,\epsStrongEff\,\tau' \\
        &\quad + C\!\left[(1-\lambda_\kappa)^{\frac{1}{2}}{\delta'}^{-\frac{1}{2}}\eta{\tau'}^{\frac{1}{2}}
        + (1-\lambda_\kappa){\delta'}^{-1}\eta^2 s\tau'
        + C_1(\|\Sigma_\kappa^\perp w^\tau_\kappa\|^2 + 7C_1\eta\tau')\eta{\tau'}^{\frac{1}{2}}\right], \\
        Q^t &= Q^{t-1} + 5C_1\eta(Q^{t-1})^2,
    \end{align}
    we have $\|\Sigma_\kappa^\perp w^{\tau+t}_\kappa\|^2 \leq Q^t$ for all $t = 0, 1, \dots, \tau'$.
\end{lem}

\begin{proof}
    First, note that
    \begin{align}\label{eq:strong-deviation-1}
        \Sigma_\kappa^\perp P_{w_\kappa^t}^\perp \theta_\kappa = - (\theta_\kappa^\top w_\kappa^t) \Sigma_\kappa^\perp w_\kappa^t , \quad \Sigma_\kappa^\perp P_{w_\kappa^t}^\perp \tilde{w}_\kappa^t = - (\tilde{w}_\kappa^{t\top} w_\kappa^t) \Sigma_\kappa^\perp w_\kappa^t = - \|\Sigma_\kappa w_\kappa^t\| \Sigma_\kappa^\perp w_\kappa^t.
    \end{align}
    Set $g = g^t_\kappa$ and define
    \[
        \xi^t = 1 - \eta\lambda_\kappa \sum_{i=1}^q \left(i\bar{\alpha}_i\beta_{\kappa,i} + \sqrt{(i+2)(i+1)}\bar{\alpha}_i\beta_{\kappa,i+2}\right)(\theta_\kappa^\top \tilde{w}_\kappa^t)^{i-1}(\theta_\kappa^\top w_\kappa^t).
    \]
    When $\theta_\kappa^\top w^t_\kappa \geq 0$, we have $0 \leq \xi^t \leq 1$. By \cref{lem:strong-onestep-1} and \eqref{eq:strong-deviation-1},
    \begin{align}
        \Sigma_\kappa^\perp P_{w_\kappa^t}^\perp g
        &= -\sum_{i=1}^q \lambda_\kappa\left(i\bar{\alpha}_i\beta_{\kappa,i} + \sqrt{(i+2)(i+1)}\bar{\alpha}_i\beta_{\kappa,i+2}\right)
        (\theta_\kappa^\top \tilde{w}_\kappa^t)^{i-1}(\theta_\kappa^\top w_\kappa^t)\Sigma_\kappa^\perp w_\kappa^t
        + \Sigma_\kappa^\perp P_{w_\kappa^t}^\perp(Z^t + R^t).
    \end{align}
    Hence
    \begin{align}
        \|\Sigma_\kappa^\perp w_\kappa^{t+1}\|^2 &= \frac{(w_\kappa^t + \eta P_{w_\kappa^t}^\perp g)^\top \Sigma_\kappa^\perp (w_\kappa^t + \eta P_{w_\kappa^t}^\perp g)}{\|w_\kappa^t + \eta P_{w_\kappa^t}^\perp g\|^2} \\ &
        \le \|\xi^t \Sigma_\kappa^\perp w_\kappa^t + \eta \Sigma_\kappa^\perp P_{w_\kappa^t}^\perp (Z^t + R^t)\|^2 \\ &
        = (\xi^t)^2 \|\Sigma_\kappa^\perp w_\kappa^t\|^2 + 2\eta \xi^t w_\kappa^{t\top} \Sigma_\kappa^\perp P_{w_\kappa^t}^\perp (Z^t + R^t) + \eta^2 \|\Sigma_\kappa^\perp P_{w_\kappa^t}^\perp (Z^t + R^t)\|^2 \\ &
        \leq (\xi^t)^2 \|\Sigma_\kappa^\perp w_\kappa^t\|^2 + 2\eta \xi^t (w^{t\top}_\kappa \Sigma^\perp_\kappa Z^t - \|\Sigma^\perp_\kappa w^t_\kappa\|^2 w^{t\top}_\kappa Z^t + w^t_\kappa \Sigma^\perp_\kappa R^t - \|\Sigma^\perp_\kappa w^t_\kappa\|^2 w^t_\kappa R^t) \\ &
        \quad + 2\eta^2 \left[ \|\Sigma_\kappa^\perp Z^t\|^2 + \|\Sigma^\perp_\kappa R^t\|^2 + \|\Sigma^\perp_\kappa w^t_\kappa\|^2 \left((w^{t\top}_\kappa Z^t)^2 + (w^{t\top}_\kappa R^t)^2 \right) \right] \\ &
        \leq (\xi^t)^2 \|\Sigma^\perp_\kappa w^t_\kappa\|^2 + 2C_1 \eta \xi^t \|\Sigma^\perp_\kappa w^t_\kappa\|^4 + 3C_1^2 \eta^2 \| \Sigma_\kappa^\perp w^t_\kappa\|^2 \\ &
        \quad + 4C_1 \eta \xi^t \epsStrongEff \|\Sigma^\perp_\kappa w^t_\kappa\| + 4C_1^2 \eta^2 \epsStrongEff^2 s \\ &
        \quad + 2 \left(\eta \xi^t w^{t\top}_\kappa \Sigma^\perp_\kappa Z^t - \eta \xi^t \|\Sigma^\perp_\kappa w^t_\kappa\|^2 w^{t\top}_\kappa Z^t + \eta^2\|\Sigma^\perp_\kappa Z^t\|^2 \right) \\ &
        \leq \|\Sigma^\perp_\kappa w^t_\kappa\|^2 + 2C_1 \eta \|\Sigma^\perp_\kappa w^t_\kappa\|^4 + 3C_1^2 \eta^2 \| \Sigma_\kappa^\perp w^t_\kappa\|^2 + 8C_1 \eta  \epsStrongEff \\ &
        \quad + 2 \left(\eta \xi^t w^{t\top}_\kappa \Sigma^\perp_\kappa Z^t - \eta \xi^t \|\Sigma^\perp_\kappa w^t_\kappa\|^2 w^{t\top}_\kappa Z^t + \eta^2\|\Sigma^\perp_\kappa Z^t\|^2 \right).
        \label{eq:strong-deviation-2}
     \end{align}
    Here the second inequality uses $w_\kappa^{t\top}\Sigma_\kappa^\perp P_{w_\kappa^t}^\perp Z^t = w_\kappa^t\Sigma_\kappa^\perp Z^t - \|\Sigma_\kappa^\perp w^t_\kappa\|^2 w_\kappa^t Z^t$ and similarly for $R^t$, together with
    \begin{align}
        \|\Sigma_\kappa^\perp P_{w_\kappa^t}^\perp(Z^t+R^t)\|^2
        \leq 2(\|\Sigma_\kappa^\perp Z^t\|^2 + \|\Sigma_\kappa^\perp R^t\|^2)
        + 2\|\Sigma_\kappa^\perp w^t_\kappa\|^2((w_\kappa^{t\top}Z^t)^2 + (w_\kappa^{t\top}R^t)^2).
    \end{align}
    The third inequality uses $w_\kappa^{t\top}\Sigma_\kappa^\perp R^t \leq C_1\,\epsStrongEff\|\Sigma_\kappa^\perp w^t_\kappa\|$, $w_\kappa^{t\top}R^t \leq C_1(\epsStrongEff \vee \|\Sigma_\kappa^\perp w^t_\kappa\|^2)$, $\|\Sigma_\kappa^\perp R^t\|^2 \leq C_1^2\epsStrongEff^2 s$, $(w_\kappa^{t\top}R^t)^2 \leq C_1^2(\epsStrongEff^2 \vee \|\Sigma_\kappa^\perp w^t\|^4)$, and $(w_\kappa^{t\top}Z^t)^2 \leq C_1^2$.
    The fourth uses $\theta_\kappa^\top w^t_\kappa \geq 0$.
    The dominant fluctuation terms in \eqref{eq:strong-deviation-2} are $2C_1\eta\|\Sigma_\kappa^\perp w^t_\kappa\|^4$, $2\eta\xi^t w^{t\top}_\kappa\Sigma_\kappa^\perp Z^t$, and $-2\eta\xi^t\|\Sigma_\kappa^\perp w^t_\kappa\|^2 w_\kappa^{t\top}Z^t$, from which $\|\Sigma_\kappa^\perp w^{t+1}_\kappa\|^2 - \|\Sigma_\kappa^\perp w^t_\kappa\|^2 \leq 7C_1\eta$.

    Summing \eqref{eq:strong-deviation-2} over $t' = 0, \dots, t-1$ gives the stated bound, where the stochastic terms are controlled via Doob's maximal inequality for the components in $V_\kappa^\perp$, exploiting the low occurrence probability $1 - \lambda_\kappa$, and standard concentration inequalities for the components in $V_\kappa$, as detailed below. The sequences $\{\sum_{t'=0}^{t-1}\xi^{\tau+t'} w^{(\tau+t')\top}_\kappa\Sigma_\kappa^\perp Z^{\tau+t'}\}$ and $\{\sum_{t'=0}^{t-1}\xi^{\tau+t'}\|\Sigma_\kappa^\perp w^{\tau+t'}_\kappa\|^2 w^{(\tau+t')\top}_\kappa Z^{\tau+t'}\}$ are martingales, and $\{\sum_{t'=0}^{t-1}\|\Sigma_\kappa^\perp Z^{\tau+t'}\|^2\}$ is a submartingale satisfying
    \begin{align}
        \mathbb{E}[(\xi^t w_\kappa^{t\top}\Sigma_\kappa^\perp Z^t)^2] \lesssim (1-\lambda_\kappa)^2, \quad
        \mathbb{E}[(\xi^t\|\Sigma_\kappa^\perp w^t_\kappa\|^2 w_\kappa^{t\top}Z^t)^2] \lesssim \|\Sigma_\kappa^\perp w^t_\kappa\|^4, \quad
        \mathbb{E}[\|\Sigma_\kappa^\perp Z^t\|^2] \lesssim (1-\lambda_\kappa)s,
    \end{align}
    Markov's and Doob's inequalities give, with probability $1 - \delta'$,
    \begin{align}
        \max_{0 \leq t \leq \tau'}\left|\sum_{t'=0}^{t-1}\xi^{\tau+t'} w^{(\tau+t')\top}_\kappa\Sigma_\kappa^\perp Z^{\tau+t'}\right|^2
        &\lesssim {\delta'}^{-1}(1-\lambda_\kappa)\tau', \\
        \max_{0 \leq t \leq \tau'}\sum_{t'=0}^{t-1}\|\Sigma_\kappa^\perp Z^{\tau+t'}\|^2
        &\lesssim {\delta'}^{-1}(1-\lambda_\kappa)s\tau'.
    \end{align}
    Furthermore,
    \begin{align}
        \sum_{t'=0}^{t-1}\xi^{\tau+t'}\|\Sigma_\kappa^\perp w^{\tau+t'}_\kappa\|^2 w^{(\tau+t')\top}_\kappa Z^{\tau+t'}
        \leq C_1(\|\Sigma_\kappa^\perp w^\tau_\kappa\|^2 + 7C_1\eta\tau')\sqrt{\tau'}.
    \end{align}
    Thus, for $t = 1, 2, \dots \tau'$,
    \begin{align}
        &2\eta \sum_{t'=0}^{t-1} \left( \xi^{\tau + t'} w^{(\tau + t')\top}_\kappa \Sigma^\perp_\kappa Z^{\tau + t'} - \xi^{\tau + t'} \|\Sigma^\perp_\kappa w^{\tau + t'}_\kappa \|^2 w^{(\tau + t')\top}_\kappa Z^{\tau + t'} + \eta \|\Sigma^\perp_\kappa Z^{\tau + t'} \|^2 \right) \\ &
        \lesssim (1-\lambda_\kappa)^\frac{1}{2} {\delta'}^{-\frac{1}{2}} \eta {\tau'}^\frac{1}{2} + (1-\lambda_\kappa) {\delta'}^{-1} \eta^2 s  \tau' + C_1 (\|\Sigma^\perp_\kappa w^{\tau}_\kappa\|^2 + 7C_1 \eta \tau') \eta {\tau'}^\frac{1}{2}.
    \end{align}
    Combining these estimates yields the bound for $\|\Sigma_\kappa^\perp w^{\tau+t}_\kappa\|^2$.
    The sequence $Q^t$ is then constructed so that $Q^t \geq C_1\eta$ and $Q^t \geq \|\Sigma_\kappa^\perp w^{\tau+t}_\kappa\|^2$ follows by induction.
\end{proof}

\subsection{Weak Alignment}

\begin{lem}\label{lem:strong-weakrecovery-1}
    Let $\eta^t = \eta \leq c_\eta s^{-3/2}$, $\epsStrongEff \leq c_\varepsilon s^{-1/2}$, and $\tau = \lfloor \frac{1}{2} s^{-1/2} C_1^{-1}\eta^{-1} \rfloor$.
    Suppose $\theta_\kappa^\top w^0_\kappa \geq s^{-1/2}$, $\|\Sigma_\kappa^\perp w^0_\kappa\|^2 \leq c_r \theta_\kappa^\top w^0_\kappa$, and $\theta_\kappa^\top w^t_\kappa \leq c_1$ for all $t = 0, 1, \dots, \tau$.
    Then, with probability $1 - \eta\tau\delta$,
    \begin{align}
        \theta_\kappa^\top w_\kappa^\tau \geq \left(1 + \tfrac{1}{2}\bar{\alpha}_2\beta_{\kappa,2}\eta\tau\right)\theta_\kappa^\top w_\kappa^0,
    \end{align}
    and furthermore $\|\Sigma_\kappa^\perp w^\tau_\kappa\|^2 \leq c_r \theta_\kappa^\top w^\tau_\kappa$.
\end{lem}

\begin{proof}
    By \cref{lem:strong-onestep-2}, if $\theta_\kappa^\top w^t_\kappa \geq \frac{1}{2}s^{-1/2}$ then $|\theta_\kappa^\top w^{t+1}_\kappa - \theta_\kappa^\top w^t_\kappa| \leq C_1\eta$ with high probability, so by induction $\theta_\kappa^\top w^t_\kappa \geq \frac{1}{2}s^{-1/2}$ for all $t = 0, 1, \dots, \tau$.
    Applying \cref{lem:strong-deviation-1} with $\tau \gets 0$, $\tau' \gets \tau$, $\delta' \gets \eta\tau\delta$, define
    \begin{align}
        Q^0 &= \|\Sigma_\kappa^\perp w^0_\kappa\|^2 \vee C_1\eta + 8C_1\eta\,\epsStrongEff\,\tau
        + C\!\left[(1-\lambda_\kappa)^{1/2}\delta^{-1/2}\eta^{1/2}
        + (1-\lambda_\kappa)\delta^{-1}\eta s
        + C_1(\|\Sigma_\kappa^\perp w^0_\kappa\|^2 + 7C_1\eta\tau)\eta\tau^{1/2}\right], \\
        Q^t &= Q^{t-1} + 5C_1\eta(Q^{t-1})^2.
    \end{align}
    With probability $1 - \eta\tau\delta$, $\|\Sigma_\kappa^\perp w^t_\kappa\|^2 \leq Q^t$ for all $t = 0, 1, \dots, \tau$.
    By the Bihari--LaSalle inequality, $Q^t \leq \frac{Q^0}{1 - 5C_1\eta Q^0 t}$, and in particular
    \begin{align}
        Q^\tau &\leq (1 + 6 C_1\eta \tau Q^0) Q^0 \\ &
        \leq (1 + 6 C_1\eta \tau (\|\Sigma^\perp_\kappa w^0_\kappa\|^2 \vee C_1 \eta + c_3 s^{-1})) (\|\Sigma^\perp_\kappa w^0_\kappa\|^2 \vee C_1 \eta + c_3 s^{-1}) \\ &
        \leq (1 + 7C_1 c_1 c_r \eta \tau) (\|\Sigma^\perp_\kappa w^0_\kappa\|^2 \vee C_1 \eta + c_3 s^{-1}) \\ &
        \leq (1 + \frac{1}{2} \alpha_2 \beta_{\kappa, 2} \eta \tau) (\|\Sigma^\perp_\kappa w^0_\kappa\|^2 \vee c_r s^{-\frac{1}{2}}) = (1 + \frac{1}{2} \alpha_2 \beta_{\kappa, 2} \eta \tau) c_r (\theta^\top_\kappa w^0_\kappa).
        \label{eq:strong-weakrecovery-1}
    \end{align}
    where the second inequality uses $\eta \leq c_\eta s^{-3/2}$ and $\epsStrongEff \leq c_\varepsilon s^{-1/2}$ to bound
    \begin{align}
        &8C_1 \eta  \epsStrongEff \tau + C \left[ (1-\lambda_\kappa)^\frac{1}{2} \delta^{-\frac{1}{2}} \eta^\frac{1}{2} + (1-\lambda_\kappa) \delta^{-1} \eta s + C_1 (\|\Sigma^\perp_\kappa w^{\tau}_\kappa\|^2 + 7C_1 \eta \tau) \eta {\tau}^\frac{1}{2} \right] \\ &
        \leq \left(4 c_\varepsilon + C \left[c_\varepsilon^\frac{1}{2} \delta^{-\frac{1}{2}} c_\eta^\frac{1}{2} + c_\varepsilon \delta^{-1} c_\eta + 4C_1^\frac{1}{2} c_\eta^{\frac{1}{2}}\right] \right) s^{-1} \\ &
        \leq c_3 s^{-1}
    \end{align}
    the third uses $\|\Sigma_\kappa^\perp w^0_\kappa\|^2 \vee C_1\eta + 5C_1 s^{-1} \leq \frac{7}{6}c_r(\theta_\kappa^\top w^0_\kappa) \leq \frac{7}{6}c_r c_1$, and the fourth uses $\eta\tau c_r s^{-1/2} = \frac{1}{2}C_1^{-1}c_r s^{-1} \gtrsim c_3 s^{-1}$.

    By \cref{lem:strong-onestep-2},
    \begin{align}
        \theta_\kappa^\top w_\kappa^{t+1}  &\geq \theta_\kappa^\top w_\kappa^t + \eta \lambda_\kappa \theta_\kappa^\top P_{w_\kappa^t}^\perp (Z^t + R^t) - \eta^2 C_1^2 (\theta_\kappa^\top w_\kappa^t) s \\ &
        + \eta \lambda_\kappa \sum_{i=1}^q \left[ i \alpha_i \beta_{\kappa,i} \| \Sigma_\kappa w_\kappa^t \|^{-(i-1)} (\theta_\kappa^\top w_\kappa^t)^{i-1} (1 - (\theta_\kappa^\top w_\kappa^t)^2) \right. \\ & 
        \qquad \left.+ \sqrt{(i+2)(i+1)} \alpha_i \beta_{\kappa,i+2} \| \Sigma_\kappa w_\kappa^t \|^{-(i+1)} \| \Sigma^\perp_\kappa w_\kappa^t \|^2 (\theta_\kappa^\top w_\kappa^t)^{i+1} \right] \\ &
        \geq \theta_\kappa^\top w_\kappa^t + 2\eta \lambda_\kappa \alpha_2 \beta_{\kappa, 2}  (\theta_\kappa^\top w_\kappa^t) (1 - c_1) + \eta \lambda_\kappa \theta_\kappa^\top P_{w_\kappa^t}^\perp (Z^t + R^t).
    \end{align}
    where we used $\eta \leq c_\eta s^{-3/2}$, $\|\Sigma_\kappa w^t_\kappa\|^2 = \Theta(1)$, and $\theta_\kappa^\top w^t_\kappa, \|\Sigma_\kappa^\perp w^t_\kappa\|^2 \leq c_1$ for the second inequality.
    Thus,
    \begin{align}
        \theta_\kappa^\top w_\kappa^{\tau} & \geq \theta^\top_\kappa w^0_\kappa + \eta \lambda_\kappa  \sum_{t=0}^{\tau - 1} \left[2(1 - c_1) \alpha_2 \beta_{\kappa, 2}  (\theta_\kappa^\top w_\kappa^t) + \theta_\kappa^\top P_{w_\kappa^t}^\perp (Z^t + R^t)\right] \\ &
        \geq \theta^\top_\kappa w^0_\kappa + 2(1 - c_1) \eta \lambda_\kappa \alpha_2 \beta_{\kappa, 2} \tau (\theta_\kappa^\top w_\kappa^0 - \frac{1}{2} C_1 \eta \tau ) + \eta \lambda_\kappa  \sum_{t=0}^{\tau - 1} \left[\theta_\kappa^\top P_{w_\kappa^t}^\perp (Z^t + R^t)\right] \\ &
        \geq \theta^\top_\kappa w^0_\kappa + \alpha_2 \beta_{\kappa, 2} \eta \tau (\theta_\kappa^\top w_\kappa^0) - C_1 \eta \left( \tau^\frac{1}{2} + (\epsStrongEff \vee Q^\tau )\tau  \right) \\ &
        \geq \theta^\top_\kappa w^0_\kappa + \alpha_2 \beta_{\kappa, 2} \eta \tau \theta_\kappa^\top w_\kappa^0 - \eta \tau (\sqrt{2} c_\eta^{\frac{1}{2}} C_1^\frac{3}{2} s^{-\frac{1}{2}} + C_1 c_\varepsilon s^{-\frac{1}{2}} + 2C_1 c_r (\theta^\top_\kappa w^0_\kappa)) \\ &
        \geq (1 + \frac{1}{2} \alpha_2 \beta_{\kappa, 2} \eta \tau) \theta_\kappa^\top w_\kappa^0.
        \label{eq:strong-weakrecovery-2}
    \end{align}
    holds with high probability.
    The second inequality uses $|\theta_\kappa^\top w^{t+1}_\kappa - \theta_\kappa^\top w^t_\kappa| \leq C_1\eta$ with high probability.
    The third uses $C_1\eta\tau \leq \frac{1}{2}\theta_\kappa^\top w^0_\kappa$ and $(1-c_1)\lambda_\kappa \geq \frac{2}{3}$.
    The fourth uses $Q^\tau \leq 2c_r(\theta_\kappa^\top w^0_\kappa)$, which follows from \eqref{eq:strong-weakrecovery-1}, together with $\epsStrongEff \leq c_\varepsilon s^{-1/2}$.
    From \eqref{eq:strong-weakrecovery-1} and \eqref{eq:strong-weakrecovery-2}, $\|\Sigma_\kappa^\perp w^\tau_\kappa\|^2 \leq Q^\tau \leq c_r \theta_\kappa^\top w^\tau_\kappa$.
\end{proof}

\begin{lem}\label{lem:strong-weakrecovery-2}
    Let $\eta^t = \eta \leq c_\eta s^{-3/2}$, $\epsStrongEff \leq c_\varepsilon s^{-1/2}$, $\theta_\kappa^\top w^0_\kappa \geq s^{-1/2}$, and $\|\Sigma_\kappa^\perp w^0_\kappa\|^2 \leq c_r s^{-1/2}$.
    With probability $1 - \tilde{\delta}$, there exists $t_1 \leq T_{1,1} = \Theta(\eta^{-1}\log s)$ such that $\theta_\kappa^\top w^{t_1}_\kappa > c_1$.
    Furthermore, with probability $1 - \tilde{\delta} - \delta$, $\|\Sigma_\kappa^\perp w^t_\kappa\|^2 \leq 3c_r s^{-1/2}$ for all $t = 0, 1, \dots, t_1 - 1$.
\end{lem}

\begin{proof}
    Suppose $\theta_\kappa^\top w^t_\kappa \leq c_1$ for all $t = 0, 1, \dots, T_{1,1}$.
    Set $\tau = \lfloor \frac{1}{2}s^{-1/2}C_1^{-1}\eta^{-1}\rfloor$ and $T_{1,1} = \lceil \frac{3\log s}{2\bar{\alpha}_2\beta_{\kappa,2}\eta\tau}\rceil\tau$.
    Applying \cref{lem:strong-weakrecovery-1} repeatedly with probability $1 - \frac{3\log s}{2\bar{\alpha}_2\beta_{\kappa,2}}\delta$,
    \begin{align}
        \log\frac{\theta_\kappa^\top w^{T_{1,1}}_\kappa}{\theta_\kappa^\top w^0_\kappa}
        &\geq \left\lceil\frac{3\log s}{2\bar{\alpha}_2\beta_{\kappa,2}\eta\tau}\right\rceil
        \log\!\left(1 + \tfrac{1}{2}\bar{\alpha}_2\beta_{\kappa,2}\eta\tau\right)
        \geq \frac{3\log s}{2\bar{\alpha}_2\beta_{\kappa,2}\eta\tau} \cdot \frac{\bar{\alpha}_2\beta_{\kappa,2}\eta\tau}{3}
        \geq \log s^{1/2}.
    \end{align}
    Hence $\theta_\kappa^\top w^{T_{1,1}}_\kappa \geq 1$, contradicting $\theta_\kappa^\top w^{T_{1,1}}_\kappa \leq c_1$.
    Let $t_1$ be the first time $\theta_\kappa^\top w^t_\kappa > c_1$.
    For $t = 0, 1, \dots, t_1 - 1$, applying \cref{lem:strong-deviation-1} gives $\|\Sigma_\kappa^\perp w^t_\kappa\|^2 \leq Q^t$ with probability $1 - \delta$, where $Q^0 \leq \|\Sigma_\kappa^\perp w^0_\kappa\|^2 + \frac{1}{2}c_3 s^{-1/2}$.
    By the Bihari--LaSalle inequality, $\|\Sigma_\kappa^\perp w^{t_1-1}_\kappa\|^2 \leq Q^{t_1-1} \leq 2Q^0 \leq 3c_r s^{-1/2}$.
\end{proof}

\subsection{Amplification of Alignment}

We reset the time index at $t_1$ (when weak alignment is achieved) and write $t \gets t - t_1$.

\begin{lem}\label{lem:strong-amplification-1}
    Let $\eta^t = \eta \leq c_\eta s^{-3/2}$, $\epsStrongEff \leq c_\varepsilon s^{-1/2}$, and $\tau = \lfloor \frac{1}{2}c_1 C_1^{-1}\eta^{-1}\rfloor$.
    Suppose $\theta_\kappa^\top w^0_\kappa \geq c_1$, $\|\Sigma_\kappa^\perp w^0_\kappa\|^2 \leq c_r\theta_\kappa^\top w^0_\kappa$, and $\theta_\kappa^\top w^t_\kappa \leq 1 - c_1$ for all $t = 0, 1, \dots, \tau$.
    Then, with probability $1 - \eta\tau\delta$,
    \begin{align}
        \theta_\kappa^\top w^\tau_\kappa \geq \left(1 + \tfrac{1}{2}c_1\bar{\alpha}_2\beta_{\kappa,2}\eta\tau\right)\theta_\kappa^\top w^0_\kappa,
    \end{align}
    and $\|\Sigma_\kappa^\perp w^\tau_\kappa\|^2 \leq c_r\theta_\kappa^\top w^\tau_\kappa$.
\end{lem}

\begin{proof}
    The argument parallels \cref{lem:strong-weakrecovery-1}.
    For $t = 0, 1, \dots, \tau$, we have $\theta_\kappa^\top w^t_\kappa \geq \frac{1}{2}c_1 \geq \frac{1}{2}s^{-1/2}$.
    Applying \cref{lem:strong-deviation-1} with $\tau \gets 0$, $\tau' \gets \tau$, $\delta' \gets \eta\tau\delta$ defines $Q^0$ and $Q^t$; the dominant term of $Q^0$ is $\|\Sigma_\kappa^\perp w^0_\kappa\|^2 \vee C_1\eta + 8C_1\eta\,\epsStrongEff\,\tau$, giving $Q^0 \leq \|\Sigma_\kappa^\perp w^0_\kappa\|^2 \vee C_1\eta + s^{-1/2}$.
    By the Bihari--LaSalle inequality \cref{lem:technical-bihari},
    \begin{align}
        Q^\tau
        &\leq (1 + 6C_1\eta\tau Q^0)Q^0 \\
        &\leq (1 + 7c_r C_1\eta\tau)(\|\Sigma_\kappa^\perp w^0_\kappa\|^2 \vee C_1\eta + s^{-1/2}) \\
        &\leq (1 + \tfrac{1}{2}c_1\bar{\alpha}_2\beta_{\kappa,2}\eta\tau)(\|\Sigma_\kappa^\perp w^0_\kappa\|^2 \vee c_r c_1)
        \leq (1 + \tfrac{1}{2}c_1\bar{\alpha}_2\beta_{\kappa,2}\eta\tau) c_r(\theta_\kappa^\top w^0_\kappa).
    \end{align}
    By \cref{lem:strong-onestep-2},
    \begin{align}
        \theta_\kappa^\top w_\kappa^{t+1}  &\geq \theta_\kappa^\top w_\kappa^t + \eta \lambda_\kappa \theta_\kappa^\top P_{w_\kappa^t}^\perp (Z^t + R^t) - \eta^2 C_1^2 (\theta_\kappa^\top w_\kappa^t) s \\ &
        + \eta \lambda_\kappa \sum_{i=1}^q \left[ i \alpha_i \beta_{\kappa,i} \| \Sigma_\kappa w_\kappa^t \|^{-(i-1)} (\theta_\kappa^\top w_\kappa^t)^{i-1} (1 - (\theta_\kappa^\top w_\kappa^t)^2) \right. \\ & 
        \qquad \left.+ \sqrt{(i+2)(i+1)} \alpha_i \beta_{\kappa,i+2} \| \Sigma_\kappa w_\kappa^t \|^{-(i+1)} \| \Sigma^\perp_\kappa w_\kappa^t \|^2 (\theta_\kappa^\top w_\kappa^t)^{i+1} \right] \\ &
        \geq \theta_\kappa^\top w_\kappa^t + c_1 \eta \alpha_2 \beta_{\kappa, 2}  (\theta_\kappa^\top w_\kappa^t) + \eta \lambda_\kappa \theta_\kappa^\top P_{w_\kappa^t}^\perp (Z^t + R^t).
    \end{align}
    where we used $1 - (\theta_\kappa^\top w^t_\kappa)^2 \geq 2c_1 - c_1^2$, $\eta \leq c_\eta s^{-3/2}$, and $\|\Sigma_\kappa^\perp w^t_\kappa\|^2 \leq Q^\tau \leq 3c_r\theta_\kappa^\top w^t_\kappa$ for the second inequality.
    Then we have
    \begin{align}
        \theta_\kappa^\top w_\kappa^{\tau} & \geq \theta^\top_\kappa w^0_\kappa + c_1 \alpha_2 \beta_{\kappa, 2} \eta \tau (\theta_\kappa^\top w_\kappa^0 - \frac{1}{2} C_1 \eta \tau) + \eta \lambda_\kappa  \sum_{t=0}^{\tau - 1} \left[\theta_\kappa^\top P_{w_\kappa^t}^\perp (Z^t + R^t)\right] \\ &
        \geq \theta^\top_\kappa w^0_\kappa + c_1 \alpha_2 \beta_{\kappa, 2} \eta \tau (\theta_\kappa^\top w_\kappa^0 - \frac{1}{2} C_1 \eta \tau) - C_1 \eta \left( \tau^\frac{1}{2} + (\epsStrongEff \vee Q^\tau )\tau  \right) \\ &
        \geq (1 + \frac{1}{2} c_1 \alpha_2 \beta_{\kappa, 2} \eta \tau) \theta_\kappa^\top w_\kappa^0.
    \end{align}
    with high probability.
    Hence $\|\Sigma_\kappa^\perp w^\tau_\kappa\|^2 \leq Q^\tau \leq c_r\theta_\kappa^\top w^\tau_\kappa$.
\end{proof}

\begin{lem}\label{lem:strong-amplification-2}
    Let $\eta^t = \eta \leq c_\eta s^{-3/2}$, $\epsStrongEff \leq c_\varepsilon s^{-1/2}$, $\theta_\kappa^\top w^0_\kappa \geq c_1$, and $\|\Sigma_\kappa^\perp w^0_\kappa\|^2 \leq c_r c_1$.
    With probability $1 - \tilde{\delta}$, there exists $t_2 \leq T_{1,2} = \tilde{\Theta}(\eta^{-1})$ such that $\theta_\kappa^\top w^{t_2}_\kappa > 1 - c_1$ and $\|\Sigma_\kappa^\perp w^{t_2}_\kappa\|^2 \leq 2c_r$.
\end{lem}

\begin{proof}
    Suppose $\theta_\kappa^\top w^t_\kappa \leq 1 - c_1$ for all $t = 0, 1, \dots, T_{1,2}$.
    Set $\tau = \lfloor \frac{1}{2}c_1 C_1^{-1}\eta^{-1}\rfloor$ and $T_{1,2} = \lceil \frac{3\log s}{2c_1\bar{\alpha}_2\beta_{\kappa,2}\eta\tau}\rceil\tau$.
    Applying \cref{lem:strong-amplification-1} repeatedly with probability $1 - \frac{3\log s}{2\bar{\alpha}_2\beta_{\kappa,2}}c_1^{-1}\delta$,
    \begin{align}
        \log\frac{\theta_\kappa^\top w^{T_{1,2}}_\kappa}{\theta_\kappa^\top w^0_\kappa}
        \geq \left\lceil\frac{3\log s}{2c_1\bar{\alpha}_2\beta_{\kappa,2}\eta\tau}\right\rceil
        \log\!\left(1 + \tfrac{1}{2}c_1\bar{\alpha}_2\beta_{\kappa,2}\eta\tau\right)
        \geq \log s^{1/2},
    \end{align}
    so $\theta_\kappa^\top w^{T_{1,2}}_\kappa \geq 1$, contradicting $\theta_\kappa^\top w^{T_{1,2}}_\kappa \leq 1 - c_1$.
    Let $t_2$ be the first time $\theta_\kappa^\top w^t_\kappa > 1 - c_1$, and set $t_2' = \lfloor t_2/\tau\rfloor\tau$.
    At time $t_2'$, $\|\Sigma_\kappa^\perp w^{t_2'}_\kappa\|^2 \leq c_r\theta_\kappa^\top w^{t_2'}_\kappa \leq c_r$.
    By the same argument as in \cref{lem:strong-amplification-1},
    $\|\Sigma_\kappa^\perp w^{t_2}_\kappa\|^2 \leq (1 + \frac{1}{2}c_1\bar{\alpha}_2\beta_{\kappa,2}\eta\tau) c_r(\theta_\kappa^\top w^{t_2'}_\kappa) \leq 2c_r$.
\end{proof}

\subsection{Strong Alignment}\label{strong-strongalignment}

\begin{lem}\label{lem:strong-strongrecovery-1}
    Let $\bar{\varepsilon} \geq \epsStrongEff$, $\eta = \eta^t \leq c_\eta\bar{\varepsilon} s^{-1} \wedge c_\eta\bar{\varepsilon}^2$, and $\|\Sigma_\kappa^\perp w^0_\kappa\|^2 \leq 3c_r$.
    Let $\tau > 0$ and suppose $\theta_\kappa^\top w^t_\kappa \geq \frac{1}{2}$ and $\|\Sigma_\kappa^\perp w^t_\kappa\|^2 \geq \bar{\varepsilon}$ for all $t = 0, 1, \dots, \tau - 1$.
    Then for every $t = 0, 1, \dots, \tau$,
    \begin{align}
        \|\Sigma_\kappa^\perp w_\kappa^t\|^2 \leq \left(1 - \tfrac{1}{5}c_w\eta\right)^t \|\Sigma_\kappa^\perp w^0_\kappa\|^2 + c_r\bar{\varepsilon}
    \end{align}
    with high probability.
\end{lem}

\begin{proof}
    Setting $\xi^t$ as in the proof of \cref{lem:strong-deviation-1}, for $\theta_\kappa^\top w^t_\kappa \geq 0$ we have by the same computation as in \eqref{eq:strong-deviation-2},
    \begin{align}
        \|\Sigma_\kappa^\perp w^{t+1}_\kappa\|^2
        &\leq (\xi^t)^2\|\Sigma_\kappa^\perp w^t_\kappa\|^2 + 2C_1\eta\xi^t\|\Sigma_\kappa^\perp w^t_\kappa\|^4
        + 3C_1^2\eta^2\|\Sigma_\kappa^\perp w^t_\kappa\|^2 \\
        &\quad + 4C_1\eta\xi^t\,\epsStrongEff\|\Sigma_\kappa^\perp w^t_\kappa\|
        + 4C_1^2\eta^2\epsStrongEff^2 s \\
        &\quad + 2\!\left(\eta\xi^t w_\kappa^{t\top}\Sigma_\kappa^\perp Z^t
        - \eta\xi^t\|\Sigma_\kappa^\perp w^t_\kappa\|^2 w_\kappa^{t\top}Z^t
        + \eta^2\|\Sigma_\kappa^\perp Z^t\|^2\right).
    \end{align}
    Fix $t \leq \tau$ and assume $\|\Sigma_\kappa^\perp w^{t'}_\kappa\|^2 \leq 4c_r$ for all $t' \leq t-1$.
    Using $(\xi^{t'})^2 \leq \xi^{t'}$, $\theta_\kappa^\top w^{t'}_\kappa \geq \frac{1}{2}$, and $\|\Sigma_\kappa^\perp w^{t'}_\kappa\| \geq \bar{\varepsilon}$, and bounding $4C_1\eta\xi^{t'}\epsStrongEff\|\Sigma_\kappa^\perp w^{t'}_\kappa\| \leq 8c_r C_1\bar{\varepsilon}\eta$, $4C_1^2\eta^2\epsStrongEff^2 s \leq 4C_1^2 c_\eta\bar{\varepsilon}^3\eta$, and $\eta\|\Sigma_\kappa^\perp Z^{t'}\|^2 \leq c_\eta C_1\bar{\varepsilon}$,
    \begin{align}
        \|\Sigma^\perp_\kappa w^{t'+1}_\kappa\|^2 & \leq \|\Sigma^\perp_\kappa w^{t'}_\kappa\|^2 - \frac{1}{4} \lambda_\kappa c_w \eta\|\Sigma^\perp_\kappa w^{t'}_\kappa\|^2 + 2C_1 \eta \xi^{t'} \|\Sigma^\perp_\kappa w^{t'}_\kappa\|^4 + 3C_1^2 \eta^2 \| \Sigma_\kappa^\perp w^{t'}_\kappa\|^2 \\ &
        \quad + 4C_1 \eta \xi^{t'} \epsStrongEff \|\Sigma^\perp_\kappa w^{t'}_\kappa\| + 4C_1^2 \eta^2 \epsStrongEff^2 s \\ &
        \quad + 2 \left(\eta \xi^{t'} w^{t'\top}_\kappa \Sigma^\perp_\kappa Z^{t'}- \eta \xi^{t'} \|\Sigma^\perp_\kappa w^{t'}_\kappa\|^2 w^{t'\top}_\kappa Z^{t'} + \eta^2\|\Sigma^\perp_\kappa Z^{t'}\|^2 \right) \\ &
        \leq (1 - \frac{1}{5} c_w \eta) \|\Sigma^\perp_\kappa w^{t'}_\kappa\|^2+ 2\left(\eta \xi^{t'} w^{t'\top}_\kappa \Sigma^\perp_\kappa Z^{t'}- \eta \xi^{t'} \|\Sigma^\perp_\kappa w^{t'}_\kappa\|^2 w^{t'\top}_\kappa Z^{t'}\right).
    \end{align}
    Unrolling this recursion,
    \begin{align}
        \|\Sigma_\kappa^\perp w_\kappa^t\|^2
        \leq (1 - \tfrac{1}{5}c_w\eta)^t\|\Sigma_\kappa^\perp w^0_\kappa\|^2
        + 2\eta\sum_{t'=0}^{t-1}(1 - \tfrac{1}{5}c_w\eta)^{t-t'-1}
        \!\left(\xi^{t'} w_\kappa^{t'\top}\Sigma_\kappa^\perp Z^{t'} - \xi^{t'}\|\Sigma_\kappa^\perp w^{t'}_\kappa\|^2 w_\kappa^{t'\top}Z^{t'}\right).
    \end{align}
    The conditional variance of the martingale increment at step $t'$ is at most $(1 - \frac{1}{5}c_w\eta)^{2(t-t'-1)}C_1^2$. Summing over $t' = 0, \dots, t-1$ and bounding the resulting geometric series gives
    \begin{align}
        \sum_{t'=0}^{t-1} \left(1 - \tfrac{1}{5} c_w \eta\right)^{2(t-t'-1)} C_1^2 \leq \frac{C_1^2}{1 - \left(1 - \frac{1}{5} c_w \eta\right)^2} \leq \frac{5C_1^2}{c_w \eta}.
    \end{align}
    Hence each weighted sum is at most $\sqrt{5C_1^2/(c_w\eta)}$ with high probability, and since $\eta \leq c_\eta\bar{\varepsilon}^2$,
    \begin{align}
        \|\Sigma_\kappa^\perp w_\kappa^t\|^2
        \leq (1 - \tfrac{1}{5}c_w\eta)^t\|\Sigma_\kappa^\perp w^0_\kappa\|^2 + 4\sqrt{5}C_1 c_w^{-1/2}\sqrt{\eta}
        \leq (1 - \tfrac{1}{5}c_w\eta)^t\|\Sigma_\kappa^\perp w^0_\kappa\|^2 + c_r\bar{\varepsilon}.
    \end{align}
    This gives $\|\Sigma_\kappa^\perp w^t_\kappa\|^2 \leq 4c_r$, completing the induction.
\end{proof}

\begin{lem}\label{lem:strong-strongrecovery-2}
    Let $\eta \leq c_\eta s^{-1}$, $\epsStrongEff \leq c_\varepsilon s^{-1/2}$, $\theta_\kappa^\top w^0_\kappa \geq 1 - 3c_1$, and suppose $\theta_\kappa^\top w^t_\kappa \leq 1 - c_2$ and $\|\Sigma_\kappa^\perp w^t_\kappa\|^2 \leq 4c_r$ for all $t = 0, 1, \dots, \tau$.
    Then, with high probability,
    \begin{align}
        \theta_\kappa^\top w^t_\kappa \geq \theta_\kappa^\top w^0_\kappa - C_2\eta + \bar{\alpha}_2\beta_{\kappa,2} c_2\eta t
    \end{align}
\end{lem}

\begin{proof}
    Fix $t \leq \tau$ and assume $\theta_\kappa^\top w^{t'}_\kappa \geq 1 - 4c_1$ for all $t' \leq t - 1$.
    By \cref{lem:strong-onestep-2}, using $\eta \leq c_\eta s^{-1}$, $1 - (\theta_\kappa^\top w^{t'}_\kappa)^2 \geq c_2$, $\|\Sigma_\kappa^\perp w^{t'}_\kappa\|^2 \leq 4c_r$, and $c_r \lesssim c_2$,
    \begin{align}
        \theta_\kappa^\top w_\kappa^{t'+1}  &\geq \theta_\kappa^\top w_\kappa^{t'} + \eta \lambda_\kappa \theta_\kappa^\top P_{w_\kappa^{t'}}^\perp (Z^{t'} + R^{t'}) - \eta^2 C_1^2 (\theta_\kappa^\top w_\kappa^{t'}) s \\ &
        + \eta \lambda_\kappa \sum_{i=1}^q \left[ i \alpha_i \beta_{\kappa,i} \| \Sigma_\kappa w_\kappa^{t'} \|^{-(i-1)} (\theta_\kappa^\top w_\kappa^{t'})^{i-1} (1 - (\theta_\kappa^\top w_\kappa^{t'})^2) \right. \\ & 
        \qquad \left.+ \sqrt{(i+2)(i+1)} \alpha_i \beta_{\kappa,i+2} \| \Sigma_\kappa w_\kappa^{t'} \|^{-(i+1)} \| \Sigma^\perp_\kappa w_\kappa^{t'} \|^2 (\theta_\kappa^\top w_\kappa^{t'})^{i+1} \right] \\ &
        \geq \theta^\top_\kappa w^{t'}_\kappa + \frac{3}{2} \eta \alpha_2 \beta_{\kappa,2} c_2 + \eta \lambda_\kappa \theta_\kappa^\top P_{w_\kappa^{t'}}^\perp (Z^{t'} + R^{t'}).
    \end{align}
    Using $\sum_{t'=0}^{t-1}\eta \lambda_\kappa \theta_\kappa^\top P_{w_\kappa^{t'}}^\perp R^{t'} \leq \eta C_1 (\epsStrongEff \vee 2c_r) \leq \frac{1}{4}\eta \alpha_2 \beta_{\kappa, 2} c_2$ and
    \begin{align}
        \sum_{t'=0}^{t-1}\eta\lambda_\kappa\theta_\kappa^\top P_{w_\kappa^{t'}}^\perp Z^{t'}
        \leq \eta C_1\sqrt{t}
        \leq \begin{cases}
            C_2\eta & (t \leq C_2) \\
            \frac{1}{4}\bar{\alpha}_2\beta_{\kappa,2} c_2\eta t & (t > C_2)
        \end{cases}
        \quad \text{with high probability},
    \end{align}
    we obtain $\theta_\kappa^\top w^t_\kappa \geq \theta_\kappa^\top w^0_\kappa - C_2\eta + \bar{\alpha}_2\beta_{\kappa,2} c_2\eta t$.
    In particular $\theta_\kappa^\top w^t_\kappa \geq 1 - 4c_1$, completing the induction.
\end{proof}

\begin{lem}\label{lem:strong-strongrecovery-3}
    Let $\epsStrongEff \leq \bar{\varepsilon} \leq c_r$, $\eta \leq c_\eta\bar{\varepsilon} s^{-1} \wedge c_\eta\bar{\varepsilon}^2$, $\theta_\kappa^\top w^0_\kappa \geq 1 - 3c_1$.
    Let $\tau > 0$ and suppose $\|\Sigma_\kappa^\perp w^t_\kappa\| \leq 2\bar{\varepsilon}$ and $\theta_\kappa^\top w^t_\kappa \leq 1 - C_2\bar{\varepsilon}$ for all $t = 0, 1, \dots, \tau - 1$.
    Then
    \begin{align}
        \theta_\kappa^\top w^t_\kappa
        \geq 1 - (1 - \eta\bar{\alpha}_2\beta_{\kappa,2})^t(1 - \theta_\kappa^\top w^0_\kappa) - c_3\bar{\varepsilon}.
    \end{align}
\end{lem}

\begin{proof}
    Fix $t \leq \tau$ and assume $\theta_\kappa^\top w^{t'}_\kappa \geq 1 - 4c_1$ for all $t' \leq t - 1$.
    By the same argument as in \cref{lem:strong-strongrecovery-2},
    \begin{align}
        \theta_\kappa^\top w_\kappa^{t'+1}  &\geq \theta_\kappa^\top w_\kappa^{t'} + \eta \lambda_\kappa \theta_\kappa^\top P_{w_\kappa^{t'}}^\perp (Z^{t'} + R^{t'}) - \eta^2 C_1^2 (\theta_\kappa^\top w_\kappa^{t'}) s \\ &
        + \eta \lambda_\kappa \sum_{i=1}^q \left[ i \alpha_i \beta_{\kappa,i} \| \Sigma_\kappa w_\kappa^{t'} \|^{-(i-1)} (\theta_\kappa^\top w_\kappa^{t'})^{i-1} (1 - (\theta_\kappa^\top w_\kappa^{t'})^2) \right. \\ & 
        \qquad \left.+ \sqrt{(i+2)(i+1)} \alpha_i \beta_{\kappa,i+2} \| \Sigma_\kappa w_\kappa^{t'} \|^{-(i+1)} \| \Sigma^\perp_\kappa w_\kappa^{t'} \|^2 (\theta_\kappa^\top w_\kappa^{t'})^{i+1} \right] \\ &
        \geq \theta^\top_\kappa w^{t'}_\kappa + \frac{3}{2} \eta \alpha_2 \beta_{\kappa,2} (1 - \theta^\top_\kappa w^{t'}_\kappa) + \eta \lambda_\kappa \theta_\kappa^\top P_{w_\kappa^{t'}}^\perp (Z^{t'} + R^{t'}).
    \end{align}
    Since $\eta\lambda_\kappa\theta_\kappa^\top P_{w_\kappa^{t'}}^\perp R^{t'} \leq \eta C_1(\epsStrongEff \vee 2\bar{\varepsilon}) \leq \frac{1}{2}\eta\bar{\alpha}_2\beta_{\kappa,2}(1 - \theta_\kappa^\top w^{t'}_\kappa)$,
    \begin{align}
        1 - \theta^\top_\kappa w^{t}_\kappa &\leq (1 - \eta \alpha_2 \beta_{\kappa, 2})^t (1 - \theta^\top_\kappa w^0_\kappa) + \eta \lambda_\kappa \sum_{t'=0}^{t-1} (1 - \eta \alpha_2 \beta_{\kappa, 2})^{t - t' - 1} \theta^\top_\kappa P^\perp_{w^{t'}_\kappa} Z^{t'} \\ &
        \leq (1 - \eta \alpha_2 \beta_{\kappa, 2})^t (1 - \theta^\top_\kappa w^0_\kappa) + \frac{\sqrt{\eta} C_1 \lambda_\kappa}{\sqrt{\alpha_2 \beta_{\kappa, 2}}}.
    \end{align}
    where the weighted martingale sum is bounded using the conditional variance of the martingale increments.
    Since $\eta \leq c_\eta\bar{\varepsilon}^2$, the second term is at most $c_3\bar{\varepsilon}$, giving the stated bound.
    In particular $\theta_\kappa^\top w^t_\kappa \geq 1 - 4c_1$, completing the induction.
\end{proof}

\begin{lem}\label{lem:strong-strongrecovery-4}
    Let $\epsStrongEff \leq c_\varepsilon s^{-1/2}$ and $C_2\,\epsStrongEff \leq \tilde{\varepsilon}$. Set
    \begin{align}
        \eta^t = \begin{cases}
            \eta_1 \leq c_\eta s^{-3/2} & (0 \leq t \leq \Delta T_1 + \Delta T_2 - 1) \\
            \eta_2 \leq c_\eta C_2^{-1}\tilde{\varepsilon} s^{-1} \wedge c_\eta C_2^{-2}\tilde{\varepsilon}^2
            & (\Delta T_1 + \Delta T_2 \leq t \leq \Delta T_1 + \Delta T_2 + T_{1,3} + T_{1,4} - 1),
        \end{cases}
    \end{align}
    where $\Delta T_1 = T_{1,1} - t_1$, $\Delta T_2 = T_{1,2} - t_2$, $T_{1,3} = \tilde{\Theta}(\eta_2^{-1}\log\tilde{\varepsilon}^{-1})$, and $T_{1,4} = \tilde{\Theta}(\eta_2^{-1}\log\tilde{\varepsilon}^{-1})$.
    Then for every neuron with $\theta_\kappa^\top w^0_\kappa \geq 1 - c_1$ and $\|\Sigma_\kappa^\perp w^0_\kappa\|^2 \leq 2c_r$,
    \[
        \theta_\kappa^\top w^{\Delta T_1 + \Delta T_2 + T_{1,3} + T_{1,4}}_\kappa > 1 - 2\tilde{\varepsilon}
    \]
    with high probability.
\end{lem}

\begin{proof}
    \textit{Phase 1: $0 \leq t \leq \Delta T_1 + \Delta T_2$.}
    Set $\tau = \lceil 5c_w^{-1}\eta_1^{-1}\rceil$ and $\bar{\varepsilon} = s^{-1/2}$, so $\eta_1 \leq c_\eta\bar{\varepsilon} s^{-1}$.

    We first control $\|\Sigma_\kappa^\perp w^t_\kappa\|^2$.
    By \cref{lem:strong-onestep-2}, $\theta_\kappa^\top w^t_\kappa \geq 1 - C_1\eta_1\tau \geq \frac{1}{2}$ for all $t \leq \tau$, so \cref{lem:strong-strongrecovery-1} is applicable throughout $[0, \tau]$.
    If $\|\Sigma_\kappa^\perp w^t_\kappa\|^2 \geq \bar{\varepsilon}$ throughout $[0,\tau]$, then \cref{lem:strong-strongrecovery-1} gives $\|\Sigma_\kappa^\perp w^\tau_\kappa\|^2 \leq e^{-c_w\eta_1\tau/5}\|\Sigma_\kappa^\perp w^0_\kappa\|^2 + c_r\bar{\varepsilon} \leq 2c_r$, and $\|\Sigma_\kappa^\perp w^t_\kappa\|^2 \leq 3c_r$ for all $t \leq \tau$.
    If instead $\|\Sigma_\kappa^\perp w^{t'}_\kappa\|^2 < \bar{\varepsilon}$ for some $t' \leq \tau$, then by \cref{lem:strong-deviation-1} the quantity increases by at most $7C_1\eta_1$ per step, so $\|\Sigma_\kappa^\perp w^{t'+1}_\kappa\|^2 \leq \bar{\varepsilon} + 7C_1\eta_1 \leq c_r$, and \cref{lem:strong-strongrecovery-1} applies again from $t'+1$.
    Repeating this argument at each such $t'$ yields $\|\Sigma_\kappa^\perp w^t_\kappa\|^2 \leq 3c_r$ for all $t \leq \Delta T_1 + \Delta T_2$.

    We then control $\theta_\kappa^\top w^t_\kappa$ via \cref{lem:strong-strongrecovery-2}.
    If $\theta_\kappa^\top w^t_\kappa \leq 1 - c_2$ throughout $[0,\tau]$, then $\theta_\kappa^\top w^\tau_\kappa \geq \theta_\kappa^\top w^0_\kappa - C_2\eta_1 + \bar{\alpha}_2\beta_{\kappa,2}c_2\eta_1\tau \geq 1 - c_1$, and $\theta_\kappa^\top w^t_\kappa \geq 1 - 2c_1$ for all $t \leq \tau$.
    If instead $\theta_\kappa^\top w^{t'}_\kappa > 1 - c_2$ for some $t' \leq \tau$, then by \cref{lem:strong-onestep-2} the alignment decreases by at most $C_2\eta_1$ per step, so $\theta_\kappa^\top w^{t'+1}_\kappa \geq 1 - c_2 - C_2\eta_1 \geq 1 - 2c_1$, and the same argument applies from $t'+1$.
    Repeating at each such $t'$ gives $\theta_\kappa^\top w^t_\kappa \geq 1 - 2c_1$ for all $t \leq \Delta T_1 + \Delta T_2$.

    \textit{Phase 2: $\Delta T_1 + \Delta T_2 \leq t \leq \Delta T_1 + \Delta T_2 + T_{1,3} + T_{1,4}$.}
    Set $\bar{\varepsilon} = C_2^{-1}\tilde{\varepsilon}$ and define
    \begin{align}
        T_{1,3} = \left\lceil 5c_w^{-1}\eta_2^{-1}\log\bar{\varepsilon}^{-1}\right\rceil,
        \qquad
        T_{1,4} = \left\lceil \eta_2^{-1}\bar{\alpha}_2^{-1}\beta_{\kappa,2}^{-1}\log\bar{\varepsilon}^{-1}\right\rceil.
    \end{align}
    By the same argument as Phase~1, $\theta_\kappa^\top w^t_\kappa \geq 1 - 3c_1$ and $\|\Sigma_\kappa^\perp w^t_\kappa\|^2 \leq 4c_r$ throughout this phase.
    If $\|\Sigma_\kappa^\perp w^t_\kappa\|^2 \geq \bar{\varepsilon}$ for all $t \leq T$, then \cref{lem:strong-strongrecovery-1} gives
    \begin{align}\label{eq:strong-strongrecovery-1}
        \|\Sigma_\kappa^\perp w^{\Delta T_1 + \Delta T_2 + T_{1,3}}_\kappa\|^2
        \leq (1 - \tfrac{1}{5}c_w\eta_2)^{T_{1,3}} \cdot 4c_r + c_r\bar{\varepsilon}
        \leq 4c_r\bar{\varepsilon} < \bar{\varepsilon},
    \end{align}
    a contradiction. Hence there exists $t_3 \leq T_{1,3}$ with $\|\Sigma_\kappa^\perp w^{\Delta T_1 + \Delta T_2 + t_3}_\kappa\|^2 < \bar{\varepsilon}$.
    Applying \cref{lem:strong-strongrecovery-1} repeatedly thereafter keeps $\|\Sigma_\kappa^\perp w^t_\kappa\|^2 \leq 2\bar{\varepsilon}$.
    If $\theta_\kappa^\top w^{\Delta T_1 + \Delta T_2 + t}_\kappa \leq 1 - C_2\bar{\varepsilon}$ for all $t$ in $[t_3, T_{1,3} + T_{1,4}]$, then \cref{lem:strong-strongrecovery-3} gives
    \begin{align}\label{eq:strong-strongrecovery-2}
        \theta_\kappa^\top w^{\Delta T_1 + \Delta T_2 + T_{1,3} + T_{1,4}}_\kappa
        \geq 1 - (1 - \eta_2\bar{\alpha}_2\beta_{\kappa,2})^{T_{1,4}-t_3}(1 - \theta_\kappa^\top w^{\Delta T_1+\Delta T_2+t_3}_\kappa) - c_3\bar{\varepsilon}
        \geq 1 - 3c_1\bar{\varepsilon} - c_3\bar{\varepsilon} > 1 - C_2\bar{\varepsilon},
    \end{align}
    a contradiction. Hence there exists $t_4 \leq T_{1,4}$ with $\theta_\kappa^\top w^{\Delta T_1+\Delta T_2+T_{1,3}-t_3+t_4}_\kappa > 1 - C_2\bar{\varepsilon}$.
    Applying \cref{lem:strong-strongrecovery-3} for the remaining steps gives
    $\theta_\kappa^\top w^{\Delta T_1 + \Delta T_2 + T_{1,3} + T_{1,4}}_\kappa \geq 1 - C_2\bar{\varepsilon} - C_1\eta_2 \geq 1 - 2\tilde{\varepsilon}$.
\end{proof}

\subsection{Non-Forgetting of Other Tasks}

\begin{lem}\label{lem:strong-unforgetting-1}
    Let $\eta = \eta^t$ with $\chi_k C_1\eta s^{1/2} \leq 1$.
    Then for every $t' = 1, 2, \dots$,
    \begin{align}
        \|\Sigma_\kappa w^{t'}_k\|^2
        \leq \left(\|\Sigma_\kappa w^0_k\|^2 + 2\chi_k C_1\eta\!\left(\sqrt{t'} + \epsStrongEff\, t'\right)
        + 2(\chi_k C_1\eta)^2 s\, t'\right)\exp(2\chi_k C_1\eta\, t').
    \end{align}
\end{lem}

\begin{proof}
    Set $g = g^t_k$ using the notation of \cref{lem:strong-onestep-1}.
    Since $\eta\|P^\perp_{w^t_k} g\| \leq \chi_k C_1\eta s^{1/2} \leq 1$ with high probability,
    \begin{align}
        \|\Sigma_\kappa w^{t+1}_k\|^2 &= \frac{(w^t_k + \eta P^\perp_{w^t_k} g)^\top \Sigma_\kappa (w^t_k + \eta P^\perp_{w^t_k} g)}{\|w^t_k + \eta P^\perp_{w^t_k} g\|^2} \\ &
        \leq (w^t_k + \eta P^\perp_{w^t_k} g)^\top \Sigma_\kappa (w^t_k + \eta P^\perp_{w^t_k} g) (1 - \eta^2 \|P^\perp_{w^t_k} g\|^2) \\ &
        = (\|\Sigma_\kappa w^t_k\|^2 + 2\eta w^t_k \Sigma_\kappa P^\perp_{w^t_k} g + \eta^2 \|\Sigma_\kappa P^\perp_{w^t_k} g\|^2) (1 - \eta^2 \|P^\perp_{w^t_k} g\|^2) \\ &
        \leq \|\Sigma_\kappa w^t_k\|^2 + 2\eta w^t_k \Sigma_\kappa P^\perp_{w^t_k} g + \eta^2 \|\Sigma_\kappa P^\perp_{w^t_k} g\|^2 + \eta^2 \|P^\perp_{w^t_k} g\|^2 |2 \eta w^t_k \Sigma_\kappa P^\perp_{w^t_k} g| \\ & 
        \leq \|\Sigma_\kappa w^t_k\|^2 + 2\eta w^t_k \Sigma_\kappa P^\perp_{w^t_k} g + 2(\chi_k C_1 \eta)^2 s.
    \end{align}
    By \cref{lem:strong-onestep-1}~(ii),
    \begin{align}
        2\eta w^{t\top}_k \Sigma_\kappa P^\perp_{w^t_k} g
        &= 2\chi_k\eta w^{t\top}_k \Sigma_\kappa P^\perp_{w^t_k}(Z^t + R^t) \\
        &\leq 2\chi_k\eta\!\left[w^{t\top}_k \Sigma_\kappa P^\perp_{w^t_k} Z^t
        + C_1\,\epsStrongEff + C_1\|\Sigma_\kappa w^t_k\|^2\right].
    \end{align}
    Summing over $t = 0, \dots, t'-1$ and bounding the martingale term $\sum_{t=0}^{t'-1} w^{t\top}_k \Sigma_\kappa P^\perp_{w^t_k} Z^t \leq C_1\sqrt{t'}$ with high probability,
    \begin{align}
        \|\Sigma_\kappa w^{t'}_k\|^2
        \leq \|\Sigma_\kappa w^0_k\|^2
        + 2\chi_k C_1\eta\!\left(\sum_{t=0}^{t'-1}\|\Sigma_\kappa w^t_k\|^2 + \sqrt{t'} + \epsStrongEff\, t'\right)
        + 2(\chi_k C_1\eta)^2 s\, t'.
    \end{align}
    The stated bound follows by Gr\"{o}nwall's inequality.
\end{proof}

\begin{lem}\label{lem:strong-unforgetting-2}
    Assume the hypotheses of \cref{lem:strong-strongrecovery-3}: $\epsStrongEff \leq c_\varepsilon s^{-1/2}$, $\epsStrongEff \leq \bar{\varepsilon}$, and $\theta_k^\top w^0_k \geq \frac{3}{4}$.
    Set
    \begin{align}
        \eta^t = \begin{cases}
            \eta_1 \leq c_\eta s^{-3/2}
            & (0 \leq t \leq T_{1,1} + T_{1,2} - 1) \\
            \eta_2 \leq c_\eta\bar{\varepsilon} s^{-1} \wedge c_\eta\bar{\varepsilon}^2
            & (T_{1,1} + T_{1,2} \leq t \leq T_{1,1} + T_{1,2} + T_{1,3} + T_{1,4} - 1).
        \end{cases}
    \end{align}
    \begin{enumerate}[leftmargin=*]
        \item[(i)] If $\chi_k \leq c_\chi$, then
        \begin{align}
            \theta_k^\top w^{T_{1,1}+T_{1,2}}_k
            \geq \theta_k^\top w^0_k
            - \chi_k C_3\!\left(\|\Sigma_\kappa w^0_k\|^2 + \epsStrongEff\right)
            - \chi_k C_3\sqrt{\eta_1}
            - \chi_k^2 C_3\eta_1 s.
        \end{align}
        \item[(ii)] If additionally $\chi_k \leq c_2(\log\bar{\varepsilon}^{-1})^{-1}$, then
        \begin{align}
            &\theta_k^\top w^{T_{1,1}+T_{1,2}+T_{1,3}+T_{1,4}}_k \label{eq:strong-unforgetting-2}\\
            &\geq \theta_k^\top w^{T_{1,1}+T_{1,2}}_k
            - \chi_k C_2\log\bar{\varepsilon}^{-1}\!\left(\|\Sigma_\kappa w^{T_{1,1}+T_{1,2}}_k\|^2 + \epsStrongEff\right)
            - \chi_k C_2\sqrt{\eta_2\log\bar{\varepsilon}^{-1}}
            - \chi_k^2 C_2\eta_2 s\log\bar{\varepsilon}^{-1} \nonumber\\
            &\geq \theta_k^\top w^0_k
            - \chi_k(C_3 \vee 2C_2\log\bar{\varepsilon}^{-1})\!\left(\|\Sigma_\kappa w^0_k\|^2 + \epsStrongEff
            + c_3(\chi_k s^{-\frac{3}{4}} + \chi_k^2 s^{-\frac{1}{2}})\right)
            - \chi_k c_3\bar{\varepsilon}\log\bar{\varepsilon}^{-1}. \nonumber
        \end{align}
    \end{enumerate}
\end{lem}

\begin{proof}
    Set $g = g^t_k$ using the notation of \cref{lem:strong-onestep-1}.
    Since $\chi_k \leq c_\chi$, we have $\chi_k C_1\eta^t s^{1/2} \leq \frac{1}{2}$ at every step.
    When $\theta_k^\top w^t_k \geq \frac{1}{2}$, since $\theta_k^\top(w^t_k + \eta^t P^\perp_{w^t_k} g) \geq \theta_k^\top w^t_k - \chi_k C_1\eta^t \geq 0$, we have $\theta_k^\top w^{t+1}_k \geq 0$ and
    \begin{align}
        \theta_k^\top w^{t+1}_k
        &\geq \theta_k^\top w^t_k + \eta^t\theta_k^\top P^\perp_{w^t_k} g
        - \tfrac{1}{2}(\eta^t)^2\|g\|^2(\theta_k^\top w^t_k)
        - \tfrac{1}{2}(\eta^t)^3|\theta_k^\top P^\perp_{w^t_k} g|\|g\|^2.
        \label{eq:strong-unforgetting-3}
    \end{align}

    For $\eta = \eta_1$ or $\eta = \eta_2$, set $\bar{t} = \chi_k C_1\eta t$ and define
    \begin{align}
        P^t &= \theta_k^\top w^0_k
        - \bar{t} e^{2\bar{t}}\!\left[\|\Sigma_\kappa w^0_k\|^2
        + \tfrac{4}{3}\sqrt{\chi_k C_1\eta\,\bar{t}}
        + \bar{t}\!\left(\epsStrongEff + \chi_k C_1\eta s\right)\right]
        - \sqrt{\chi_k C_1\eta\,\bar{t}}\!\left(1 + \epsStrongEff\sqrt{t}\right)
        - \chi_k C_1\eta s\,\bar{t}.
        \label{eq:strong-unforgetting-4}
    \end{align}
    We prove by induction that $\theta_k^\top w^t_k \geq P^t$ for all $t = 0, 1, \dots, T$ whenever $P^T \geq \frac{1}{2}$.
    The base case $t = 0$ is immediate.
    Given the hypothesis up to $t'-1$, by \eqref{eq:strong-unforgetting-3} and the proof of \cref{lem:strong-unforgetting-1},
    \begin{align}
        \theta_k^\top w^{t'}_k
        &\geq \theta_k^\top w^0_k
        - \chi_k C_1\eta\!\left(\sum_{t=0}^{t'-1}\|\Sigma_\kappa w^t_k\|^2 + \sqrt{t'} + \epsStrongEff\, t'\right)
        - (\chi_k C_1\eta)^2 s\, t',
    \end{align}
    and bounding $\sum_{t=0}^{t'-1}\|\Sigma_\kappa w^t_k\|^2$ via \cref{lem:strong-unforgetting-1} yields 
    \begin{align}
        \sum_{t= 0}^{t'-1}\|\Sigma_\kappa w^t_k\|^2 & \leq  \sum_{t=0}^{t'-1} \left(\|\Sigma_\kappa w^0_k\|^2 + 2\chi_k C_1 \eta (\sqrt{t} + \epsStrongEff t ) + 2(\chi_k C_1 \eta)^2 s t \right) \exp (2\chi_k C_1 \eta t) \\ &
        \leq \exp (2\chi_k C_1 \eta t') \left[ \|\Sigma_\kappa w^0_k\|^2 t' + 2\chi_k C_1 \eta \left( \frac{2}{3} t'^{\frac{3}{2}} + \frac{1}{2}\epsStrongEff t'^2 \right) + (\chi_k C_1 \eta)^2 s t'^2  \right]
    \end{align}
    and 
    \begin{align}
        & \theta^\top_k w^0_k - \chi_k C_1 \eta t' \exp (2\chi_k C_1 \eta t') \left[ \|\Sigma_\kappa w^0_k\|^2 + \frac{4}{3} \chi_k C_1 \eta \sqrt{t'} + \chi_k C_1 \eta \epsStrongEff t' + (\chi_k C_1 \eta)^2 s t'  \right] \\ &
         \quad -\chi_k C_1 \eta \left( \sqrt{t'} + \epsStrongEff t' \right) - (\chi_k C_1 \eta)^2 s t' \\ &
        = \theta^\top_k w^0_k - \bar{t} \exp (2\bar{t}) \left[ \|\Sigma_\kappa w^0_k\|^2 + \frac{4}{3} \sqrt{\chi_k C_1 \eta \bar{t}} + \bar{t} \left( \epsStrongEff + \chi_k C_1 \eta s \right) \right] \\ &
         \quad - \sqrt{\chi_k C_1 \eta \bar{t}} \left(1 + \epsStrongEff \sqrt{t'} \right) - \chi_k C_1 \eta s \bar{t} \\ &
          = P^{t'}.
    \end{align}
    Hence $\theta_k^\top w^{t'}_k \geq P^{t'}$.
    Since $P^t$ is decreasing in $t$, the induction goes through whenever $P^T \geq \frac{1}{2}$.

    \textit{Proof of (i).}
    Let $\eta = \eta_1 \leq c_\eta s^{-3/2}$ and set $\tau_1 = \lfloor \frac{1}{2} s^{-1/2} C_1^{-1} \eta_1^{-1} \rfloor$ and $\tau_2 = \lfloor \frac{1}{2} c_1 C_1^{-1} \eta_1^{-1} \rfloor$.
    Then
    \begin{align}
        T_{1,1} + T_{1,2}
        &= \left\lceil \frac{3\log s}{2\alpha_2\beta_{\kappa,2}\eta_1\tau_1} \right\rceil \tau_1
         + \left\lceil \frac{3\log s}{2c_1\alpha_2\beta_{\kappa,2}\eta_1\tau_2} \right\rceil \tau_2 \nonumber\\
        &\leq \frac{3\log s}{2\alpha_2\beta_{\kappa,2}\eta_1} + \tau_1
         + \frac{3\log s}{2c_1\alpha_2\beta_{\kappa,2}\eta_1} + \tau_2
         \leq C_2\eta_1^{-1}.
    \end{align}
    Hence $\bar{t} \coloneqq \chi_k C_1\eta_1 T \leq \chi_k C_1 C_2 \leq c_3$.
    Substituting $T = T_{1,1} + T_{1,2}$ into \eqref{eq:strong-unforgetting-4} and using $\bar{t} \leq \chi_k C_1 C_2$ and $T \leq C_2\eta_1^{-1}$,
    \begin{align}
        P^T
        &\geq \theta_k^\top w^0_k
        - \chi_k C_1 C_2 e^{2c_3}\!\left[\|\Sigma_\kappa w^0_k\|^2
          + \tfrac{4}{3}\chi_k C_1\sqrt{C_2\eta_1}
          + \chi_k C_1 C_2\!\left(\epsStrongEff + \chi_k C_1\eta_1 s\right)\right] \nonumber\\
        &\quad - \chi_k C_1\sqrt{C_2\eta_1} - \chi_k C_1 C_2\,\epsStrongEff - \chi_k^2 C_1^2 C_2\eta_1 s \\
        &\geq \theta_k^\top w^0_k
        - \chi_k C_1 C_2 e^{2c_3}\|\Sigma_\kappa w^0_k\|^2
        - \bigl(1 + \tfrac{4}{3}c_3 e^{2c_3}\bigr)\chi_k C_1\sqrt{C_2\eta_1} \nonumber\\
        &\quad - \chi_k C_1 C_2\bigl(1 + c_3 e^{2c_3}\bigr)\epsStrongEff
        - \bigl(1 + e^{2c_3}\bigr)\chi_k^2 C_1^2 C_2\eta_1 s \\
        &\geq \theta_k^\top w^0_k
        - 2\chi_k C_1 C_2\|\Sigma_\kappa w^0_k\|^2
        - 2\chi_k C_1\sqrt{C_2\eta_1}
        - 2\chi_k C_1 C_2\,\epsStrongEff
        - 3\chi_k^2 C_1^2 C_2\eta_1 s \\
        &\geq \theta_k^\top w^0_k
        - \chi_k C_3\!\left(\|\Sigma_\kappa w^0_k\|^2 + \epsStrongEff\right)
        - \chi_k C_3\sqrt{\eta_1}
        - \chi_k^2 C_3\eta_1 s
        \geq \tfrac{2}{3},
    \end{align}
    where the second inequality expands the bracket, the third uses $c_3 \leq \frac{3}{8}$ and $e^{2c_3} \leq 2$, the fourth uses $C_1 \lesssim C_2 \lesssim C_3$, and the last uses $\theta_k^\top w^0_k \geq \frac{3}{4}$, $\epsStrongEff \leq c_\varepsilon s^{-1/2}$, and $\eta_1 \leq c_\eta s^{-3/2}$.

    \textit{Proof of (ii).}
    Re-indexing $t \mapsto t - T_{1,1} - T_{1,2}$, consider $0 \leq t \leq T_{1,3} + T_{1,4}$ with $\eta = \eta_2$.
    We have
    \[
        T_{1,3} + T_{1,4}
        = \left\lceil 5c_w^{-1}\eta_2^{-1}\log\bar{\varepsilon}^{-1} \right\rceil
        + \left\lceil \eta_2^{-1}\alpha_2^{-1}\beta_{\kappa,2}^{-1}\log\bar{\varepsilon}^{-1} \right\rceil
        \leq C_1\eta_2^{-1}\log\bar{\varepsilon}^{-1},
    \]
    so $\bar{t} \coloneqq \chi_k C_1\eta_2 T \leq \chi_k C_1^2\log\bar{\varepsilon}^{-1} \leq c_1$.
    Substituting $T = T_{1,3} + T_{1,4}$ into \eqref{eq:strong-unforgetting-4} and using $\bar{t} \leq \chi_k C_1^2\log\bar{\varepsilon}^{-1}$ and $T \leq C_1\eta_2^{-1}\log\bar{\varepsilon}^{-1}$,
    \begin{align}
        P^T
        &\geq \theta_k^\top w^0_k
        - \chi_k C_1^2\log\bar{\varepsilon}^{-1}\, e^{2c_1}\!\left[
          \|\Sigma_\kappa w^0_k\|^2
          + \tfrac{4}{3}\sqrt{\chi_k^2 C_1^3\eta_2\log\bar{\varepsilon}^{-1}}
          + \chi_k C_1^2\log\bar{\varepsilon}^{-1}\!\left(\epsStrongEff + \chi_k C_1\eta_2 s\right)
        \right] \nonumber\\
        &\quad - \sqrt{\chi_k^2 C_1^3\eta_2\log\bar{\varepsilon}^{-1}}
        - \sqrt{\chi_k^2 C_1^4\!\left(\log\bar{\varepsilon}^{-1}\right)^2}\,\epsStrongEff
        - \chi_k^2 C_1^3\eta_2 s\log\bar{\varepsilon}^{-1} \\
        &\geq \theta_k^\top w^0_k
        - 2\chi_k C_1^2\log\bar{\varepsilon}^{-1}\!\left(\|\Sigma_\kappa w^0_k\|^2 + \epsStrongEff\right)
        - 2\sqrt{\chi_k^2 C_1^3\eta_2\log\bar{\varepsilon}^{-1}}
        - 3\chi_k^2 C_1^3\eta_2 s\log\bar{\varepsilon}^{-1} \\
        &\geq \theta_k^\top w^0_k
        - \chi_k C_2\log\bar{\varepsilon}^{-1}\!\left(\|\Sigma_\kappa w^0_k\|^2 + \epsStrongEff\right)
        - \chi_k C_2\sqrt{\eta_2\log\bar{\varepsilon}^{-1}}
        - \chi_k^2 C_2\eta_2 s\log\bar{\varepsilon}^{-1},
    \end{align}
    establishing the first line of \eqref{eq:strong-unforgetting-2}.
    For the second line, substitute (i) and \cref{lem:strong-unforgetting-1} into \eqref{eq:strong-unforgetting-2} to bound $\|\Sigma_\kappa w^{T_{1,1}+T_{1,2}}_k\|^2$,
    \begin{align}
        &\theta_k^\top w^{T_{1,1}+T_{1,2}+T_{1,3}+T_{1,4}}_k \nonumber\\
        &\geq \theta_k^\top w^0_k
        - \chi_k C_3\!\left(\|\Sigma_\kappa w^0_k\|^2 + \epsStrongEff\right)
        - \chi_k C_3\sqrt{\eta_1}
        - \chi_k^2 C_3\eta_1 s \nonumber\\
        &\quad - \chi_k C_2\log\bar{\varepsilon}^{-1}\!\left(
          \|\Sigma_\kappa w^0_k\|^2
          + 2\chi_k C_1\eta_1\!\left(\sqrt{T_{1,1}+T_{1,2}} + \epsStrongEff(T_{1,1}+T_{1,2})\right)
          + 4(\chi_k C_1\eta_1)^2 s(T_{1,1}+T_{1,2})
        \right) \nonumber\\
        &\qquad \times \exp\!\left(2\chi_k C_1\eta_1(T_{1,1}+T_{1,2})\right) \nonumber\\
        &\quad - \chi_k C_2\log\bar{\varepsilon}^{-1}\,\epsStrongEff
        - \chi_k C_2\sqrt{\eta_2\log\bar{\varepsilon}^{-1}}
        - \chi_k^2 C_2\eta_2 s\log\bar{\varepsilon}^{-1} \\
        &\geq \theta_k^\top w^0_k
        - \chi_k C_3\!\left(\|\Sigma_\kappa w^0_k\|^2 + \epsStrongEff\right)
        - \chi_k C_3\sqrt{c_\eta}\, s^{-\frac{3}{4}}
        - \chi_k^2 C_3 c_\eta\, s^{-\frac{1}{2}} \nonumber\\
        &\quad - \chi_k C_2\log\bar{\varepsilon}^{-1}\!\left(\|\Sigma_\kappa w^0_k\|^2
        + 2\chi_k C_1\!\left(\sqrt{C_2 c_\eta}\, s^{-\frac{3}{4}} + \epsStrongEff C_2\right)
        + 4\chi_k^2 C_1^2 C_2 c_\eta\, s^{-\frac{1}{2}}\right) e^{2\chi_k C_1 C_2 c_\eta} \nonumber\\
        &\quad - \chi_k C_2\log\bar{\varepsilon}^{-1}\,\epsStrongEff
        - \chi_k C_2\sqrt{c_\eta\bar{\varepsilon}^2\log\bar{\varepsilon}^{-1}}
        - \chi_k^2 C_2 c_\eta\bar{\varepsilon}\log\bar{\varepsilon}^{-1} \\
        &\geq \theta_k^\top w^0_k
        - \chi_k(C_3 \vee 2C_2\log\bar{\varepsilon}^{-1})\!\left(
          \|\Sigma_\kappa w^0_k\|^2 + \epsStrongEff
          + 2\chi_k C_1\sqrt{C_2 c_\eta}\, s^{-\frac{3}{4}}
          + 4\chi_k^2 C_1^2 C_2 c_\eta\, s^{-\frac{1}{2}}
        \right) \nonumber\\
        &\quad - \left(\chi_k C_2\sqrt{c_\eta} + \chi_k^2 C_2 c_\eta\right)\bar{\varepsilon}\log\bar{\varepsilon}^{-1} \\
        &\geq \theta_k^\top w^0_k
        - \chi_k(C_3 \vee 2C_2\log\bar{\varepsilon}^{-1})\!\left(\|\Sigma_\kappa w^0_k\|^2 + \epsStrongEff
        + c_3(\chi_k s^{-\frac{3}{4}} + \chi_k^2 s^{-\frac{1}{2}})\right)
        - \chi_k c_3\bar{\varepsilon}\log\bar{\varepsilon}^{-1},
    \end{align}
    where the third inequality uses $T_{1,1} + T_{1,2} \leq C_2\eta_1^{-1}$ and $\chi_k C_1 C_2 \leq c_3$.
\end{proof}

\section{Full proof of \texorpdfstring{\cref{thm:sft}}{Theorem~\ref*{thm:sft}}: SFT forgetting}\label{app:sft}

In this section we show that, when the strong model is trained by SFT on a specific task (as in \cref{alg:sft-forgetting}), catastrophic forgetting of previously learned features occurs during first-layer training.
As in \cref{app:strong}, we rescale the learning rate by $\eta \leftarrow N_\kappa \pi_\kappa^{-1} \eta$ to normalize for the signal of task $\kappa$.
Under this rescaling, the update rule for the weight $w_k$ associated with task $k$ in \cref{alg:sft-forgetting} reads
\begin{align}
    w_k^{t+1} \gets w_k^t + \eta g_k^t, \quad \text{where} \quad g_k^t = \chi_k \pi_k^{-1} \nabla_{w_k} \bar{y}^t a_k \sigma_k(w_k^{t\top}x + b_k).
\end{align}

By the same arguments as in \cref{strong-strongalignment}, $\Theta(N_\kappa)$ of the $N_\kappa$ neurons associated with task $\kappa$ achieve alignment $1 - \tilde{\varepsilon}$ within
\[
    \tilde{\Theta}\!\left(d^{\frac{p}{2}} s^{\frac{p-2}{2}} \vee d \tilde{\varepsilon}^{-1} \log \tilde{\varepsilon}^{-1} \vee \tilde{\varepsilon}^{-2} \log \tilde{\varepsilon}^{-1}\right)
\]
steps of online SGD.
For neurons associated with task $k \neq \kappa$, however, alignment with the true feature $\theta_k$ may be destroyed.
This section derives the conditions under which such catastrophic forgetting occurs.

Set $c_{l,k} = p \alpha_p \beta_{k,p} = \tilde{\Theta}(1)$ and $c_{u,k} = q^2 \max_{p \leq i \leq q} |\alpha_i \beta_{k,i}| = \tilde{\Theta}(1)$.
We further fix constants of order $\mathop{\mathrm{polylog}} d$ satisfying
\begin{align}
    \begin{cases}
        C_1 \lesssim c_1^{-1} \lesssim c_2^{-1} \\
        c_{u,k}^{-1} \leq c_{l,k}^{-1}
    \end{cases}
    \lesssim C_2 \lesssim c_\eta^{-1} = \tilde{O}(1).
\end{align}
These constants are distinct from those defined analogously in \cref{app:weak,app:strong}.

\begin{asm}[SFT initialization conditions]\label{asm:sft-init}
    Fix $k \neq \kappa$ with $\chi_k \lesssim c_1^{-1}$ and $n \in [N_k]$.
    The strong model is initialized as in \cref{asm:strong-init}, and the cross-task alignment satisfies
    \[
        \theta_\kappa^\top w^0_{k,n} \;\geq\; C_1\, s^{-1/2} \left( \frac{2^{p+1}(1+c_1)^{p-1}\, c_{u,\kappa}}{(1 - \chi_k c_1)^{p-2}\,\chi_k\, c_{l,k}} \right)^{1/(p-2)}.
    \]
\end{asm}

\begin{rem*}
    The intended setting is that the pre-trained strong model has task-$k$ neurons satisfying $\theta_k^\top w^0_{k,n} \approx 1$, then $\theta_\kappa^\top w^0_{k,n} \approx \theta_k^\top \theta_\kappa$, and \cref{asm:sft-init} reduces to requiring $\theta_k^\top \theta_\kappa \gtrsim s^{-1/2}$.
    The condition $\theta_k^\top \theta_\kappa \gtrsim s^{-1/2}$ holds when tasks $k$ and $\kappa$ are more similar than two random directions in the $s$-dimensional subspace, i.e., when the two tasks represent similar capabilities.
\end{rem*}

Under \cref{asm:sft-init}, the cross-task alignment $\theta_\kappa^\top w^0_{k,n}$ is large enough that task-$k$ neurons converge to $\theta_\kappa$ faster than the task-$\kappa$ neurons do, causing forgetting of the pre-trained feature $\theta_k$.

\setcounter{thm}{\numexpr\value{num@thm@sft}-1\relax}
\begin{thm}[SFT forgetting; formal version of \cref{thm:sft}]\label{thm:sft-formal}
    Assume \cref{asm:additive,asm:strong-init} (i) with $\IE(\sigma^*_\kappa) = p \geq 3$, and that the strong model is trained by \cref{alg:sft-forgetting} with learning rate $\eta^t \leq c_\eta d^{-p/2}$.
    Let $t_{1,\kappa} = \min\bigl\{t : \theta_\kappa^\top w^t_{\kappa,n} \geq 1 - o(1) \text{ for some } n \in [N_\kappa]\bigr\}$.
    Then, for every $k \neq \kappa$ and $n \in [N_k]$ satisfying \cref{asm:sft-init} whose Hermite coefficients satisfy $\alpha_p \tilde{\beta}_{k,p} > 0$ and $\alpha_i \tilde{\beta}_{k,i} \geq 0$ for all $i \geq p$, with high probability, $\theta_\kappa^\top w^t_{k,n} \geq 1 - o(1)$ for all $t \geq t_{1,\kappa}$.
\end{thm}

The neurons in group $k$ thus converge to $\theta_\kappa$ no later than the task-$\kappa$ neurons do, so their alignment $\theta_k^\top w^t_{k,n}$ with the pre-trained feature $\theta_k$ is destroyed once task $\kappa$ is learned.

\begin{rem*}
    Under the random initialization in \cref{strong-initialization}, $\Theta(N_k)$ neurons in group $k$ satisfy the sign condition $\alpha_p \tilde{\beta}_{k,p} > 0$ and $\alpha_i \tilde{\beta}_{k,i} \geq 0$ for all $i \geq p$ with high probability.
\end{rem*}

The proof is given by \cref{lem:sft-forgetting-1,lem:sft-forgetting-2,lem:sft-forgetting-3} below.

\begin{algorithm}[t]
    \caption{Online SGD training of the strong model (SFT)}\label{alg:sft-forgetting}
    \DontPrintSemicolon

    \KwIn{Initialized strong model $r_{(\Theta_k)_{k=1}^K}$, $\Theta_k = (a_{n,k}, b_{n,k}, w^0_{n,k})_{n=1}^{N_k}$, learning rate $\eta^t$, number of steps $T$, initialization scale $C_b$.}

    \Indm \textbf{Phase I: first-layer training} \Indp

    \For{$t=0, 1, \dots, T-1$}{
        $x^t \sim \mathcal{N}(0, I_d)$, $y^t = r^*_\kappa(\theta^\top_\kappa x) + \zeta$.\;
        $w_{k,n}^{t+1} \gets w_{k,n}^t + \eta^t y^t \tilde\nabla_w r^s_{((a_{k,n}, b_{k,n}, w_{k,n}^t)_{n=1}^{N_k})_{k=1}^{K}}(x^t)$ \;
        $w_{k,n}^{t+1} \gets w_{k,n}^{t+1} / \|w_{k,n}^{t+1}\|, \quad (n = 1,\dots,N_k),\quad (k = 1,\dots,K)$ \;
    }
    $\hat{w}_{k,n} \gets w_{k,n}^{T}$ \;

    \KwOut{$r_{\hat{\Theta}}(x)$ with $\hat{\Theta}_k = (\hat{a}_{k,n}, b_{k,n}, \hat{w}_{k,n})_{n=1}^{N_k}$.}
\end{algorithm}

\begin{lem}\label{lem:sft-forgetting-1}
    Let $p \geq 2$, $\eta = \eta^t \leq c_\eta d^{-\frac{p}{2}}$, and $\chi_k \lesssim c_1^{-1}$.
    Suppose $\theta^\top_\kappa w^0_k \geq s^{-\frac{1}{2}}$.
    If $\theta^\top_\kappa w^t_k \leq c_1$ for all $t = 1, 2, \dots, \tau$, then for every $t \leq \tau$,
    \begin{align}
        &(1 - \chi_k c_1) \theta^\top_\kappa w^0_k + \chi_k \eta \sum_{t'=0}^{t} (1 - c_1) c_{l,k} (\theta^\top_\kappa w^{t'}_k)^{p-1} \\ &
        \leq \theta_k^\top w^{t+1}_k \\ &
        \leq (1 + \chi_k c_1) \theta^\top_\kappa w^0_k + \chi_k \eta \sum_{t'=0}^{t} (1 + c_1) c_{u,k} (\theta^\top_\kappa w^{t'}_k)^{p-1}.
    \end{align}
    Moreover, defining sequences $(P^t_k)_{t=0}^{\tau+1}$ and $(Q^t_k)_{t=0}^{\tau+1}$ by
    $P^0_k = (1 - \chi_k c_1) \theta^\top_\kappa w^0_k$, $Q^0_k = (1 + \chi_k c_1) \theta^\top_\kappa w^0_k$, and for $t = 1, 2, \dots, \tau$,
    \begin{align}
        P^{t+1}_k &= P^t_k + \chi_k \eta (1 - c_1) c_{l,k} (P^t_k)^{p-1},\\
        Q^{t+1}_k &= Q^t_k + \chi_k \eta (1 + c_1) c_{u,k} (Q^t_k)^{p-1},
    \end{align}
    we have $P^t_k \leq \theta^\top_\kappa w^t_k \leq Q^t_k$ for every $t \leq \tau+1$ with high probability.
\end{lem}

\begin{proof}
    For any $k \in [K]$, the gradient term decomposes as
    \begin{align}
        g^t_k = \chi_k \sum_{i=p}^q \left[ i \alpha_i \beta_{k,i} (\theta^\top_\kappa w^t_k)^{i-1} \theta_\kappa + \sqrt{(i+2)(i+1)} \alpha_i \beta_{k,i+2} (\theta^\top_\kappa w^t_k)^i w^t_k \right] + \chi_k Z^t,
    \end{align}
    where $Z^t$ has mean $0$ and satisfies $\|Z^t\| \leq c_1 d^{\frac{1}{2}}$ with high probability and $|v^\top Z^t| \leq c_1$ with high probability for every $v \in S^{d-1}$.
    If $\theta^\top_\kappa w^{t+1}_k \geq \frac{1}{2} d^{-\frac{1}{2}}$, then by the same argument as in \cref{lem:weak-onestep-2},
    \begin{align}
        \theta^\top_\kappa w^{t+1}_k
        &\geq (\theta^\top_\kappa w^t_k + \eta \theta^\top_\kappa P^\perp_{w^t_k} g)(1 - \tfrac{1}{2}\eta^2 \|P^\perp_{w^t_k} g\|^2) \\
        &\geq \theta^\top_\kappa w^t_k + \chi_k \eta c_{l,k} (\theta^\top_\kappa w^t_k)^{p-1}(1 - (\theta^\top_\kappa w^t_k)^2) - \chi_k^2 \eta^2 c_1^2 (\theta^\top_\kappa w^t_k) d + \chi_k \eta \theta^\top_\kappa P^\perp_{w^t_\kappa} Z^t,
    \end{align}
    where we used $\chi_k \eta c_1 \leq 1$ and hence $\eta |\theta^\top_\kappa P^\perp_{w^t_k} g| \leq 1$ with high probability.
    By an argument analogous to \cref{lem:weak-weakrecovery-1},
    \begin{align}
        \theta^\top_\kappa w^{t+1}_k \geq (1 - \chi_k c_1) \theta^\top_\kappa w^0_k + \chi_k \eta (1 - c_1) \sum_{t'=0}^{t} c_{l,k} (\theta^\top_\kappa w^{t'}_k)^{p-1}.
        \label{eq:sft-forgetting-1}
    \end{align}
    On the other hand,
    \begin{align}
        \theta^\top_\kappa w^{t+1}_k
        &\leq \theta^\top_\kappa w^t_k + \eta \theta^\top_\kappa P^\perp_{w^t_k} g \\
        &= \theta^\top_\kappa w^t_k + \chi_k \eta \sum_{i=p}^q i \alpha_i \beta_{k,i} (\theta^\top_\kappa w^t_k)^{i-1} \theta^\top_\kappa P^\perp_{w^t_k} \theta_\kappa + \chi_k \eta \theta^\top_\kappa P^\perp_{w^t_k} Z^t \\
        &\leq \theta^\top_\kappa w^t_k + \chi_k \eta c_{u,k} (\theta^\top_\kappa w^t_k)^{p-1}(1 - (\theta^\top_\kappa w^t_k)^2) + \chi_k \eta \theta^\top_\kappa P^\perp_{w^t_k} Z^t,
    \end{align}
    and by an argument analogous to \cref{lem:weak-weakrecovery-1},
    \begin{align}
        \theta^\top_\kappa w^{t+1}_k \leq (1 + \chi_k c_1) \theta^\top_\kappa w^0_k + \chi_k \eta (1 + c_1) \sum_{t'=0}^{t} c_{u,k} (\theta^\top_\kappa w^{t'}_k)^{p-1}.
        \label{eq:sft-forgetting-2}
    \end{align}
    Combining \eqref{eq:sft-forgetting-1} and \eqref{eq:sft-forgetting-2} gives the desired bounds.
\end{proof}

In the following, we shift the time origin to $t_{1,k}$, the first time at which weak alignment of $w^t_k$ is achieved.

\begin{lem}\label{lem:sft-forgetting-2}
    Let $\chi_k \lesssim c_1^{-1}$.
    Suppose $\theta^\top_\kappa w^0_k \geq c_1$, $\eta = \eta_t \leq c_\eta d^{-\frac{p}{2}}$, and $\theta^\top_\kappa w^t_k \leq 1 - c_1$ for all $t \leq \tau$.
    Then,
    \begin{align}
        (1 - \chi_k c_1) \theta^\top_\kappa w^0_k + \chi_k \eta \sum_{t'=0}^t c_1 c_{l,k} (\theta^\top_\kappa w^{t'}_k)^{p-1}
        &\leq \theta^\top_\kappa w^{t+1}_k \\
        &\leq (1 + \chi_k c_1) \theta^\top_\kappa w^0_k + \chi_k \eta \sum_{t'=0}^t c_1 c_{u,k} (\theta^\top_\kappa w^{t'}_k)^{p-1}.
    \end{align}
    Moreover, defining $(P^t_k)_{t=0}^{\tau+1}$ and $(Q^t_k)_{t=0}^{\tau+1}$ by $P^0_k = (1-\chi_k c_1)\theta^\top_\kappa w^0_k$, $Q^0_k = (1+\chi_k c_1)\theta^\top_\kappa w^0_k$,
    $P^{t+1}_k = P^t_k + \chi_k \eta c_1 c_{l,k} (P^t_k)^{p-1}$, and $Q^{t+1}_k = Q^t_k + \chi_k \eta c_1 c_{u,k} (Q^t_k)^{p-1}$,
    we have $P^t_k \leq \theta^\top_\kappa w^t_k \leq Q^t_k$ for every $t \leq \tau+1$.
\end{lem}

\begin{proof}
    The proof is analogous to that of \cref{lem:weak-amplification-1}.
\end{proof}

\begin{lem}\label{lem:sft-forgetting-3}
    Let $p \geq 3$, $\eta = \eta_t \leq c_\eta d^{-\frac{p}{2}}$, $\chi_k \lesssim c_1^{-1}$, $\theta^\top_\kappa w^t_\kappa \geq s^{-\frac{1}{2}}$, and $c_1^{-(p-2)}(1+c_1)^{p-2}(\theta^\top_\kappa w^0_\kappa)^{p-2} \leq \frac{1}{2}$.
    If $\theta^\top_\kappa w^{t}_\kappa > c_1$, then
    \begin{align}
        t \geq \frac{(1+c_1)^{-(p-1)}(\theta^\top_\kappa w^0_\kappa)^{-(p-2)}}{4(p-2)\eta c_{u,\kappa}}.
    \end{align}
    On the other hand, if
    \begin{align}
        t \geq 2^{p-2} \frac{(1-c_1)^{-1}(\theta^\top_\kappa w^0_k)^{-(p-2)} + c_1^{-(p-1)}}{(1-\chi_k c_1)^{p-2}(p-2)\chi_k \eta c_{l,k}},
    \end{align}
    then $\theta^\top_\kappa w^{t}_k > 1 - c_1$.
    Furthermore, if $c_1 \gtrsim \theta^\top_\kappa w^0_k \geq C_1 s^{-\frac{1}{2}} \left( \frac{2^{p+1}(1+c_1)^{p-1} c_{u,\kappa}}{(1-\chi_k c_1)^{p-2} \chi_k c_{l,k}} \right)^{\frac{1}{p-2}}$, then $\theta^\top_\kappa w^{t}_k > 1-c_1$ for every $t \geq t_{1,\kappa}$.
\end{lem}

\begin{proof}
    By \cref{lem:sft-forgetting-1} and the Bihari--LaSalle inequality, for every $t \leq t_{1,\kappa}$,
    \begin{align}
        \theta^\top_\kappa w^{t}_\kappa \leq \frac{(1+c_1)\theta^\top_\kappa w^0_\kappa}{\left(1 - \eta(1+c_1)c_{u,k}(p-2)((1+c_1)\theta^\top_\kappa w^0_\kappa)^{p-2} t_{1,\kappa}\right)^{\frac{1}{p-2}}},
    \end{align}
    so a necessary condition for $\theta^\top_\kappa w^{t_{1,\kappa}}_\kappa > c_1$ is
    \[
        t_{1,\kappa} > \frac{1 - c_1^{-(p-2)}(1+c_1)^{p-2}(\theta^\top_\kappa w^0_\kappa)^{p-2}}{\eta(1+c_1)^{p-1} c_{u,\kappa}(p-2)(\theta^\top_\kappa w^0_\kappa)^{p-2}}.
    \]

    Since $\theta^\top_\kappa w^{t_{1,\kappa}}_\kappa \leq c_1 + C_1 \eta$, for every $t \leq t_{2,\kappa}$,
    \begin{align}
        \theta^\top_\kappa w^{t_{1,\kappa}+t}_\kappa \leq \frac{(1+c_1)(c_1+C_1\eta)}{\left(1 - \eta c_1 c_{u,\kappa}(p-2)((1+c_1)(c_1+C_1\eta))^{p-2} t_{2,\kappa}\right)^{\frac{1}{p-2}}},
    \end{align}
    so a necessary condition for $\theta^\top_\kappa w^{t_{1,\kappa}+t_{2,\kappa}}_\kappa > 1-c_1$ is
    \[
        t_{2,\kappa} > \frac{1 - (1-c_1)^{-(p-2)}(1+c_1)^{p-2}(c_1+C_1\eta)^{p-2}}{\eta c_1 c_{u,\kappa}(p-2)(1+c_1)^{p-2}(c_1+C_1\eta)^{p-2}}.
    \]
    Therefore,
    \begin{align}
        t_{1,\kappa} + t_{2,\kappa}
        &\geq \frac{1}{(1+c_1)^{p-2}(p-2)\eta c_{u,\kappa}} \notag\\
        &\quad \times \left( \frac{1 - c_1^{-(p-2)}(1+c_1)^{p-2}(\theta^\top_\kappa w^0_\kappa)^{p-2}}{(1+c_1)(\theta^\top_\kappa w^0_\kappa)^{p-2}} + \frac{1-(1-c_1)^{-(p-2)}(1+c_1)^{p-2}(c_1+C_1\eta)^{p-2}}{c_1(c_1+C_1\eta)^{p-2}} \right) \\
        &\geq \frac{(1+c_1)^{-(p-1)}(\theta^\top_\kappa w^0_\kappa)^{-(p-2)} + c_1^{-1}(c_1+C_1\eta)^{-(p-2)}}{4(p-2)\eta c_{u,\kappa}},
    \end{align}
    where the second inequality uses $c_1^{-(p-2)}(1+c_1)^{p-2}(\theta^\top_\kappa w^0_\kappa)^{p-2} \leq \frac{1}{2}$, $(1-c_1)^{-(p-2)}(c_1+C_1\eta)^{p-2} \leq \frac{1}{2}$, and $(1+c_1)^{p-2} \leq 2$.
    On the other hand, by \cref{lem:sft-forgetting-1} and the Bihari--LaSalle inequality, for every $t \geq t_{1,k}$,
    \begin{align}
        \theta^\top_\kappa w^t_k \geq \frac{(1-\chi_k c_1)\theta^\top_\kappa w^0_k}{\left(1 - \chi_k\eta(1-c_1)c_{l,k} 2^{-(p-1)}(p-2)((1-\chi_k c_1)\theta^\top_\kappa w^0_k)^{p-2} t_{1,k}\right)^{\frac{1}{p-2}}},
    \end{align}
    using $\chi_k\eta(1-c_1)c_{l,k} \leq 1$.
    A sufficient condition for $\theta^\top_\kappa w^{t_{1,k}}_k > c_1$ is then
    \[
        t_{1,k} > \frac{1 - c_1^{-(p-2)}((1-\chi_k c_1)\theta^\top_\kappa w^0_k)^{p-2}}{\chi_k\eta(1-c_1)c_{l,k} 2^{-(p-1)}(p-2)((1-\chi_k c_1)\theta^\top_\kappa w^0_k)^{p-2}}.
    \]
    Since $\theta^\top_\kappa w^{t_{1,k}}_k \geq c_1$, for every $t \geq t_{2,k}$,
    \begin{align}
        \theta^\top_\kappa w^{t_{1,k}+t}_k \geq \frac{(1-\chi_k c_1) c_1}{\left(1 - \chi_k\eta c_1 c_{l,k} 2^{-(p-1)}(p-2)((1-\chi_k c_1) c_1)^{p-2} t_{2,k}\right)^{\frac{1}{p-2}}},
    \end{align}
    using $\chi_k\eta c_1 c_{l,k} \leq 1$.
    A sufficient condition for $\theta^\top_\kappa w^{t_{2,k}}_k > 1-c_1$ is
    \[
        t_{2,k} > \frac{1 - (1-c_1)^{-(p-2)}((1-\chi_k c_1) c_1)^{p-2}}{\chi_k\eta c_1 c_{l,k} 2^{-(p-1)}(p-2)((1-\chi_k c_1) c_1)^{p-2}}.
    \]
    Therefore,
    \begin{align}
        t_{1,k} + t_{2,k}
        &> \frac{2^{p-1}}{(1-\chi_k c_1)^{p-2}(p-2)\chi_k\eta c_{l,k}} \notag\\
        &\quad \times \left( \frac{1-c_1^{-(p-2)}((1-\chi_k c_1)\theta^\top_\kappa w^0_k)^{p-2}}{(1-c_1)(\theta^\top_\kappa w^0_\kappa)^{p-2}} + \frac{1-(1-c_1)^{-(p-2)}((1-\chi_k c_1) c_1)^{p-2}}{c_1^{p-1}} \right) \\
        &\geq 2^{p-2} \frac{(1-c_1)^{-1}(\theta^\top_\kappa w^0_k)^{-(p-2)} + c_1^{-(p-1)}}{(1-\chi_k c_1)^{p-2}(p-2)\chi_k\eta c_{l,k}},
    \end{align}
    where the second inequality uses $c_1^{-(p-2)}((1-\chi_k c_1)\theta^\top_\kappa w^0_k)^{p-2} \leq \frac{1}{2}$ and $(1-c_1)^{-(p-2)}((1-\chi_k c_1) c_1)^{p-2} \leq \frac{1}{2}$.
    Finally, if $\theta^\top_\kappa w^0_k \geq C_1 s^{-\frac{1}{2}} \left( \frac{2^{p+1}(1+c_1)^{p-1} c_{u,\kappa}}{(1-\chi_k c_1)^{p-2}\chi_k c_{l,k}} \right)^{\frac{1}{p-2}}$, then with high probability,
    \begin{align}
        \left(\frac{\theta^\top_\kappa w^0_k}{\theta^\top_\kappa w^0_\kappa}\right)^{p-2} \geq \frac{2^{p+1}(1+c_1)^{p-1} c_{u,\kappa}}{(1-\chi_k c_1)^{p-2}\chi_k c_{l,k}},
    \end{align}
    where we used $\theta^\top_\kappa w^0_\kappa \leq C_1 s^{-\frac{1}{2}}$ with high probability, which follows from the sub-Gaussianity of the uniform distribution on the sphere.
    Using $(1-c_1)^{-1}(\theta^\top_\kappa w^0_k)^{-(p-2)} + c_1^{-(p-1)} \leq 2(\theta^\top_\kappa w^0_k)^{-(p-2)}$,
    \begin{align}
        \frac{(1+c_1)^{-(p-1)}(\theta^\top_\kappa w^0_\kappa)^{-(p-2)}}{4(p-2)\eta c_{u,\kappa}} \geq 2^{p-2} \frac{(1-c_1)^{-1}(\theta^\top_\kappa w^0_k)^{-(p-2)} + c_1^{-(p-1)}}{(1-\chi_k c_1)^{p-2}(p-2)\chi_k\eta c_{l,k}}.
    \end{align}
    Since $t_{1,\kappa} \geq t_{1,k} + t_{2,k}$, it follows that $\theta^\top_\kappa w^{t}_k > 1-c_1$ for every $t \geq t_{1,\kappa}$.
\end{proof}

\section{Technical lemmas}\label{app:lemmas}

\begin{lem}[Bihari--LaSalle and Gronwall inequalities; cf.\ \citet{oko2024learningsumdiversefeatures,benarous2021online}] \label{lem:technical-bihari}
    Let $p \geq 3$, $c > 0$, and let $(a^t)_{t=0}^\infty$ be a sequence of positive reals satisfying $a^{t+1} = a^t + c(a^t)^{p-1}$. Then
    \begin{align}
        a^t \leq \frac{a^0}{\bigl(1 - c(p-2)(a^0)^{p-2}t\bigr)^{\frac{1}{p-2}}}.
    \end{align}
    Moreover, if $a^t \leq 1$ for all $t \leq T-1$, then
    \begin{align}
        a^t \geq \frac{a^0}{\bigl(1 - c(1+c)^{-(p-1)}(p-2)(a^0)^{p-2}t\bigr)^{\frac{1}{p-2}}}.
    \end{align}
\end{lem}
\begin{proof}
    Since the integrand $x^{-(p-1)}$ is decreasing,
    \begin{align}
        c = \frac{a^{t+1} - a^t}{(a^t)^{p-1}} \geq \int_{a^t}^{a^{t+1}} \frac{dx}{x^{p-1}}
        = \frac{1}{p-2}\left[\frac{1}{(a^t)^{p-2}} - \frac{1}{(a^{t+1})^{p-2}}\right].
    \end{align}
    Summing from $0$ to $t-1$ gives
    \begin{align}
        c(p-2)t \geq (a^0)^{-(p-2)} - (a^t)^{-(p-2)},
    \end{align}
    and rearranging yields the upper bound
    \begin{align}
        a^t \leq \frac{1}{\bigl((a^0)^{-(p-2)} - c(p-2)t\bigr)^{\frac{1}{p-2}}}
        = \frac{a^0}{\bigl(1 - c(p-2)(a^0)^{p-2}t\bigr)^{\frac{1}{p-2}}}.
    \end{align}
    For the lower bound, note that $a^t \leq 1$ implies $a^{t+1} \leq (1+c)a^t$, so
    \begin{align}
        c = \frac{a^{t+1} - a^t}{(a^t)^{p-1}}
        = \int_{a^t}^{a^{t+1}} \frac{dx}{(a^t)^{p-1}}
        \leq \int_{a^t}^{a^{t+1}} \frac{(1+c)^{p-1}}{(a^{t+1})^{p-1}} dx
        \leq \frac{(1+c)^{p-1}}{p-2}\left[\frac{1}{(a^t)^{p-2}} - \frac{1}{(a^{t+1})^{p-2}}\right].
    \end{align}
    Summing and rearranging gives
    \begin{align}
        a^t \geq \frac{a^0}{\bigl(1 - c(1+c)^{-(p-1)}(p-2)(a^0)^{p-2}t\bigr)^{\frac{1}{p-2}}}.
    \end{align}
\end{proof}

\begin{lem}[Matrix Bernstein inequality {\citep{Vershynin_2018}}] \label{lem:technical-matrixbernstein}
    Let $X_1, \dots, X_N$ be independent, mean-zero, symmetric $d \times d$ random matrices with $\|X_i\| \leq K$ almost surely for all $i$. Then for all $t \geq 0$,
    \begin{align}
        \mathbb{P}\!\left[\left\|\sum_{i=1}^N X_i\right\| \geq t\right]
        \leq 2d \exp\!\left(-\frac{t^2/2}{\sigma^2 + Kt/3}\right),
    \end{align}
    where $\sigma^2 = \bigl\|\sum_{i=1}^N \mathbb{E}[X_i^2]\bigr\|$ is the operator norm of the variance of the sum.
\end{lem}

\begin{lem}[\citet{yu2015useful}, Corollary~3] \label{lem:technical-yu}
    Let $\Sigma, \hat{\Sigma} \in \mathbb{R}^{d \times d}$ be symmetric matrices with eigenvalues $\lambda_1 \geq \dots \geq \lambda_d$ and $\hat{\lambda}_1 \geq \dots \geq \hat{\lambda}_d$, respectively. Fix $j \in [d]$ and assume $\min(\lambda_{j-1} - \lambda_j,\, \lambda_j - \lambda_{j+1}) > 0$, where $\lambda_0 = \infty$ and $\lambda_{d+1} = -\infty$. If $v, \hat{v} \in \mathbb{R}^d$ satisfy $\Sigma v = \lambda_j v$ and $\hat{\Sigma}\hat{v} = \hat{\lambda}_j \hat{v}$, then
    \begin{align}
        \sin\Theta(\hat{v}, v) \leq \frac{2\|\hat{\Sigma} - \Sigma\|}{\min(\lambda_{j-1} - \lambda_j,\, \lambda_j - \lambda_{j+1})},
    \end{align}
    where $\Theta(\hat{v}, v) \in [0, \pi]$ denotes the angle between $\hat{v}$ and $v$. Furthermore, if $\hat{v}^\top v \geq 0$, then
    \begin{align}
        \|\hat{v} - v\| \leq \frac{2^{3/2}\|\hat{\Sigma} - \Sigma\|}{\min(\lambda_{j-1} - \lambda_j,\, \lambda_j - \lambda_{j+1})}.
    \end{align}
\end{lem}


\end{document}